\documentclass[authoryear]{elsarticle}

\usepackage[a4paper,margin=2cm]{geometry}

\newcommand{\tcite}[1]{\citet{#1}}
\newcommand{\pcite}[1]{\citep{#1}}

\usepackage{lipsum}
\makeatletter
\def\ps@pprintTitle{%
 \let\@oddhead\@empty
 \let\@evenhead\@empty
 \def\@oddfoot{ A preprint. \hfill \today }%
 \def\@evenfoot{\thepage\hfill}}
\makeatother

\usepackage{graphicx, color}
\usepackage{subcaption}
\usepackage[utf8]{inputenc}

\usepackage{fancyhdr}
\usepackage{amsmath}
\usepackage{amsthm}
\usepackage{amssymb}
\usepackage{amsfonts}
\usepackage{qtree}
\usepackage[counterclockwise]{rotating}
\usepackage{scalerel}
\usepackage{textgreek}
\usepackage{algorithm}
\usepackage{algpseudocode}
\usepackage{tikz,forest}
\usepackage{tablefootnote}
\usepackage{hyperref}
\usepackage{dirtytalk}
\usetikzlibrary{arrows.meta}
\usepackage{xspace}
\usepackage{wrapfig}

\usepackage{changes}

\usepackage{dcolumn}
\newcolumntype{.}{D{.}{.}{-1}}
\makeatletter
\newcolumntype{B}{>{\boldmath\DC@{.}{.}{-1}}c<{\DC@end}}
\newcolumntype{E}{>{\centering\DC@{.}{-}{-1}}c<{\DC@end}}
\makeatother
\newcommand\mc[1]{\multicolumn{1}{c}{#1}}
\newcommand\bft[1]{\multicolumn{1}{B}{#1}}

\tikzset{Variable/.style={circle,thick,draw}}
\tikzset{Factor/.style={rectangle,thick,draw}}
\tikzset{WeirdArrow/.style={dashed}}
\tikzset{every edge/.style={draw=black,very thick}}

\theoremstyle{definition}

\newdefinition{defff}{Definition}

\newcommand{\Appendix}{}

\newenvironment{deff}[1]{%
    \begin{defff}#1}{
    \end{defff}%
}

\theoremstyle{definition}
\newtheorem{exmp}{Example}

\newtheorem{prop}{Proposition}

\newcommand{\predicates}{\mathcal{P}}
\newcommand{\constants}{C}
\newcommand{\fol}{\mathcal{L}}

\newcommand{\objects}{O}

\newcommand{\dataset}{\mathcal{D}}

\newcommand{\btheta}{{\boldsymbol{\theta}}}
\newcommand{\luk}{\L{}ukasiewicz}
\newcommand{\loss}{\mathcal{L}}
\newcommand{\corpus}{\mathcal{K}}

\newcommand{\interpretationfol}{\eta}
\newcommand{\interpretation}{\eta_{\btheta}}

\newcommand{\val}{e_{\btheta}}

\newcommand{\batch}{b}
\newcommand{\deri}{d}
\newcommand{\dmp}[2]{\dmpa{I}{#1}{#2}}
\newcommand{\dmpa}[3]{\frac{\partial #1(#2, #3)}{\partial #3}}
\newcommand{\dmt}[2]{\dmta{I}{#1}{#2}}
\newcommand{\dmta}[3]{-\frac{\partial #1(#2, #3)}{\partial #2}}
\newcommand{\dmtp}[2]{\dmtap{I}{#1}{#2}}
\newcommand{\dmtap}[3]{-\frac{\partial #1(#2, #3)}{\partial #2}}
\newcommand{\instantiations}{M}
\newcommand{\instantiation}{\mu}

\newcommand{\bx}{{\boldsymbol{x}}}

\newcommand{\world}{{\boldsymbol{w}}}

\newcommand{\graphwidth}{0.35\linewidth}

\newcommand{\cons}{\mathrm{cons}}
\newcommand{\ant}{\mathrm{ant}}
\newcommand{\mpmag}{|\cons|}
\newcommand{\mtmag}{|\ant|}
\newcommand{\mpratio}{\cons\%}

\newcommand{\mpcorupdate}{{cu_{\cons}}}
\newcommand{\mtcorupdate}{{cu_{\ant}}}

\newcommand{\mpupdateratio}{cu_\cons\%}
\newcommand{\mtupdateratio}{cu_\ant\%}

\newcommand{\dl}{DL\xspace}
\newcommand{\dfl}{DFL\xspace}
\newcommand{\dlogic}{Differentiable Logics\xspace}
\newcommand{\dfuzz}{Differentiable Fuzzy Logics\xspace}

\newcommand{\productlogic}{Differentiable Product Fuzzy Logic\xspace}
\newcommand{\dprob}{Differentiable Probabilistic Logics\xspace}

\newcommand{\dpfl}{DPFL\xspace}

\newcommand{\logprod}{A_{\log T_P}}

\newcommand{\predP}{\pred{P}}
\newcommand{\predQ}{\pred{Q}}

\newcommand{\indic}[1]{\mathbf{1}_{#1}}

\newcommand{\dualfigure}[4]{\begin{figure}[h]
\centering
\begin{subfigure}[b]{\graphwidth}
\includegraphics[width=\linewidth]{#1}
\end{subfigure}
\begin{subfigure}[b]{\graphwidth}
\includegraphics[width=\linewidth]{#2}
\end{subfigure}%
\caption{#3}
\label{#4}
\end{figure}}

\newcommand{\imgT}{images/theory/}
\newcommand{\imgE}{images/experiments/}

\newcommand{\pred}[1]{\textsf{#1}}
\newcommand{\argmax}{\text{argmax}}
\newcommand{\argmin}{\text{argmin}}

\DeclareMathOperator*{\aggregate}{\scalerel*{\mathrm{A}}{\sum}}

\DeclareMathOperator*{\Eaggregate}{\scalerel*{\mathrm{E}}{\sum}}

\graphicspath{ {images/} }

\forestset{
    .style={
        for tree={
            base=bottom,
            child anchor=north,
            align=center,
            s sep+=1cm,
    straight edge/.style={
        edge path={\noexpand\path[\forestoption{edge},thick,-{Latex}] 
        (!u.parent anchor) -- (.child anchor);}
    },
    if n children={0}
        {tier=word, draw, thick, rectangle}
        {draw, diamond, thick, aspect=2},
    if n=1{%
        edge path={\noexpand\path[\forestoption{edge},thick,-{Latex}] 
        (!u.parent anchor) -| (.child anchor) node[pos=.2, above] {Y};}
        }{
        edge path={\noexpand\path[\forestoption{edge},thick,-{Latex}] 
        (!u.parent anchor) -| (.child anchor) node[pos=.2, above] {N};}
        }
        }
    }
}

\begin{document}
\title{Analyzing Differentiable Fuzzy Logic Operators}

\begin{abstract}
The AI community is increasingly putting its attention towards combining symbolic and neural approaches, as it is often argued that the strengths and weaknesses of these approaches are complementary. 
One recent trend in the literature are weakly supervised learning techniques that employ operators from fuzzy logics. 
In particular, these use prior background knowledge described in such logics to help the training of a neural network from unlabeled and noisy data. 
By interpreting logical symbols using neural networks, this background knowledge can be added to regular loss functions, hence making reasoning a part of learning.

We study, both formally and empirically, how a large collection of logical operators from the fuzzy logic literature behave in a differentiable learning setting. 
We find that many of these operators, including some of the most well-known, are highly unsuitable in this setting.
A further finding concerns the treatment of implication in these fuzzy logics, and shows a strong imbalance between gradients driven by the antecedent and the consequent of the implication. 
Furthermore, we introduce a new family of fuzzy implications (called sigmoidal implications) to tackle this phenomenon. 
Finally, we empirically show that it is possible to use \dfuzz for semi-supervised learning, and compare how different operators behave in practice. We find that, to achieve the largest performance improvement over a supervised baseline, we have to resort to non-standard combinations of logical operators which perform well in learning, but no longer satisfy the usual logical laws. 
\end{abstract}


\date{}

\author[add1]
  {Emile van Krieken\corref{cor1}}
\ead{e.van.krieken@vu.nl}

\author[add1,add2]
  {Erman Acar}
\ead{erman.acar@vu.nl}

\author[add1]
  {Frank van Harmelen}
\ead{Frank.van.Harmelen@vu.nl}
   
\cortext[cor1]{Corresponding author}
\address[add1]{Vrije Universiteit Amsterdam}
\address[add2]{Civic AI Lab, Amsterdam}
\maketitle





\section{Introduction}
In recent years, integrating symbolic and statistical 
approaches to Artificial Intelligence (AI) gained considerable attention \pcite{garcez2012neural,Besold2017a}. This research line has gained further traction due to recent influential critiques on purely statistical deep learning \pcite{marcus2018deep,pearl2018theoretical}, which has been the focus of the AI community in the last decade. While deep learning has brought many important breakthroughs in computer vision \pcite{brock2018large}, natural language processing \pcite{radford2019language} and reinforcement learning \pcite{silver2017mastering}, the concern is that progress will be halted if its shortcomings are not dealt with. Among these is the massive amounts of data that deep learning models need to learn even a simple concept. In contrast, symbolic AI can easily reuse concepts and can express domain knowledge using only a single logical statement.
Finally, it is much easier to integrate background knowledge using symbolic AI. 

However, symbolic AI has issues with scalability: dealing with large amounts of data while performing complex reasoning task, and is not able to deal with the noise and ambiguity of e.g. sensory data. The latter is related to the well-known \textit{symbol grounding problem} which \tcite{harnad1990symbol} defines as how \say{the semantic interpretation of a formal symbol system can be made intrinsic to the system, rather than just parasitic on the meanings in our heads}. In particular, symbols refer to concepts that have an intrinsic meaning to us humans, but computers manipulating these symbols cannot \emph{understand} (or \emph{ground}) this meaning. On the other hand, a properly trained deep learning model excels at modeling complex sensory data. These models could bridge the gap between symbolic systems and the real world. Therefore, several recent approaches \pcite{diligenti2017,garnelo2016towards,Serafini2016,DBLP:conf/nips/2018,Evans2018} aim at interpreting symbols that are used in logic-based systems using deep learning models. These are among the first systems to implement \say{a hybrid nonsymbolic/symbolic system (...) in which the elementary symbols are grounded in (...) non-symbolic representations that pick out, from their proximal sensory projections, the distal object categories to which the elementary symbols refer.} \tcite{harnad1990symbol}.

\subsection{Reasoning and Learning using Gradient Descent}

We introduce \textit{\dfuzz} (\dfl) which aims to integrate reasoning and learning by using logical formulas expressing background knowledge. The symbols in these formulas are interpreted using a deep learning model of which the parameters are to be learned. \dfl constructs differentiable loss functions based on these formulas that can be minimized using gradient descent. This ensures that the deep learning model acts in a manner that is consistent with the background knowledge as we can backpropagate towards the parameters of the deep learning model.

To ensure loss functions are differentiable, \dfl uses fuzzy logic semantics \pcite{klir1995fuzzy}. 
Predicates, functions and constants are interpreted using the deep learning model. 
By maximizing the degree of truth of the background knowledge using gradient descent, both learning and reasoning are performed in parallel. 
We can apply the loss functions constructed using \dfl for more challenging machine learning tasks than purely supervised learning. 
These methods fall under the umbrella of weakly supervised learning \pcite{zhou2017}. 
For example, it can be used for semi-supervised learning \pcite{pmlr-v80-xu18h,P16-1228} or to detect noisy or inaccurate supervision \pcite{Donadello2017}.
For such problems, \dfl corrects the predictions of the deep learning model when it is logically inconsistent with the background knowledge.  


To further our understanding of such losses, we present in this paper an analysis of the choice of operators used to compute the logical connectives in \dfl. 
For example, functions called \textit{t-norms} are used to connect two fuzzy propositions \pcite{klir1995fuzzy}. 
Because they return the degree of truth of the event that both propositions are true, such t-norms generalize the classical conjunction. 
Similarly, a fuzzy implication generalizes the classical implication. 
Most of these operators are differentiable, which enable their use in \dfl. 
Interestingly, the derivatives of these operators determine how \dfl corrects the deep learning model when its predictions are inconsistent with the background knowledge. 
We show that the qualitative properties of these derivatives are integral to both the theory and practice of \dfl. 
We approach this problem both from the view of symbolic and of statistical approaches to AI, to bridge the conceptual gap between those views. 
This provides insights that otherwise would be overlooked.  


\subsection{Contributions}
 The main contribution of this article is to answer the following question: \emph{``What fuzzy logic operators for existential quantification, universal quantification, conjunction, disjunction and implication have convenient theoretical properties when using them in gradient descent?''}
We analyze both theoretically and empirically the effect of the choice of operators used to compute the logical connectives in \dfuzz on the learning behaviour of a \dfl system. To this end, 
\begin{itemize}
    \item We introduce \dfuzz (Section \ref{sec:reallogic}) which combines fuzzy logic and gradient-based learning, and analyze its behaviour over different choices of fuzzy logic operators  (Section \ref{sec:background}).
    \item We analyze the theoretical properties of aggregation functions, which are used to compute the universal quantifier $\forall$ and the existential quantifier $\exists$, t-norms and t-conorms which are used to compute the connectives $\wedge$ and $\vee$, and fuzzy implications which are used to compute the connective $\rightarrow$.
    \item We introduce a new family of fuzzy implications called sigmoidal implications (Section \ref{sec:connectives}) using the insights from these analyses.
    \item We perform experiments to compare fuzzy logic operators in a semi-supervised experiment (Section \ref{chapter:experiments}).
    \item We give several recommendations for choices of operators.
\end{itemize}

\section{\dlogic}
\label{sec:ldr}

Loss functions are real-valued functions that represent a cost and must be minimized.
\dlogic (\dl) are logics for which differentiable loss functions are constructed that compute the truth value of given formulas using the semantics of the logic. 
These logics use background knowledge to deduce the truth value of statements in unlabeled or poorly labeled data, allowing us to use such data during learning, possibly together with normal labeled data. 
This can be beneficial as unlabeled, poorly labeled and partially labeled data is cheaper and easier to come by. 
This approach differs from Inductive Logic Programming \pcite{MUGGLETON1994629} which derives formulas from data. 
\dl instead informs what the data could have been. 

We motivate the use of \dl with the following classification scenario we consider throughout our analysis. 
Assume we have an agent $A$ whose goal is to describe the scene on an image. 
It gets feedback from a supervisor $S$, 
who does not have an exact description of these images available. However, $S$ does have a background knowledge base $\corpus$ about the concepts contained on the images. The intuition behind \dlogic is that $S$ can correct $A$'s descriptions of scenes when they are not consistent with the knowledge base $\corpus$.
\begin{figure}
    \centering
    \includegraphics[width=0.33\textwidth]{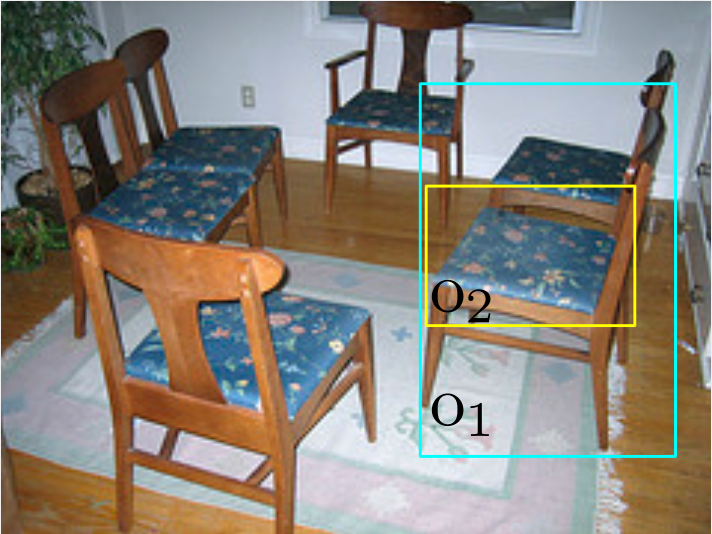}
    \caption{In this running example, we have an image with two objects on it, $o_1$ and $o_2$. }
    \label{fig:chair2}
\end{figure}
\begin{exmp}
\label{exmp:diffreason}
We illustrate this idea with the following example. Agent $A$ has to describe image $I$ in Figure \ref{fig:chair2} that contains objects $o_1$ and $o_2$. 
$A$ and the supervisor $S$ consider the unary class predicates $\{\pred{chair}, \pred{cushion}, \pred{armRest}\}$ and the binary predicate $\{\pred{partOf}\}$. 
Since $S$ does not have a description of $I$, it will have to correct $A$ based on the knowledge base $\corpus$. 
$A$ describes the image by using a confidence value in $[0, 1]$ for each observation. 
For instance, $p(\pred{chair}(o_1))$ indicates the confidence $A$ assigns to $\pred{chair}(o_1)$, i.e., whether $o_1$ is a chair or not.  
\begin{align*}
p(\pred{chair}(o_1))&=0.9  &p(\pred{chair}(o_2))&=0.4\\
p(\pred{cushion}(o_1))&=0.05 & p(\pred{cushion}(o_2))&=0.5\\
p(\pred{armRest}(o_1))&=0.05 & p(\pred{armRest}(o_2))&=0.1\\
p(\pred{partOf}(o_1, o_1)) &= 0.001 & p(\pred{partOf}(o_2, o_2)) &= 0.001\\
p(\pred{partOf}(o_1, o_2)) &= 0.01 & p(\pred{partOf}(o_2, o_1)) &= 0.95
\end{align*}
Suppose that $\corpus$ contains the following logic formula which says parts of a chair are either cushions or armrests:
\begin{equation*}
    \forall x, y\ \pred{chair}(x) \wedge \pred{partOf}(y, x) \rightarrow \pred{cushion}(y) \vee \pred{armRest}(y).
\end{equation*}
$S$ might reason that since $A$ is relatively confident of $\pred{chair}(o_1)$ and $\pred{partOf}(o_2, o_1)$ that the antecedent of this formula is satisfied, and either $\pred{cushion}(o_2)$ or $\pred{armRest}(o_2)$ has to hold. Since $p(\pred{cushion}(o_2)) > p(\pred{armRest}(o_2))$, a possible correction would be to tell $A$ to increase its degree of belief in $\pred{cushion}(o_2)$.
\end{exmp}

We would like to automate the kind of supervision $S$ performs in the previous example. 
To this end, we study what we call as \textit{\dfuzz} (\dfl), a family of \dlogic in the literature based on fuzzy logic. 
Here, the term ``\dlogic'' refers to a logic together with a translation scheme from logical expressions to differentiable loss functions. Then, ``\dfuzz''  stands for the case where the logic is a fuzzy logic and the translation scheme applies to logical expressions which include fuzzy operators. 
Note that due to its numerical character, fuzzy logic is an obvious candidate for a differentiable logic. 
Therefore, in \dfl, truth values of ground atoms are numbers in $[0, 1]$, and logical connectives are interpreted using some function over these truth values. 
Examples of logics in this family are Logic Tensor Networks \pcite{Serafini2016}, the similarly named Deep Fuzzy Logic \pcite{marra2019learning}, and the logics underlying Semantic Based Regularization \pcite{diligenti2017}, LYRICS \pcite{marra2018} and KALE \pcite{guo2016}, which we compare in Section \ref{sec:rel-dfuzz}. This family stands orthogonal to the well-studied mathematical  classification of the fuzzy logics landscape  \citep{cintula2011handbook}. Instead, in our analysis, we use a variety of individual T-norms with different properties 
combined with a variety of aggregation functions.   

\section{Background}
\label{sec:background}

We assume basic familiarity on the syntax and  the semantics of first-order logic. We shall denote predicates using the sans serif font (e.g., $\pred{cushion}$), a set $V$ of variables denoted by $x, y, z, \ldots$ or $x_1, x_2, x_3, \ldots$, and a set $O$ of domain objects denoted by $o_1, o_2, \ldots$, and constants $a, b, c, \ldots$. 
We limit ourselves to function-free formulas in \textit{prenex normal form} that start with quantifiers followed by a quantifier-free subformula. An example of a formula in prenex form is $\forall x, y\ \pred{P}(x, y) \wedge \pred{Q}(x) \rightarrow \pred{R}(y) $. An \textit{atom} is $\predP(t_1, ..., t_m)$ where $t_1, ..., t_m$ are terms. If $t_1, ..., t_m$ are all constants, we say it is a \textit{ground atom}.

Fuzzy logic is a many-valued logic where truth values are real numbers in $[0, 1]$ where 0 denotes completely false and 1 denotes completely true. 
It is often used to model reasoning in the presence of \emph{vagueness} i.e., without sharp boundaries or to imprecisely classify concepts such as a \emph{tall person} or a \emph{small number} \pcite{hajek1998metamathematics,Nov_k_1999}.
We will look at predicate fuzzy logics in particular, which extend propositional fuzzy logics with universal and existential quantification. 
For a brief background on fuzzy logic operators, see \ref{sec:background-operators}, and for an extensive treatment on mathematical fuzzy logic we refer the reader to standard textbooks which include \pcite{hajek1998metamathematics}, \pcite{Nov_k_1999} and \pcite{cintula2011handbook} .

\begin{table}[]
    \centering
    \begin{tabular}{lll}
    \hline
    Name          &  T-norm & Properties\\ \hline 
    Gödel (minimum) & $T_G(a, b) = \min(a, b)$ & idempotent, continuous \\ 
    Product       & $T_P(a, b) = a\cdot b$ & strict \\ 
    \luk          & $T_{LK}(a, b) = \max(a + b - 1, 0)$ & continuous \\ 
    Drastic product & $T_D(a, b) = \begin{cases}
        \min(a, b), & \text{if } a =1 \text{ or } b=1 \\
        0, & \text{otherwise}
        \end{cases}$ & \\
    Nilpotent minimum & $T_{nM}(a, b) = \begin{cases}
        0, & \text{if } a + b \leq 1 \\
        \min(a, b), & \text{otherwise}
      \end{cases}$ & left-continuous \\
    Yager & $T_Y(a, b) = \max(1 - ((1-a)^p+(1-b)^p)^{\frac{1}{p}}, 0), p \geq 1$ & continuous \\ \hline
    \end{tabular}
    \caption{The t-norms of interest.}
    \label{tab:tnorms}
\end{table}
\begin{table}[]
    \centering
    \begin{tabular}{lll}
    \hline
    Name          &  T-conorm & Properties\\ \hline 
    Gödel (maximum) & $S_G(a, b) = \max(a, b)$ & idempotent, continuous \\ 
    Product (probabilistic sum)       & $S_P(a, b) = a + b - a \cdot b$ & strict \\ 
    \luk          & $S_{LK}(a, b) = \min(a + b, 1)$ & continuous \\ 
    Drastic sum & $S_D(a, b) = \begin{cases}
        \max(a, b), & \text{if } a =0 \text{ or } b=0 \\
        1, & \text{otherwise}
        \end{cases}$ & \\ 
    Nilpotent maximum & $S_{nM}(a, b) = \begin{cases}
        1, & \text{if } a + b \geq 1 \\
        \max(a, b), & \text{otherwise}
      \end{cases}$ & right-continuous \\ 
   Yager & $S_Y(a, b) = \min((a^p+b^p)^{\frac{1}{p}}, 1), p \geq 1$ & continuous \\ \hline
    \end{tabular}
    \caption{The t-conorms of interest.}
    \label{tab:snorms}
\end{table}
In this paper we exclusively use the \emph{strong negation} $N(a) = 1 - a$.
Conjunctions $a\wedge b$ are generalized using \emph{T-norms}, which are any function $T:[0, 1]^2 \rightarrow  [0, 1]$ that is commutative, associative, monotonic, and has $T(1, a)=a$. 
In fuzzy logic, conjunctions using the t-norm are often denoted on the logical formula level as $a \otimes b$.
However, on the semantic level (i.e. $T(a, b)$), $a$ and $b$ will refer to the truth values of the corresponding logical formulas $a$ and $b$. 
An overview of the most common T-norms is given in Table \ref{tab:tnorms} alongside common properties that are defined in \ref{appendix:t-norms}. 
\emph{T-conorms} generalize disjunctions and are often denoted $a\oplus b$ in fuzzy logic literature. They are formed from a t-norm $T$ using $S(a, b) = 1 - T(1 - a, 1 - b)$, following DeMorgans law $a\oplus b = \neg (\neg a \otimes \neg b)$. The most common t-conorms are given in Table \ref{tab:snorms}. 

\sloppy Universal and existential quantification are generalized using increasing \emph{aggregation operators} $A, E\in \bigcup_{n\in\mathbb{N}} [0, 1]^n \rightarrow [0,1]$, where $A(1, ..., 1)=E(1, ..., 1)=1$ and $A(0, ..., 0)=E(0, ..., 0)=0$ \pcite{calvoAggregationOperatorsProperties2002}. 
Universal aggregators can be constructed from a t-norm $T$ and existential aggregators from a t-conorm $S$:
\begin{align}
A_T()&= 1 \quad & E_S() &= 0 \\
A_T(x_1, ..., x_n) &= T(x_1, A_T(x_2, ..., x_n)) \quad & E_S(x_1, ..., x_n) &= S(x_1, E_S(x_2, ..., x_n))
\end{align}
An overview of common aggregation operators is given in Table \ref{tab:aggregation}. 
A formal introduction is in \ref{appendix:aggregation}.
\begin{table}[]
    \centering
    \begin{tabular}{llll}
    \hline
    Name          &  Generalizes & Aggregation operator \\ \hline 
    Minimum & $T_G$ & $A_{T_G}(x_1, ..., x_n) = \min(x_1, ..., x_n)$  \\ 
    Product & $T_P$ & $A_{T_P}(x_1, ..., x_n) = \prod_{i=1}^n x_i$  \\ 
    \luk  & $T_{LK}$ & $A_{T_{LK}}(x_1, ..., x_n) = \max(\sum_{i=1}^n x_i - (n - 1), 0)$ \\ 
    Maximum & $S_G$
    & $E_{S_G}(x_1, ..., x_n) = \max(x_1, ..., x_n)$ 
     \\ 
    Probabilistic sum & $S_G$ & $E_{S_P}(x_1, ..., x_n) = 1 - \prod_{i=1}^n(1 - x_i)$  \\ 
    Bounded sum & $S_{LK}$ & $E_{S_{LK}}(x_1, ..., x_n) = \min\left(\sum_{i=1}^n x_i, 1\right)$ \\ \hline
    \end{tabular}
    \caption{Some common aggregation operators.}
    \label{tab:aggregation}
\end{table}
\begin{table}[]
    \centering
    \begin{tabular}{lllll}
    \hline
    Name          &  T-conorm & S-implication & Properties\\ \hline 
    Gödel (Kleene-Dienes) & $S_G$ & $I_{KD}(a, c) = \max(1-a, c)$ & All but IP \\
    Product (Reichenbach)      & $S_P$ & $I_{RC}(a, c) = 1 - a + a\cdot c$ & All but IP \\ 
    \luk                  & $S_{LK}$ & $I_{LK}(a, c) = \min(1-a+c, 1)$ & All\\ 
    Dubouis-Prade & $S_D$ & $I_{DP}(a, c) = \begin{cases}
        c, & \text{if } a = 1 \\
        1-a, & \text{if } c = 0 \\
        1, & \text{otherwise}
    \end{cases}$ & All \\ 
    Nilpotent (Fodor)     & $S_{Nm}$ & $I_{FD}(a, c) = \begin{cases}
        1, & \text{if } a \leq c \\
        \max(1 - a, c), & \text{otherwise}
      \end{cases}$ & All \\ \hline
    
    \end{tabular}
    \caption{S-implications formed from $\neg a\oplus c$ with the common t-conorms from Table \ref{tab:snorms}.}
    \label{tab:simplications}
\end{table}

Finally, implications are generalized using \emph{fuzzy implications} \pcite{Jayaram2008}, which are functions $I: [0, 1]^2\rightarrow [0, 1]$ that are increasing with respect to the first argument, and decreasing with respect to the second. 
Furthermore, we assume boundary conditions $I(0, 0)=I(1, 1)=1$, and $I(1, 0)=0$, which follows from the classical implication.
Fuzzy implications can have additional properties that we discuss in \ref{sec:fuzz_imp}. 
We will discuss two classes of fuzzy implications. 
The first is \emph{S-implications} (\ref{appendix:s-implications}, examples in Table \ref{tab:simplications}), which are formed using a t-conorm by generalizing the material implication $a\rightarrow c = \neg a \oplus c$ using $I_S(a, c) = S(N(a), c)$.
The second is \emph{R-implications}, which is the standard choice in t-norm fuzzy logics (\ref{appendix:r-implications}, examples in Table \ref{tab:rimplications}).
These are constructed from t-norms using $I_T(a, c)=\sup_{b\in [0, 1]} T(a, b) \leq c$.

\begin{table}[]
    \centering
    \begin{tabular}{lllll}
    \hline
    Name          &  T-norm & R-implication & Properties\\ \hline 
    Gödel  & $T_G$ & $I_G(a, c) =\begin{cases}
        1, & \text{if } a \leq c \\
        c, & \text{otherwise}
      \end{cases}$ & LN, EP, IP \\
    product (Goguen) & $T_P$ & $I_{GG}(a, c) =\begin{cases}
        1, & \text{if } a \leq c \\
        \frac{c}{a}, & \text{otherwise}
      \end{cases}$ & LN, EP, IP \\
    \luk           & $T_{LK}$ & $I_{LK}(a, c) = \min(1-a+c, 1)$ & All\\
    Weber & $T_D$ & $I_{WB}(a, c) = \begin{cases}
        1, & \text{if } a < 1 \\
        c, & \text{otherwise}
    \end{cases}$ & LN, EP, IP \\ 
    Nilpotent (Fodor) & $T_{Nm}$ & $I_{FD}(a, c) = \begin{cases}
        1, & \text{if } a \leq c \\
        \max(1 - a, c), & \text{otherwise}
      \end{cases}$ & All \\ \hline
    \end{tabular}
    \caption{The R-implications constructed using the t-norms from Table \ref{tab:tnorms}.}
    \label{tab:rimplications}
\end{table}

\label{chapter:theory}

\section{Differentiable Fuzzy Logics}


\label{sec:reallogic}
As mentioned earlier,  \textit{\dfuzz} (\dfl) are \dlogic based on fuzzy logic. Truth values of ground atoms are continuous, and logical connectives are interpreted using differentiable fuzzy operators. In principle, \dfl can handle both predicates and functions. To ease the discussion, we will not analyze functions and constants and leave them out of the discussion.\footnote{Functions and constants are modelled in \tcite{Serafini2016} and \tcite{marra2018}.} 

\subsection{Semantics}
\dfl defines a new semantics using vector embeddings and functions on such vectors in place of classical semantics. In classical logic, a \textit{structure} consists of a domain of discourse and an interpretation function, and is used to give meaning to the predicates. \dfl defines \textit{structures} using \textit{embedded interpretations}\footnote{\tcite{Serafini2016} uses the term \say{(semantic) grounding} or \say{symbol grounding} \pcite{Mayo2003} instead of `embedded interpretation', \say{to emphasize the fact that $\fol$ is interpreted in a `real world'} but we find this potentially confusing as this could also refer to groundings in Herbrand semantics. Furthermore, by using the word `interpretation' we highlight the parallel with classical logical interpretations.} instead:

\begin{deff}
\label{deff:distr_inter}

    A \textit{\dfuzz structure} is a tuple $\mathcal{S}=\langle \objects, \eta, \btheta \rangle$, where 
    $\objects$ is a finite but unbounded set called \emph{domain of discourse} and every $o\in \objects$ is a $d$-dimensional\footnote{Without loss of generality we fix the dimensionality of the vectors representing the objects. Extensions to a varying number of dimensions are straightforward by introducing types, such as done in \cite{badreddineLogicTensorNetworks2020}.} vector, 
    $\eta: \predicates \times \mathbb{R}^W \rightarrow ( \objects^m \rightarrow [0, 1])$ is an (\emph{embedded}) \emph{interpretation}, and $\btheta \in \mathbb{R}^W$ are \emph{parameters}.
    $\eta$ maps predicate symbols $\pred{P} \in \predicates$ with arity $m$ to a function of $m$ objects to a truth value $[0, 1]$. That is, $\eta(\pred{P}, \theta): \objects^m\rightarrow [0, 1]$. We will use the notation $\interpretation(\pred{P})$ to denote $\eta(\pred{P}, \btheta)$. 
\end{deff}

To address the \textit{symbol grounding problem} \pcite{harnad1990symbol}, objects in the domain of discourse are $d$-dimensional vectors of reals. 
Their semantics come from the underlying semantics of the vector space as terms are interpreted in a real (valued) world \pcite{Serafini2016}. 
Predicates are interpreted as functions mapping these vectors to a fuzzy truth value. 
Embedded interpretations can be implemented using neural network models\footnote{We use `models' to refer to deep learning models like neural networks, and not to models from model theory. } with trainable network parameters $\btheta$.
Note that different values for the parameters $\btheta$ will produce different \dfl structures. 
Next, we define the truth values of formulas in \dfl.%

\begin{deff}
\label{deff:val}


Let $\mathcal{S}=\langle \objects, \eta, \btheta\rangle$ be a \dfl structure, $\instantiation: V \rightarrow O$ a variable assignment, $N$ a fuzzy negation, $T$ a t-norm, $S$ a t-conorm, $I$ a fuzzy implication and $A$ and $E$ are aggregators. Then  we say that $\mathcal{S}$ satisfies the formula $\varphi \in \fol$ w.r.t. $\instantiation$ (i.e., $\mathcal{S}, \mu\models \varphi$)  in the degree of $\val(\varphi)$ (i.e., the truth value of $\varphi$) where  $\val: (V\rightarrow O)\times \mathcal{L}\rightarrow [0, 1]$ is the \textit{valuation function} defined inductively on the structure of $\varphi$ as follows:
\begin{align}
    \label{eq:rlpred}
    \val\left(\instantiation, \pred{P}(x_1, ..., x_m) \right) &= \interpretation(\pred{P})\left(\instantiation(x_1 ), ..., \instantiation(x_m )\right)\\
    \label{eq:rlneg}
    \val(\instantiation, \neg \phi ) &= N(\val(\instantiation, \phi ))\\
    \label{eq:rlconj}
    \val(\instantiation, \phi \otimes \psi ) &= T(\val(\instantiation, \phi ), \val(\instantiation, \psi ))\\
    \val(\instantiation, \phi\oplus \psi ) &= S(\val(\instantiation, \phi ), \val(\instantiation, \psi ))\\
    \label{eq:rlimp}
    \val(\instantiation, \phi\rightarrow\psi ) &= I(\val(\instantiation, \phi ), \val(\instantiation, \psi ))\\
    \label{eq:rlaggr}
    \val(\instantiation, \forall x\ \phi ) &= \aggregate_{o\in \objects} \val(\instantiation\cup \{(x, o)\}, \phi) \\
    \label{eq:rleaggr}
    \val(\instantiation, \exists x\ \phi ) &= \Eaggregate_{o\in \objects} \val(\instantiation\cup \{(x, o)\}, \phi)
\end{align}
\end{deff}


Equation \ref{eq:rlpred} defines the fuzzy truth value of an atomic formula. $\instantiation$   assigns objects to the terms $x_1, ..., x_m$ resulting in a list of $d$-dimensional vectors. These are the inputs to the interpretation $\interpretation$ of the predicate symbol $\pred{P}$ (i.e., $\interpretation(\pred{P})$) to get a fuzzy truth value. Equations \ref{eq:rlneg} - \ref{eq:rlimp} define the truth values of the connectives using the operators $N, T, S$ and $I$.
Equation \ref{eq:rlaggr} and \ref{eq:rleaggr} define the truth value of universally quantified formulas $\forall x\ \phi$ and existentially quantified formulas $\exists x\ \phi$. This is done by enumerating the domain of discourse $o\in\objects$, evaluating the truth value of $\phi$ with $o$ assigned to $x$ in $\mu$, and combining the truth values using an aggregation operator $A$. 

Note that our assumption on the finiteness of the domain is pragmatic: It reflects the finiteness of the data in machine learning settings. 
Hence, many fundamental results in the realm of mathematical (fuzzy) logic will not hold in general for the logic we defined \pcite{cintula2011handbook}. 

\subsection{Relaxing Quantifiers}
For infinite domains, or for domains that are so large that we cannot compute the full semantics of the $\forall$ and $\exists$ quantifiers, we can choose to sample a batch  $\batch$  of objects from $\objects$ to approximate the computation of the valuation. This can be done by replacing Equation \ref{eq:rlaggr} with 

\begin{equation}
    \label{eq:aggrsample}
    \val(\mu, \forall x\ \phi) = \aggregate_{i=1}^\batch \val(\mu \cup \{(x, o_i)\}, \phi),\quad o_1, ..., o_\batch \text{ chosen from } \objects.
\end{equation}

Choosing the batch of objects can be done in several ways. 
One approach would be to sample from a real-world distribution over the domain of discourse $\objects$, if available. 
For example, the domain of discourse might be the natural images, and the real-world distribution would be the distribution over natural images.
A more common approach is to assume access to a dataset $\dataset$ of independent samples from such a distribution \cite{goodfellow2016deep}(p.109) and to choose minibatches from this dataset.
Note that by relaxing quantifiers using sampling we lose the soundness of our  computation, as different batches will have different truth values for the formulas. 

\subsection{Learning using Fuzzy Maximum Satisfiability}
In \dfl, \textit{fuzzy maximum satisfiability} \pcite{Donadello2017} is the problem of finding parameters $\btheta$ that maximize the valuation of the knowledge base $\corpus$.

\begin{deff}
Let $\corpus$ be a knowledge base of formulas, $\mathcal{S}$ a \dfl structure for $\corpus$ and $\val$  a valuation function. 
Then the \textit{\dfuzz loss} $\loss_{\dfl}$ of a knowledge base of formulas $\corpus$ is computed using
\begin{equation}
\label{eq:lossrl}
    \loss_{\dfl}(\mathcal{S}, \corpus) = \loss_{\dfl}(\langle \mathcal{O}, \eta, \btheta \rangle, \corpus)  = -\sum_{\varphi\in\corpus} \val(\{\}, \varphi).
\end{equation}
The \textit{fuzzy maximum satisfiability problem} is the problem of finding parameters $\btheta^*$ that minimize Equation \ref{eq:lossrl}:
\begin{equation}
\label{eq:bestsatproblem}
    \btheta^* = \text{argmin}_{\btheta}\ \loss_{\dfl}(\mathcal{S}, \corpus).
\end{equation}
\end{deff}

This optimization problem can be solved using a gradient descent method. If the operators $N, T, S, I, A$ and $E$ are all differentiable, we can repeatedly apply the chain rule, i.e. reverse-mode differentiation, on the \dfl loss $\loss_{\dfl}(\mathcal{S}, \corpus)$. 
This procedure finds the derivative with respect to the truth values of the ground atoms $\frac{\partial \loss_{\dfl}(\mathcal{S}, \corpus)}{\partial \interpretationfol_{\btheta}(\predP)(o_1, ..., o_m)}$. 
We can use these partial derivatives to update the parameter $\btheta_n$ at iteration $n$ of the optimization process again using the chain rule, resulting in a different embedded interpretation $\interpretationfol_{\btheta_{n+1}}$.
This procedure is computed as follows for the $i$th parameter:

\begin{equation}
\label{eq:grad_desc}
    \btheta_{n+1, i} = \btheta_{n, i} - \epsilon \cdot \frac{\partial \loss_{\dfl}(\mathcal{S}_n, \corpus)}{\partial \btheta_{n, i}} = \btheta_{n, i} - \epsilon \cdot \sum_{\predP(o_1, ..., o_m)}\frac{\partial \loss_{\dfl}(\mathcal{S}_n, \corpus)}{\partial \interpretation(\predP)(o_1, ..., o_m)} \cdot \frac{\partial \interpretation(\predP)(o_1, ..., o_m)}{\partial \btheta_{n, i}},
\end{equation}

where $\epsilon$ is the learning rate. Note that the parameters $\btheta_n$ are implicitly passed to $\loss$ through the structure $\mathcal{S}_n=\langle \objects, \eta, \btheta_n\rangle$. We refer for implementation details to \Appendix \ref{sec:implementation}.

\begin{exmp}
\label{exmp:reallogic}
To illustrate the computation of the valuation function $\val$, we return to the problem in Example \ref{exmp:diffreason}. 
The \textit{domain of discourse} is the set of objects on natural images. 
We have access to a dataset of two objects $\dataset=\{o_1, o_2\}$. 
The valuation of the formula $\varphi = \forall x, y\ \pred{chair}(x) \otimes \pred{partOf}(y, x) \rightarrow \pred{cushion}(y) \oplus \pred{armRest}(y)$ is
\begin{align*}
    \val(\mu, \varphi) = A(A(I(&T(\interpretation(\pred{chair})(o_1), \interpretation(\pred{partOf})(o_1, o_1)), S(\interpretation(\pred{cushion})(o_1), \interpretation(\pred{armRest})(o_1))),\\ I(&T(\interpretation(\pred{chair})(o_1), \interpretation(\pred{partOf})(o_2, o_1)), S(\interpretation(\pred{cushion})(o_2), \interpretation(\pred{armRest})(o_2)))),\\
    A(I(&T(\interpretation(\pred{chair})(o_2), \interpretation(\pred{partOf})(o_1, o_2)), S(\interpretation(\pred{cushion})(o_1), \interpretation(\pred{armRest})(o_1))),\\ I(&T(\interpretation(\pred{chair})(o_2), \interpretation(\pred{partOf})(o_2, o_2)), S(\interpretation(\pred{cushion})(o_2), \interpretation(\pred{armRest})(o_2))))).
\end{align*}
Next, we choose the operators as $T=T_P$, $S = S_P$,  $A=A_{T_P}$ and $I=I_{RC}$, such that
\begin{align}
    \val(\mu, \varphi) =
    \label{eq:exmp:reallogic}
    & \prod_{x, y\in \constants} 1 - \interpretation(\pred{chair})(x) \cdot \interpretation(\pred{partOf})(y, x) \cdot (1 - \interpretation(\pred{cushion})(y))(1 - \interpretation(\pred{armRest})(y)) \notag.
\end{align}
If we interpret the predicate functions using the confidence values from Example \ref{exmp:diffreason} so that $\interpretation(\pred{P}(x)) = p(\pred{P}(x))$, we find that $\val(\varphi) = 0.612$. Taking $\corpus = \{\varphi\}$, we find using the chain rule that 
\begin{align*}
\frac{\partial \loss_{\dfl}( \mathcal{S}, \corpus)}{\partial \interpretation(\pred{chair})(o_1)}&= -0.4261 &\frac{\partial \loss_{\dfl}( \mathcal{S}, \corpus)}{\partial \interpretation(\pred{chair})(o_2)}&=-0.0058\\
\frac{\partial \loss_{\dfl}( \mathcal{S}, \corpus)}{\partial \interpretation(\pred{cushion})(o_1)}&=0.0029 & \frac{\partial \loss_{\dfl}( \mathcal{S}, \corpus)}{\partial \interpretation(\pred{cushion})(o_2)}&=0.7662\\
\frac{\partial \loss_{\dfl}( \mathcal{S}, \corpus)}{\partial \interpretation(\pred{armRest})(o_1)}&=0.0029 & \frac{\partial \loss_{\dfl}( \mathcal{S}, \corpus)}{\partial \interpretation(\pred{armRest})(o_2)}&=0.4257\\
\frac{\partial \loss_{\dfl}( \mathcal{S}, \corpus)}{\partial \interpretation(\pred{partOf})(o_1, o_1)} &= -0.4978 & \frac{\partial \loss_{\dfl}( \mathcal{S}, \corpus)}{\partial \interpretation(\pred{partOf})(o_2, o_2)} &= -0.1103\\
\frac{\partial \loss_{\dfl}( \mathcal{S}, \corpus)}{\partial \interpretation(\pred{partOf})(o_1, o_2)} &= -0.2219 & \frac{\partial \loss_{\dfl}( \mathcal{S}, \corpus)}{\partial \interpretation(\pred{partOf})(o_2, o_1)} &= -0.4031.
\end{align*}
We can now do a gradient update step to update the confidence values from Example \ref{exmp:diffreason}, or find what the partial derivative of the parameters $\btheta$ of some deep learning model $p_{\btheta}$ should be using Equation \ref{eq:grad_desc}.

One particularly interesting property of \dfuzz is that the partial derivatives of the subformulas with respect to the satisfaction of the knowledge base have a somewhat explainable meaning. For example, as hypothesized in Example \ref{exmp:diffreason}, the computed partial derivatives reflect whether we should increase $p(\pred{cushion}(o_2))$, as it is indeed the (\textbf{absolute}) largest partial derivative. 

\end{exmp}

\section{Derivatives of Operators}
\label{sec:connectives}


We will now show that the choice of operators that are used for the logical connectives actually determines the inferences that are done when using \dfl. If we used a different set of operators in Example \ref{exmp:reallogic}, we would have gotten very different derivatives. These could in some cases make more sense, and in some other cases less. Furthermore, it is much easier to find a global minimum of the fuzzy maximum satisfiability problem (Equation \ref{eq:bestsatproblem}) for some operators than for others. This is often because of the smoothness of the operators. In this section, we analyze a wide variety of functions that can be used for logical reasoning and present some of their properties that determine how useful they are in inferences such as those illustrated above.

We will not discuss any varieties of fuzzy negations since the strong negation $N_C(a) = 1-a$ is already continuous, intuitive and has simple derivatives. 

\begin{deff}
A function $f: \mathbb{R}\rightarrow \mathbb{R}$ is said to be \emph{vanishing} if there are $a, b\in \mathbb{R}$, $a < b$ such that for all $c\in (a, b)$, $f(c) = 0$, i.e. there is an interval for which the function is 0. Otherwise, the function is \emph{nonvanishing}.\\
A function $f: \mathbb{R}^n \rightarrow \mathbb{R}$ has a \emph{vanishing derivative} if for all $a_1, ..., a_n\in \mathbb{R}$ there is some $1\leq i\leq n$ such that $\frac{\partial f(a_1, ..., a_n)}{\partial a_i}$ is vanishing.
\end{deff}

Whenever the derivative of an operator vanishes, it loses its learning signal. This definition does not include functions that only pass through 0, such as when using the product t-conorm for $a\oplus \neg a$, where we find that the derivative of $S_P(a, 1-a)$ is 0 only at $\frac{1}{2}$. 
Furthermore, all the partial derivatives of the connectives used in the backward pass from the valuation function to the ground atoms have to be multiplied. 
If the partial derivatives are less than 1, their product will also approach 0. 
This can happen for instance with a large sequence of conjunctions using the product t-norm.

The drastic product $T_D$ and operators derived from it such as the drastic sum $S_D$ and the Dubois-Prade and Weber implications ($I_{DP}$ and $I_{WB}$) have vanishing derivatives almost everywhere. 
The output confidence values of deep learning models are the result of transformations on real numbers using functions like the sigmoid or softmax that result in truth values in $(0, 1)$. The operators derived from $T_D$ only have nonvanishing derivatives when their inputs are exactly $0$ or $1$, invalidating their use in this application

\begin{deff}
A function $f: \mathbb{R}^n\rightarrow \mathbb{R}$ is said to be \textit{single-passing} if it has nonzero derivatives on at most one input argument. That is, for all $x_1, ..., x_n\in[0, 1]$ it holds that $\left|\left\{i\middle|\frac{\partial f(x_1, ..., x_n)}{\partial x_i}\neq 0, i \in \{1, ..., n\}\right\}\right| \leq 1$.
\end{deff}

Using just single-passing Fuzzy Logic operators can be inefficient, since then at most one input will have a nonzero derivative (i.e. a learning signal), yet the complete forward pass still has to be computed to find this input. 
In particular, this will hold when choosing operators based on the Gödel t-norm.
\begin{prop}
Any composition of single-passing functions is also single-passing.
\end{prop}
For the proof, see \ref{appendix:singlepassing}.

Concluding, for any logical operator to be usable in the learning task, it will need to have a nonvanishing derivative at the majority of the input signals, so it can contribute to the learning signal at all, and ideally not be single-passing so that it can contribute effectively to the learning signal.

\section{Aggregation}
\label{sec:aggr}


After the global considerations from the previous section, we next analyze in detail aggregation operators for universal and existential quantification separately and outline their benefits and disadvantages in \dfl.

\subsection{Minimum and Maximum Aggregators}
\label{sec:aggr_min}
The minimum aggregator is given as $A_{T_G}(x_1, ..., x_n) = \min(x_1, ..., x_n)$. The partial derivatives are 
\begin{equation}
    \frac{\partial A_{T_G}(x_1, ..., x_n)}{\partial x_i} = \begin{cases}
    1 & \text{ if } i = \argmin_{j\in\{1, ..., n\}} x_j\\
    0 & \text{ otherwise}.
    \end{cases}
\end{equation}

It is single-passing with the only nonzero gradient being on the input with the lowest truth value. Many practical formulas have exceptions. An exception to a formula like $\forall x\ \pred{Raven}(x)\rightarrow \pred{Black}(x)$ would be a raven over which a bucket of red paint is thrown. The minimum aggregator would have a derivative on that exception when `red raven' is correctly predicted. Additionally, it is inefficient, as we still have to compute the forward pass for inputs that do not get a feedback signal.

The partial derivatives of the maximum aggregator $E_{S_G}(x_1, ..., x_n)=\max(x_1, ..., x_n)$ are similar, but increase the input with the highest truth value instead. 
This can be a reasonable aggregator for existential quantification, as it will reinforce the belief that the input we are most confident about is correct will make the existential quantifier true.
A downside is that it can only consider one such input, despite the fact that the condition might hold for multiple inputs.

\subsection{\luk\ Aggregator}
\label{sec:aggrluk}
The \luk \ aggregator is given as $A_{T_{LU}}(x_1, ..., x_n) = \max\left(\sum_{i=1}^n x_i - (n-1), 0\right).$ The partial derivatives are given by
\begin{equation}
    \frac{\partial A_{T_{LU}}(x_1, ..., x_n) }{\partial x_i} = \begin{cases}
    1 & \text{ if } \sum_{i=1}^n x_i > n-1\\
    0 & \text{ otherwise. }
    \end{cases}
\end{equation}
The gradient is nonvanishing only when $\sum_{i=1}^n x_i > n-1$, i.e., when the average value of $x_i$ is larger than $\frac{n-1}{n}$ \pcite{PallJonsson2018}. As $\lim_{n\rightarrow\infty} \frac{n-1}{n} = 1$, for larger values of $n$, all inputs have to be high for this to hold. 

For the next proposition, we refer to the \textit{fraction of inputs} for which some condition holds. The probability that the condition holds for a point uniformly sampled from $[0, 1]^n$ is this fraction.
\begin{prop}
\label{prop:luk}
The fraction of inputs $x_1, ..., x_n\in [0,1]$ for which the derivative of $A_{T_{LU}}$ is nonvanishing is $\frac{1}{n!}$.
\end{prop}
For proof, see \ref{appendix:lukfrac}. Clearly, for the majority of inputs there is a vanishing gradient, implying that this universal aggregator would not be useful in a \dfl learning setting. 

\sloppy A similar argument can be made for the existential \luk\ aggregator, the bounded sum ${E_{S_{LK}} (x_1, ..., x_n)} = \min(\sum_{i=1}^n x_i, 1)$, which will only have nonvanishing derivatives of 1 to all argument when the average value of $x_i$ is smaller than $\frac{1}{n}$. Like the \luk\ aggregator, it also has nonvanishing derivatives only on a fraction  $\frac{1}{n!}$ of its domain. 
This is therefore not a useful existential aggregator: The agent will learn nothing unless all inputs are close to 0. 

\subsection{Yager Aggregator}
\label{sec:aggr_yager}
The Yager universal aggregator is given by
\begin{equation}
    A_{T_Y}(x_1, ..., x_n) = \max\left(1-\left(\sum_{i=1}^n (1 - x_i)^p\right)^{\frac{1}{p}}, 0\right), \quad p> 0
\end{equation}
Here, $p=1$ corresponds to the \luk\ aggregator, $p\rightarrow \infty$ corresponds to the minimum aggregator, and $p\rightarrow 0$ corresponds to the aggregator formed from the drastic product $A_{T_D}$. The derivative of the Yager aggregator is
\begin{equation}
\label{eq:aggr_yager_deriv}
    \frac{\partial A_{T_Y}(x_1, ..., x_n)}{\partial x_i} = \begin{cases}
    \left(\sum_{j=1}^n (1 - x_j)^p\right)^{1-\frac{1}{p}}\cdot (1-x_i)^{p-1} & \text{ if } \left(\sum_{j=1}^n (1 - x_j)^p\right)^{\frac{1}{p}} < 1\\
    0 & \text{ if } \left(\sum_{j=1}^n (1 - x_j)^p\right)^{\frac{1}{p}} > 1.
    \end{cases}
\end{equation}
This derivative vanishes whenever  $\sum_{j=1}^n (1 - x_j)^p \geq 1$. As $1 - x_i \in [0, 1]$, $(1 - x_i)^p$ is a decreasing function with respect to $p$. Therefore, $\sum_{i=1}^n (1 - x_i)^p < 1$ holds for a larger fraction of inputs when $p$ increases, with the fraction being 0 for $p=0$ as it corresponds to the drastic aggregator, and 1 for $p=\infty$.

The exact fraction of inputs with a nonvanishing derivative is hard to express in the general case.\footnote{\label{footnote:yager_aggregator}Assume that $x_1, ..., x_n\sim U(0, 1)$ are independently and standard uniformly distributed. Note that $z_i=1-x_i$ is also standard uniformly distributed. $z_i^p$ is distributed by the beta distribution $Beta(1/p, 1)$ \pcite{gupta2004handbook}. $Y=\sum_{i=1}^n z_i$ is the sum of $n$ such beta-distributed variables. Unfortunately, there is no closed-form expression for the probability density function of sums of independent beta random variables \pcite{Pham1994a}. A suitable approximation would be to use the central limit theorem as $z_1, ..., z_n$ are identically and independently distributed.} However, we can find a closed-form expression for the Euclidean case $p=2$.
 \begin{prop}
 \label{prop:vol_yager}
 The fraction of inputs $x_1, ..., x_n\in[0, 1]$ for which the derivative of $A_{T_Y}$ with $p=2$ is nonvanishing is $\frac{\pi^{\frac{n}{2}}}{2^{n}\cdot \Gamma(\frac{1}{2}n + \frac{1}{2})}$, where $\Gamma$ is the Gamma function.
 \end{prop}
 \begin{figure}
     \centering
     \includegraphics[width=0.6\linewidth]{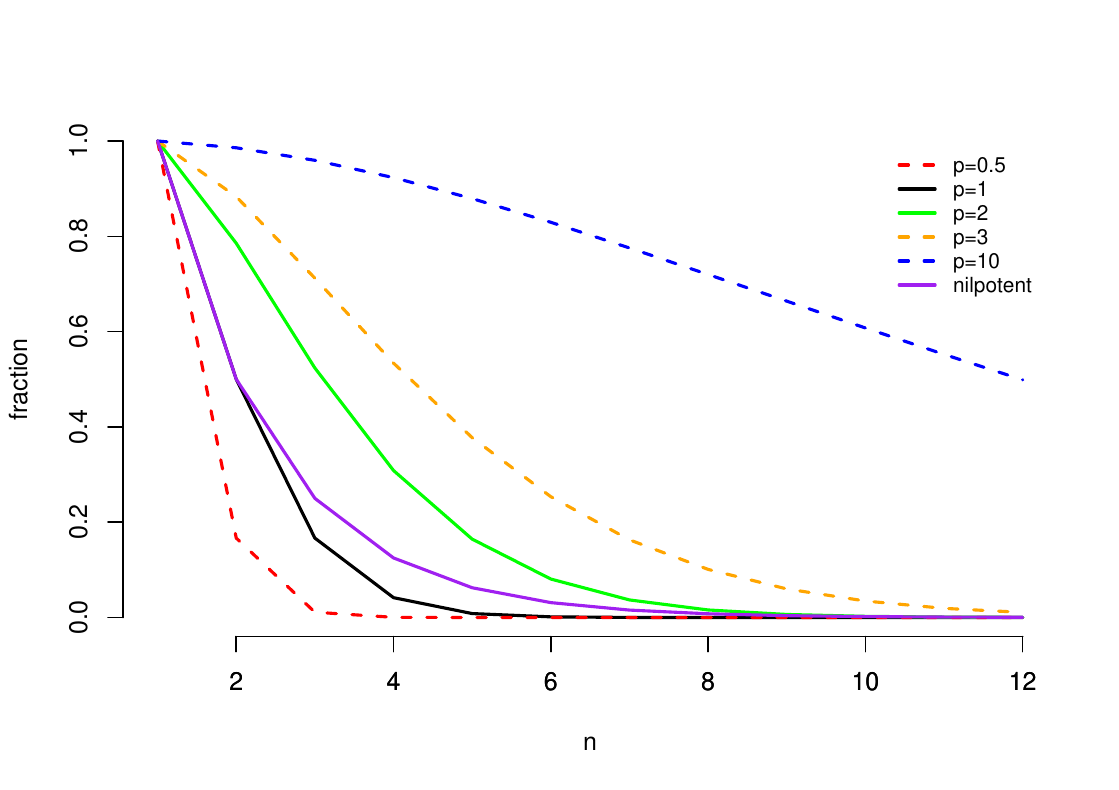}
     \caption{The fraction of inputs for which the Yager aggregator $A_{T_Y}$ for several values of $p$ and the nilpotent minimum aggregator $A_{T_{nM}}$ have nonvanishing derivatives. The values for dotted lines are estimated using Monte Carlo simulation.}
     \label{fig:ratio_2_yager}
 \end{figure}
 See \ref{appendix:yagerfrac} for the proof. We plot the fraction of nonvanishing derivatives for several values of $p$ in Figure \ref{fig:ratio_2_yager}. For fairly small $p$, the vast majority of the inputs will have a vanishing derivative, and similar for high $n$, showing that this aggregator is also of little use in a learning context.

 Similarly, the derivatives of the Yager existential aggregator $E_{S_Y}(x_1, ..., x_n) = \min\left(\left(\sum_{i=1}^n  x_i^p\right)^{\frac{1}{p}}, 1\right)$  are nonvanishing in the same fraction of inputs as the Yager universal aggregator.
\subsection{Generalized Mean and Generalized Mean Error}
If we are concerned only in maximizing the truth value of $A_{T_Y}$, we can simply remove the \textit{max} constraint, resulting in an aggregator that has a nonvanishing derivative everywhere.
However, then the co-domain of the function is no longer $[0, 1]$. We can do an affine transformation on this function to ensure this is the case (\ref{appendix:yager}), after which we obtain the \emph{Generalized Mean Error}. 

\begin{deff}
\label{deff:p-mean-error}
For any $p> 0$, the Generalized Mean Error $A_{GME}$ is defined as
\begin{equation}
    A_{GME}(x_1, ..., x_n) =1 - \left(\frac{1}{n}\sum_{i=1}^n(1 - x_i)^p\right)^{\frac{1}{p}}.
\end{equation}
\end{deff}
 The `error' here the is difference between the predicted value $x_i$ and the `ground truth' value, a truth value of 1. This function has the following derivative:
 \begin{equation}
 \label{eq:deriv_apme}
    \frac{\partial A_{GME}(x_1, ..., x_n)}{\partial x_i} =
    (1-x_i)^{p-1}\frac{1}{n}^{\frac{1}{p}}\left(\sum_{j=1}^n (1 - x_j)^p\right)^{\frac{1}{p} - 1}  .
\end{equation}
When $p>1$, this derivative is greatest for the inputs that are lowest, which can speed up the optimization by being sensitive to outliers. For $p < 1$, the opposite is true. A special case is $p=1$:
  \begin{equation}
     \label{eq:aggr_sos}
     A_{MAE}(x_1, ..., x_n) = 1-\frac{1}{n}\sum_{i=1}^n (1 - x_i)
 \end{equation}
 having the simple derivative $\frac{\partial A_{MAE}(x_1, ..., x_n)}{\partial x_i} = \frac{1}{n}$. This measure is equal to 1 minus the \textit{mean absolute error (MAE}) and is associated with the \luk\ t-norm.
 Another special case is $p=2$:
 \begin{equation}
     \label{eq:aggr_rmse}
     A_{RMSE}(x_1, ..., x_n) = 1-\sqrt{\frac{1}{n}\sum_{i=1}^n (1 - x_i)^2}.
 \end{equation}
This function is equal to 1 minus the \textit{root-mean-square error} (RMSE) which is commonly used for regression tasks and heavily weights outliers. We can do the same for the Yager existential aggregator  (\ref{appendix:yager}):
\begin{deff}
\label{deff:p-mean}
For any $p > 0$, the Generalized Mean is defined as
\begin{equation}
\label{eq:p-mean}
    E_{GM}(x_1, ..., x_n) = \left(\frac{1}{n}\sum_{i=1}^n x_i^p\right)^{\frac{1}{p}}.
\end{equation}
\end{deff}
$p=1$ corresponds to the arithmetic mean and $p=2$ to the geometric mean. 
In contrast to the Generalized Mean Error, its derivative $\frac{1}{n}\left(\frac{1}{n} \sum_{j=1}^n x_j^p \right)^{\frac{1}{p}-1}x_i^{p-1}$ has greater values for smaller inputs when $p<1$, and lower values when $p > 1$. 
Since we want to ensure the derivative is high only for inputs with high truth values to reinforce that those are likely the inputs that confirm the formula, we will want to use $p>1$. 
Note that the arithmetic mean $E_{GM}$ has the same derivative as the mean absolute error $A_{MAE}$, meaning that with $p=1$ universal and existential quantification cannot be distinguished. 
Furthermore, increasing all inputs equally is not a great idea for the existential quantifier, as there are likely only a few inputs for which the formula holds. 
Also note that, unlike existential aggregators directly formed from t-conorms, the only maximum of this aggregator is $x_1, ..., x_n=1$. 
However, it will take long until one can reach this optimum for low inputs.

\subsection{Product Aggregator and Probabilistic Sum}
The product aggregator is given as $A_{T_P}(x_1, ..., x_n) = \prod_{i=1}^n x_i$. This is also the probability of the intersection of $n$ independent events. 
It has the following partial derivatives:
\begin{equation}
\label{eq:derivprod}
    \frac{\partial A_{T_P}(x_1, ..., x_n) }{\partial x_i} = \prod_{j=1, i\neq j}^n x_j.
\end{equation}
This derivative vanishes if there are at least two $i$ so that $x_i=0$. Furthermore, the derivative for $x_i$ will be decreased if some other input $x_j$ is low. Finally, we cannot compute this aggregator in practice due to numerical underflow when multiplying many small numbers. Noting that $\argmax\ f(x) = \argmax\ \log(f(x))$, we observe that the \textit{log-product aggregator}
\begin{equation}
    \logprod(x_1, ..., x_n)=(\log \circ A_{T_P})(x_1, ..., x_n) = \sum_{i=1}^n\log(x_i)
\end{equation} 
can be used for formulas in prenex normal form, as then the truth value of the universal quantifiers is not used for another connective. Unlike the other aggregators, its codomain is the non-positive numbers instead of $[0, 1]$. Furthermore, the log-product aggregator can be seen as the log-likelihood function where we take the correct label to be 1, and thus this is similar to cross-entropy minimization. The partial derivatives are
\begin{equation}
    \frac{\partial \logprod(x_1, ..., x_n) }{\partial x_i} = \frac{1}{x_i}.
\end{equation}
In contrast to Equation \ref{eq:derivprod}, the values of the other inputs are irrelevant, and derivatives with respect to lower-valued inputs will be far greater as there is a singularity at $x=0$ (i.e. the value becomes infinite). We can conclude therefore that the product aggregator is particularly promising as it is nonvanishing and can handle outliers.  
The log-product aggregator also combines well with the generalized mean aggregator (Equation \ref{deff:p-mean}) for formulas of the form $\forall x \exists y$. The logarithm reduces the outer exponentiation, resulting in the derivative $\frac{x_i^{p-1}}{\sum_i x_i^p}$. 

The probabilistic sum aggregator $E_{S_P}(x_1, ..., x_n)=1-\prod_{i=1}^n (1-x_i)$ is trickier. Its derivatives are
\begin{equation}
    \frac{\partial E_{S_P}}{\partial x_i} = \prod_{j=1, j\neq i}^n (1-x_j).
\end{equation}
 This derivative is quite intuitive: It increases $x_i$ if the other inputs $x_j$ are low. 
However, this does not take into account the value of $x_i$ itself. If all inputs are low, this will increase all inputs equally.
Since the logarithm does not distribute over addition, we cannot use the same trick here as for the product aggregator, 
so care has to be taken when computing $\log \circ E_{S_P}$ to prevent numerical underflow errors.
\subsection{Nilpotent Aggregators}
The Nilpotent t-norm is given by $T_{nM}(a, b) = \begin{cases}
        \min(a, b), & \text{if } a + b > 1\\
        0, & \text{otherwise.}
    \end{cases}$
In \Appendix \ref{appendix:nilpotent} we show that the Nilpotent aggregator $A_{T_{nM}}$ is equal to
\begin{equation}
\label{eq:aggrnilpmin}
    A_{T_{nM}}(x_1, ..., x_n) = \begin{cases}
        \min(x_1, ..., x_n), &\text{if } x_i + x_j > 1;\ x_i \text{ and } x_j \text{ are the two lowest values in } x_1, ..., x_n   \\
        0, &\text{otherwise}
    \end{cases}.
\end{equation}

The derivative is found as follows: 
\begin{equation}
\label{eq:derivnilp}
    \frac{\partial A_{T_{nM}}(x_1, ..., x_n) }{\partial x_i} = \begin{cases}
        1, &\text{if } i=\argmin_j x_j  \text{ and } x_i + x_j > 1 \text{ where } x_j \text{ is the second lowest value in } x_1, ..., x_n    \\
        0, &\text{otherwise}
    \end{cases}.
\end{equation}
Like the minimum aggregator it is single-passing, and like the \luk\ aggregator it has a derivative that vanishes for the majority of the input space, namely  when the sum of the two smallest values is lower than 1. 
\begin{prop}
The fraction of inputs for which the derivative of $A_{T_{nM}}$ is nonvanishing is $\frac{1}{2^{n-1}}$.
\end{prop}
For the proof, see \ref{appendix:nilpfrac}. The fraction of inputs for which there is a nonvanishing derivative is plotted in Figure \ref{fig:ratio_2_yager}. Again, this means that for larger numbers of inputs $n$, this aggregator will vanish on almost every input and is not a useful construction in a learning context.

Similarly, the Nilpotent existential aggregator is given as 
\begin{equation}
    E_{S_{nM}}(x_1, ..., x_n) = \begin{cases}
        \max(x_1, ..., x_n), &\text{if } x_i + x_j < 1;\ x_i \text{ and } x_j \text{ are the two largest values in } x_1, ..., x_n   \\
        1, &\text{otherwise.}
    \end{cases}
\end{equation}
It has a similar derivative that increases the largest input if the two largest inputs together are lower than 1. This is somewhat similar to how the maximum aggregator behaves, but with an additional condition that will stop increasing the largest input if there is another that is also quite high.  

\subsection{Summary}
The minimum aggregator is computationally inefficient and cannot handle exceptions well.  
Universal aggregation operators that vanish when receiving a large amount of inputs will not scale well, and these include operators based on the Yager family of t-norms and the nilpotent aggregator. 
Removing the bounds from the Yager aggregators introduces interesting connections to loss functions from the classical machine learning literature. 
This is also the case for the logarithmic version of the product aggregator, which corresponds to the cross-entropy loss function. 
They have natural means for dealing with outliers, and thus are promising for practical use.  
We have more options for existential quantification, as problems with vanishing gradients are not as important since we only care about ensuring the formula is true for at least one input, instead of all of them. 


\section{Conjunction and Disjunction}
\label{sec:conj_disj}



Next, we analyze the partial derivatives of t-norms and t-conorms, which are used as conjunction and disjunction in Fuzzy Logics. In t-norm Fuzzy Logics, the weak disjunction $\max(a, b)$, or the Gödel t-conorm is used instead of the dual t-conorm.

Suppose that we have a t-norm $T$ and a t-conorm $S$. We define the following two quantities, where the choice of taking the partial derivative to $a$ is without loss of generality, since $T$ and $S$ are commutative by definition:
\begin{align}
    \deri_T(a, b) = \frac{\partial T(a, b)}{\partial a},\quad
    \deri_S(a, b) = \frac{\partial S(a, b)}{\partial a} 
\end{align}

It should be noted that by Definition \ref{deff:tnorm}, $\frac{\partial T(a, 1)}{\partial a} = 1$ as $T(a, 1) = a$ for any t-norm $T$, and by Definition \ref{deff:snorm}, $\frac{\partial S(a, 0)}{\partial a} = 1$ as $S(a, 0) = a$ for any t-conorm $S$. Furthermore, we note that if $S$ is a t-conorm and the $N_C$-dual of the t-norm $T$, then $\frac{\partial 1-T(1-a, 1-b)}{\partial a} =-\frac{\partial 1-T(1-a, 1-b)}{\partial 1- a} = \frac{\partial T(1-a, 1-b)}{\partial 1- a} $. 

The main difference in analyzing t-norms and t-conorms is that the maximimum of $T(a, b)$ (namely 1) is when both arguments $a$ and $b$ are 1. In contrast, in t-conorms, an infinite number of maxima exist. Some of these maxima might be more desirable than others. Referring back to the formula in Example  \ref{exmp:diffreason}, we showed that it is preferable to increase the truth value of $\pred{cushion}(y)$ and not of $\pred{armRest}(y)$. Similarly, when a conjunct is negated, or when it appears in the antecedent of an implication (like in the aforementioned formula) we have to choose which of the two conjuncts to decrease. By noting that $\frac{\partial T(a, b)}{\partial a} =\frac{\partial S(1-a, 1-b)}{\partial 1-a}$, we find that the t-norm \say{chooses} in the same way its dual t-conorm would \say{choose}. Similarly, if a disjunction is negated, it will minimize both its arguments in the way that its dual t-norm would maximize its arguments.

\begin{exmp}
We introduce a running example to analyze the behavior of different t-norms. We will optimize $(a\otimes b)\oplus (c\otimes \neg a)$ using gradient descent. The truth value of this expression is computed using $f(a, b, c) = S(T(a, b), T(1-a, c))$. Using the boundary conditions from Definition \ref{deff:tnorm} and \ref{deff:snorm}, we find the global optima $a=1.0$, $b=1.0$ and $a=0.0$, $c=1.0$. The derivative to this function is, using the chain rule,
\begin{align}
    \frac{\partial f(a, b, c)}{\partial a} &= \frac{\partial S(T(a, b), T(1-a, c))}{\partial T(a, b)} \cdot \frac{\partial T(a, b)}{\partial a} + \frac{\partial S(T(a, b), T(1-a, c))}{\partial T(1-a, c)}\cdot \frac{\partial T(1-a, c)}{\partial a}.
    \label{eq:exmp_tsnorm}
\end{align}
\end{exmp}

\subsection{Gödel T-Norm}
The Gödel t-norm is $T_G(a, b) = \min(a, b)$ and the Gödel t-conorm is $S_G(a, b) = \max(a, b)$. We find
\begin{align}
    \frac{\partial T_G(a, b)}{\partial a}=  \begin{cases}
        1, & \text{if } a < b \\
        0, & \text{if } a > b
      \end{cases}, \quad
      \frac{\partial S_G(a, b)}{\partial a}=  \begin{cases}
        1, & \text{if } a > b \\
        0, & \text{if } a < b
      \end{cases}.
\end{align}

Both $T_G$ and $S_G$ are single-passing, but their derivatives are not defined $a = b$. 
A benefit of the magnitude of the derivative nearly always being 1 is that there will not be any exploding or vanishing gradients caused by multiple repeated applications of the chain rule.

\begin{exmp}
Filling in Equation \ref{eq:exmp_tsnorm} representing $(a\otimes b) \oplus (c\otimes \neg a)$ with $T_G$ and $S_G$ gives 
\begin{align*}
     \frac{\partial f(a, b, c)}{\partial a} &= \indic{\min(a, b)> \min(1-a, c)}\cdot \indic{a< b} - \indic{\min(1-a, c)> \min(a, b)} \cdot \indic{1-a < c}\\
     &= \indic{a > \min(1-a, c)\wedge a < b} - \indic{1-a > \min(a, b) \wedge 1-a < c}
\end{align*}
where $\indic{c}$ is the indicator function. This corresponds to the decision tree in Figure \ref{fig:decision_tree}. 
The value of $a$ can be modified to increase the truth of either one of the conjunctions. In order to choose which of the two should be true, it compares $a$ with $1-a$. If $1-a < a$, it will increase the first conjunct by increasing $a$. Gradient ascent always finds a global optimum for this formula.

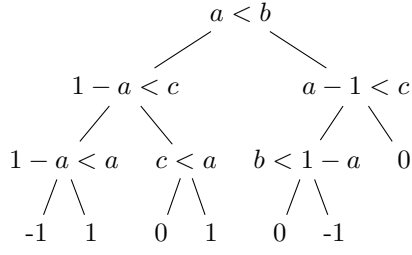
\begin{figure}
    \centering
    \begin{forest} 
    [$a < b$, 
        [$1-a < c$,
            [$1-a < a$,
                [-1]
                [1]
            ] 
            [$c < a$,
                [0]
                [1]
            ] 
        ]   
        [$a-1 < c$,
            [$b < 1-a$,
                [0] 
                [-1] 
            ]   
            [0] 
        ]   
    ] 
    \end{forest}
    \caption{Decision tree for the derivative of $S_G(T_G(a, b), T_G(1-a, c))$ with respect to $a$.}
    \label{fig:decision_tree}
\end{figure}

A small perturbation in the truth values of the inputs can flip the derivative around. For instance, if $a < b$ and $1-a < c$, then it will increase $a$ if its value is $0.501$ and decrease it if it is $0.499$. Furthermore, it can cause gradient ascent to get stuck in local optima. For instance, if $\varphi=(a\oplus b)\otimes(\neg a \oplus c)$ and $a=0.4, b=0.2$ and $c=0.1$, gradient ascent increases $a$ until $a > 0.5$, at which point the gradient flips and it decreases $a$ until $a < 0.5$. Experiments with optimizing this through gradient descent show that we can only find a global optimum in 88.8\% of random initializations of $a$, $b$ and $c$.
\end{exmp}


\subsection{\luk\ T-Norm}
The \luk \ t-norm is $T_{LK}(a, b) = \max(a + b - 1, 0)$ and the \luk \ t-conorm is $S_{LK}(a, b) = \min(a + b, 1)$. The partial derivatives are:
\begin{equation}
        \frac{\partial T_{LK}(a, b)}{\partial a} =  \begin{cases}
        1, & \text{if } a + b > 1 \\
        0, & \text{if } a + b < 1
      \end{cases}, \quad
      \frac{\partial S_{LK}(a, b)}{\partial a} =  \begin{cases}
        1, & \text{if } a + b < 1 \\
        0, & \text{if } a + b > 1
      \end{cases}
\end{equation}
These derivatives vanish on as much as half of their domain (Proposition \ref{prop:luk}). 
However, like the Gödel t-norm, when there is a gradient, it is large and will not cause vanishing or exploding gradients.

\begin{exmp}
Using the \luk\ t-norm and t-conorm in Equation \ref{eq:exmp_tsnorm} gives rise to the following computation
\begin{align*}
     \frac{\partial f(a, b, c)}{\partial a} &= \indic{\max(a + b - 1, 0) + \max(c - a, 0) < 1}\cdot\left(\indic{a+b > 1} - \indic{c-a > 0}\right).
\end{align*}
Choosing random values to initialize $a, b$ and $c$, gradient descent is able to find a (global) optimum in about 83.5\% of the initializations. 
\end{exmp}

\subsection{Yager T-Norm}
\label{sec:norm_yager}

\dualfigure{\imgT T_yager2.pdf}{\imgT S_yager2.pdf}{Left: The Yager t-norm. Right: The Yager t-conorm. For both, $p=2$.}{fig:yager2_norm}

The family of Yager t-norms \pcite{Yager1980} is $T_Y(a, b)=\max(1 - \left((1-a)^p+ (1 - b)^p\right)^{\frac{1}{p}}, 0)$ and the family of Yager t-conorms is $S_Y(a, b) = \min( \left(a^p + b^p\right)^{\frac{1}{p}}, 1)$ for $p\geq 0$. We plot these for $p=2$ in Figure \ref{fig:yager2_norm}. The derivatives are given by
\begin{align}
        \frac{\partial T_{Y}(a, b)}{\partial a} &=  \begin{cases}
        \left((1 - a)^p+ (1 - b)^p\right)^{\frac{1}{p}-1}\cdot(1 - a)^{p-1} & \text{if } (1-a)^p + (1-b)^p < 1, \\
        0 & \text{if }(1-a)^p + (1-b)^p > 1,
      \end{cases} \\
      \frac{\partial S_{Y}(a, b)}{\partial a}&=  \begin{cases}
        \left(a^p + b^p\right)^{\frac{1}{p}-1}\cdot a^{p-1} \quad \quad \quad \quad \quad \quad \quad \ \ & \text{if } a^p + b^p < 1, \\
        0 &\text{if } a^p + b^p > 1
      \end{cases}
      \label{eq:deriv_s_yager}
\end{align}

\dualfigure{\imgT dT_yager2.pdf}{\imgT dS_yager2.pdf}{Left: The derivative of the Yager t-norm. Right: The derivative of the Yager s-norm. For both, $p=2$.}{fig:yager2_deriv}

We plot these derivatives in Figure \ref{fig:yager2_deriv}, showing for each a vanishing derivative on a non-negligible section of the domain. Using the method described in footnote \ref{footnote:yager_aggregator} (Section \ref{sec:aggr_yager}), Mathematica finds a closed form expression for the fraction of inputs for which the Yager t-norm is nonvanishing as $\frac{\sqrt{\pi } 4^{-1/p} \Gamma \left(\frac{1}{p}\right)}{p \Gamma
   \left(\frac{1}{2}+\frac{1}{p}\right)}$. 
Observe that when $p\neq 1$, the derivative of $T_Y$ is undefined at $a=b=1$ and the derivative of $S_Y$ is undefined at $a=b=0$. 
This requires care in the implementation to prevent numerical issues. 
For $p>1$, the lower of the two truth values has a higher derivative for the t-norm, while for the t-conorm, the higher of the two truth values has a higher derivative. As $p$ increases, $T_Y$ and $S_Y$ will behave more like $T_G$ and $S_G$. Note that when $p< 1$, the t-norm will have higher derivatives for higher inputs as the derivative has a singularity at $\lim_{a\rightarrow 1}=\infty$ ($b<1$).

\dualfigure{\imgT T_product.pdf}{\imgT S_product.pdf}{Left: The product t-norm. Right: The product t-conorm.}{fig:prod_norm}

\subsection{Product T-Norm}
\label{sec:product_real_logic}
The product t-norm and t-conorm, visualized in Figure \ref{fig:prod_norm}, are $T_P(a, b) = a\cdot b$ and $S_P(a, b) = a + b - a\cdot b$. Their derivatives are
\begin{equation}
    \frac{\partial T_P(a, b)}{\partial a} = b, \quad \frac{\partial S_P(a, b)}{\partial a} = 1-b.
\end{equation}
The derivative of the t-norm is 0 only when $a=b=0$, and similarly when $a=b=1$ for the t-conorm. The derivative of the t-norm can be interpreted as follows: `If we wish to increase $a\otimes b$, $a$ should be increased in proportion to $b$.' This is not a sensible learning strategy: If both $a$ and $b$ are small, in which case the conjunction is most certainly not satisfied, the derivative will be low instead of high. The derivative of the t-conorm is more intuitive, as it says `If we wish to increase $a\oplus b$, $a$ should be increased in proportion to $1-b$'. If $b$ is not yet true, we definitely want at least $a$ to be true.

\begin{exmp}
By using the product t-norm and t-conorm in Equation \ref{eq:exmp_tsnorm}, we get
\begin{equation*}
    \frac{\partial f(a, b, c)}{\partial a} = (1-(1-a)\cdot c) \cdot b - (1 - a\cdot b)\cdot c 
\end{equation*}
As explained, increase $a$ in proportion to $b$ if it is not true that $c$ and $\neg a$ are true, and decrease $a$ in proportion to $c$ if it is not true that $a$ and $b$ are true. 
\end{exmp}

\subsection{Summary}
The Gödel t-norm and t-conorm are simple and effective, having strong derivatives almost everywhere. However, they can be quite brittle by making very binary choices. The \luk t-norm and t-conorm also have strong derivatives, but vanish on half of the domain. The Yager family of t-norms and t-conorms also vanish on a significant part of its domain. The derivative of the t-norm is larger for lower values, which is a sensible learning strategy. This is not the case for the product t-norm, where the derivative is dependent on the other input value. However, the product t-conorm is intuitive, and corresponds to the intuition that if one input is not true, the other one should be. 



\section{Implication}


Finally, we consider what functions are suitable for modelling the implication. We will start by discussing the particular challenges associated with the implication operator.

\subsection{Challenges of the Material Implication}
\label{sec:implication_challenges}
A significant proportion of background knowledge is written as universally quantified implications. Examples of such statements are `all humans are mortal', `laptops consist of a screen, a processor and a keyboard' and `only humans wear clothes'. These formulas are of the form  $\forall x\ \phi(x) \rightarrow \psi(x)$, where we call $\phi(x)$ the antecedent and $\psi(x)$ the consequent. 

The implication is used in two well known rules of inference from classical logic. 
\emph{Modus ponens} inference says that if $\forall x\ \phi(x) \rightarrow \psi(x)$ and we know that $\phi(x)$ is true, then $\psi(x)$ should also be true. 
\emph{Modus tollens} inference, or \emph{contraposition}, says that if $\forall x\ \phi(x) \rightarrow \psi(x)$ and we know that $\psi(x)$ is false, then $\phi(x)$ should also be false, as otherwise $\psi(x)$ should also have been.

Unlike sequences of conjunctions where each of the formulas should be true, when the agent predicts a scene in which an implication is false, the supervisor has multiple choices.
Consider the implication `all ravens are black'. 
There are 4 categories for this formula: \textit{black ravens} (BR), \textit{non-black non-ravens} (NBNR), \textit{black non-ravens} (BNR) and \textit{non-black ravens} (NBR). 
Assume our agent observes an NBR, then there are four options to consider.
\begin{enumerate}
    \item \textit{Modus Ponens} (MP): The antecedent is true, so by modus ponens, the consequent is also true. We trust the agent's observation of a raven and believe it was a black raven (BR).
    \item \textit{Modus Tollens} (MT): The consequent is false, so by modus tollens, the antecedent is also false. We trust the agent's observation of a non-black object and believe that it was not a raven (NBNR).
    \item \textit{Distrust}: We think the agent is wrong in both observations, and conclude it was a black non-raven (BNR).
    \item \textit{Exception}: We trust the agent in observing a non-black raven (NBR) and ignore the fact that its observation goes against the background knowledge that ravens are black.\footnote{This option is not completely ludicrous as white ravens do in fact exist. However, they are rare.} 
\end{enumerate}
The distrust option seems somewhat useless. The exception option can be correct, but we cannot know when there is an exception from the agent's observations alone. 

We can assume there are far more non-black objects which are not ravens, than there are ravens. 
Thus, from a statistical perspective, it is most likely that the agent observed an NBNR. 
This shows the imbalance associated with the implication, which was first noted in \tcite{vankrieken2019ravens} for the Reichenbach implication. 
It is quite similar to the \emph{class imbalance problem} in Machine Learning \pcite{japkowicz2002class} in the sense that one has far more \textit{contrapositive} examples than positive examples of the background knowledge.

This problem is closely related to the Raven paradox \pcite{Hempel1945,Vranas2004,vankrieken2019ravens} from the field of confirmation theory which ponders what evidence can confirm a statement like `ravens are black'. 
It is usually stated as follows:
\begin{enumerate}
    \item Premise 1: Observing examples of a statement contributes positive evidence towards that statement.
    \item Premise 2: Evidence for some statement is also evidence for all logically equivalent statements.
    \item Conclusion: Observing examples of non-black non-ravens is evidence for `all ravens are black'.
\end{enumerate}
The conclusion follows since `non-black objects are non-ravens' is logically equivalent to `ravens are black'. 
For \dfl a similar thing happens. 
When we correct the observation of an NBR to a BR, the difference in truth value is equal to when we correct it to NBNR. 
More precisely, representing `ravens are black' as $I(a, b)$, where, for example, $I(1, 1)$ corresponds to BR:
$$A(x_1, ..., I(1, 0), ..., x_n) - A(x_1, ..., I(1, 1), ..., x_n) = A(x_1, ..., I(1, 0), ..., x_n) - A(x_1, ..., I(0, 0), ..., x_n)$$ 
as $I(0, 0) = I(1, 1) = 1$.  
When one agent observes a thousand BR's and a single NBR, and another agent observes a thousand NBNR's and a single NBR, their truth value for `ravens are black' is equal. 
The first agent has seen many ravens of which only a single was not black. 
The second only observed non ravens, and a single raven that was not black. 
Intuitively, the first agent's beliefs seem more in line with the background knowledge. 
We will now proceed to analyze a number of implication operators in light of this discussion.



\subsection{Analyzing the Implication Operators}
In this section, we choose to take the negation of the derivative with respect to the antecedent as it makes it easier to compare them: all fuzzy implications are monotonically decreasing with respect to the antecedent. 

\begin{deff}
A fuzzy implication $I$ is 
\textit{contrapositive differentiable symmetric} if $\dmp{a}{c} = \dmtp{1-c}{1-a}$ for all $a, c \in [0, 1]$.

\end{deff}
A consequence of contrapositive differentiable symmetry is that if $c=1-a$, then derivatives with respect to the antecedent and consequent are each other's negation since $\dmp{a}{c}=\dmtp{1 - c}{1-a} = \dmtp{1 -(1 - a)}{c} = \dmt{a}{c}$. This could be seen as the `distrust' option which increases the consequent and negated antecedent equally.

\begin{prop}
If a fuzzy implication $I$ is $N$-contrapositive symmetric (that is, for all $a, c\in [0,1]$, $I(a, I(b, c))=I(b, I(a, c))$), where $N$ is the classical negation, it is also contrapositive differentiable symmetric. 
\end{prop}
By this proposition all S-implications are contrapositive differentiable symmetric. This property says there is no difference in how the implication handles the derivatives with respect to the consequent and antecedent.
\begin{prop}
\label{prop:diff_left_neutral}
If an implication $I$ is left-neutral (that is, for all $c\in [0, 1]$, $I(1, c)=c$), then $\dmp{1}{c} = 1$. If, in addition, $I$ is contrapositive differentiable symmetric, then $\dmt{a}{0} = 1$.
\end{prop}

The proofs of these two propositions are in \ref{appendix:implication-proofs}. 
All S-implications and R-implications are left-neutral, but only S-implications are all also contrapositive differentiable symmetric.
The derivatives of R-implications vanish when $a\leq c$, that is, on no less than half of the domain. 
This is not necessarily a bad property, although this depends highly on the sort of application we use DFL for. 
If, for example, $a=0.499$ and $c=0.5$ for the implication `ravens are black', this expresses a state of uncertainty, and there is probably more that can be learned! However, this will not be possible since the implication vanishes.


\subsection{Gödel-based Implications}
\label{sec:godel_implications}
\dualfigure{\imgT I_kleene.pdf}{\imgT I_godel.pdf}{Left: The Kleene Dienes implication. Right: The Gödel implication. Plots in this section are rotated so that the smallest value is in the front to help understand the shape of the functions. In particular, plots of the derivatives of the implications are rotated 180 degrees compared to the implications themselves.}{fig:kd_godel}

Implications based on the Gödel t-norm make strong discrete choices and are single-passing. As $I_{KD}(a, c) = \max(1-a, c)$, the derivatives are
\begin{equation}
    \dmpa{I_{KD}}{a}{c} =  \begin{cases}
        1, & \text{if } 1 - a < c \\
        0, & \text{if } 1 - a > c
      \end{cases}, \quad
      \dmta{I_{KD}}{a}{b}=  \begin{cases}
        1, & \text{if } 1-a > c \\
        0, & \text{if } 1-a < c
      \end{cases}   .
\end{equation}
Or, simply put, if we are more confident in the truth of the consequent than in the truth of the negated antecedent, increase the truth of the consequent. Otherwise, decrease the truth of the antecedent. This decision can be somewhat arbitrary and does not take into account the imbalance of modus ponens and modus tollens.

The Gödel implication is a simple R-implication: $I_G(a, c) =\begin{cases}
        1, & \text{if } a \leq c \\
        c, & \text{otherwise}
      \end{cases}$. 
Its derivatives are:
\begin{equation}
\label{eq:god_imp_deriv}
    \dmpa{I_G}{a}{c} =  \begin{cases}
        1, & \text{if } a > c \\
        0, & \text{otherwise}
      \end{cases}, \quad
      \dmta{I_G}{a}{b}= 0.
\end{equation}
These two implications are shown in Figure \ref{fig:kd_godel}. The Gödel implication increases the consequent whenever $a > c$, and the antecedent is never changed. This makes it a poorly performing implication in practice. For example, consider $a=0.1$ and $c=0$. Then the Gödel implication increases the consequent, even if the agent is fairly certain that neither is true. Furthermore, as the derivative with respect to the negated antecedent is always 0, it can never choose the modus tollens correction, which, as we argued, is actually often the best choice. 



\subsection{\luk\ and Yager-based Implications}
The \luk\ implication is both an S- and an R-implication. It is given by $I_{LK}(a, c) = \min(1-a + c, 1)$ and has the simple derivatives 
\begin{equation}
\label{eq:impl_deriv_luk}
    \dmpa{I_{LK}}{a}{c} = \dmta{I_{LK}}{a}{c} =  \begin{cases}
        1, & \text{if } a > c \\
        0, & \text{otherwise.}
      \end{cases}
\end{equation}
Whenever the implication is not satisfied because the antecedent is higher than the consequent, simply increase the negated antecedent and the consequent until it is lower. This could be seen as the `distrust' choice as both observations of the agent are equally corrected, and so does not take into account the imbalance between modus ponens and modus tollens cases. The derivatives of the Gödel implication $I_G$ are equal to those of $I_{LK}$ except that $I_G$ always has a zero derivative for the negated antecedent.

\dualfigure{\imgT I_yager_s_2.pdf}{\imgT I_yager_r_2.pdf}{Left: The Yager S-implication. Right: The Yager R-implication. For both, $p=2$.}{fig:yager-r-s-impl} 

The Yager S-implication is given as $I_Y(a, c) = \min\left(\left((1 - a)^p + c^p\right)^{\frac{1}{p}}, 1\right),\ p> 0$.
We plot $I_Y$ for $p=2$ in Figure \ref{fig:yager-r-s-impl}. $p=1$ is $I_{LK}$,  $p=0$ is $I_{DP}$, and $p=\infty$ is $I_{KD}$. 
The derivatives are computed as 
\begin{align}
    \dmpa{I_Y}{a}{c} &=  \begin{cases}
        \left((1-a)^p + c^p\right)^{\frac{1}{p}-1}\cdot c^{p-1}, \quad \quad \ \ & \text{if } (1-a)^p + c^p \leq 1, \\
        0, & \text{otherwise}
      \end{cases}\\
    \dmta{I_Y}{a}{c}&=  \begin{cases}
        \left((1-a)^p + c^p\right)^{\frac{1}{p}-1}\cdot (1-a)^{p-1},  & \text{if } (1-a)^p + c^p \leq 1, \\
        0, &\text{otherwise.}
      \end{cases}
\end{align}
\dualfigure{\imgT dMP_yager_s_2.pdf}{\imgT dMT_yager_s_2.pdf}{Plots of the derivatives of the Yager S-implication for $p=2$.}{fig:deriv_yager_sI2}
We plot these derivatives for $p=2$ in Figure \ref{fig:deriv_yager_sI2}. 
For all $p$, $\lim_{c\rightarrow 0}\dmpa{I_Y}{1}{c}=1$. 
Furthermore, for $p>1$, $\lim_{a\rightarrow 1} \left. \dmpa{I_Y}{a}{c}\right\rvert_{c=0}=0$ and for $p < 1$, $\lim_{a\rightarrow 1}\left. \dmpa{I_Y}{a}{0}\right\rvert_{c=0}=\infty$. 
For $p>1$, $I_Y$ can be understood as an increasingly less smooth version of the Kleene-Dienes implication $I_{KD}$. Lastly, this derivative, like those for $T_Y$ and $S_Y$ (Section \ref{sec:norm_yager}), is nonvanishing for a fraction of $\frac{\sqrt{\pi } 4^{-1/p} \Gamma \left(\frac{1}{p}\right)}{p \Gamma
   \left(\frac{1}{2}+\frac{1}{p}\right)}$ of the input space.

The Yager R-implication is found (\ref{appendix:yagerrimpl}) as $
    I_{T_Y}(a, c) = \begin{cases}
        1, & \text{if } a \leq c \\
        1 - \left((1 - c)^p - (1 - a)^p\right)^{\frac{1}{p}}, & \text{otherwise.}
      \end{cases}$
We plot $I_{T_Y}$ for $p=2$ in Figure \ref{fig:yager-r-s-impl}. 
As expected, $p=1$ reduces to $I_{LK}$, $p=0$ reduces to $I_{WB}$ and $p=\infty$ reduces to $I_G$. It is contrapositive symmetric only for $p=1$. The derivatives of this implication are 
\begin{align}
    \dmpa{I_{T_Y}}{a}{c} &= \begin{cases}
    ((1 - c)^p - (1 - a)^p)^{\frac{1}{p}-1}\cdot (1 - c)^{p-1} , & \text{if } a > c \\
    0, &\text{otherwise,}
    \end{cases}\\
    \dmta{I_{T_Y}}{a}{c}&= \begin{cases}
    ((1 - c)^p - (1 - a)^p)^{\frac{1}{p}-1}\cdot (1 - a)^{p-1}, & \text{if } a > c \\
    0, &\text{otherwise.}
    \end{cases}
\end{align}

We plot these in Figure \ref{fig:yager-r-deriv}. Note that if $p> 1$, for all $c<1$ it holds that $\lim_{a\downarrow c}\dmpa{I_{T_Y}}{a}{c} = \lim_{a\downarrow c}\dmta{I_{T_Y}}{a}{c} = \infty$ as when $a$ approaches $c$ from above, $(1 - c)^p - (1 - a)^p$ approaches 0, giving a singularity as $0^{\frac{1}{p} - 1}$ is undefined. This collection of singularities makes the training unstable in practice. 

\dualfigure{\imgT dMP_yager_r_2.pdf}{\imgT dMT_yager_r_2.pdf}{Plots of the derivatives of the Yager R-implication for $p=2$.}{fig:yager-r-deriv}

\subsection{Product-based Implications}
\label{sec:prod-implication}
\dualfigure{\imgT I_reichenbach.pdf}{\imgT I_goguen.pdf}{Left: The Reichenbach implication. Right: The Goguen implication.}{fig:goguen-i}

\dualfigure{\imgT dMT_logged_reichenbach.pdf}{\imgT dMT_RMSE1_reichenbach.pdf}{Left: The antecedent derivative of the Reichenbach implication with the log-product aggregator. Right: The antecedent derivative of the Reichenbach implication with the RMSE aggregator.}{fig:reichenbach-aggregators}

The product S-implication, also known as the Reichenbach implication, is given by $I_{RC}(a, c) = 1 - a + a \cdot c$. We plot it in Figure \ref{fig:goguen-i}. Its derivatives are given by:
\begin{equation}
    \dmpa{I_{RC}}{a}{c} = a, \quad \dmta{I_{RC}}{a}{c} = 1-c.
\end{equation}
These derivatives closely follow the modus ponens and modus tollens rules. When the antecedent is high, increase the consequent, and when the consequent is low, decrease the antecedent. However, around $(1-a)=c$, the derivative is equal and the `distrust' option is chosen. This can result in counter-intuitive behaviour. For example, if the agent predicts 0.6 for $\pred{raven}$ and 0.5 for $\pred{black}$ and we use gradient descent until we find a maximum, we could end up at 0.3 for $\pred{raven}$ and 1 for $\pred{black}$. We would end up increasing our confidence in $\pred{black}$ as $\pred{raven}$ was high. However, because of additional modus tollens reasoning, $\pred{raven}$ is barely true. 

Furthermore, if the agent most of the time predicts values around $a=0,\ c=0$ as a result of the modus tollens case being the most common, then a majority of the gradient decreases the antecedent as $\left.\dmta{I_{RC}}{a}{0}\right\rvert_{a=0}= 1$. We identify two methods that counteract this behavior. We introduce the second method in Section \ref{sec:sigm_implication}. 
The first method for counteracting the `corner' behavior notes that different aggregators change how the derivatives of the implications behave when their truth value is high.
For instance, we find that the derivatives with respect to the negated antecedent when using the log-product aggregator and RMSE aggregator are
\begin{alignat}{3}
    \frac{\partial\log \circ A_P(I_{RC}(a_1, c_1), ..., I_{RC}(a_n, c_n))}{\partial 1- a_i} &= \frac{1-c_i}{1 - a_i + a_i \cdot c_i}&&=\frac{\neg c}{a\rightarrow c}, \\
    \frac{\partial A_{RMSE}(I_{RC}(a_1, c_1), ...,I_{RC}(a_n, c_n))}{\partial 1 - a_i} &= \frac{(1 - c_i)(a_i - a_i\cdot c_i)}{\sqrt{n\sum_{j=1}^n (a_j - a_j\cdot c_j)^2}}&&=\frac{\neg c_i(\neg(a_i\rightarrow c_i))}{\sqrt{n
    \sum_{j=1}^n (\neg(a_j\rightarrow c_j))^2}}.
\end{alignat}
We plot these functions in Figure \ref{fig:reichenbach-aggregators}. 
For the RMSE aggregator 
we choose $n=2$ and $a_1,\ c_1$ so that $(a_1 - a_1 \cdot c_1)^2=0.9$. 
Note that the derivative with respect to the negated antecedent using the RMSE aggregator is 0 in $a_i=0$, $c_i=0$ as then $a_i - a_i \cdot c_i = 0$, and using the log-product aggregator, the derivative is 1. By differentiable contrapositive symmetry, the consequent derivative is 0 when using both aggregators. This shows that when using the RMSE aggregator, the derivatives will vanish at the corners $a=0,\ c=0$ and $a=1,\ c=1$, while when using the log-product aggregator, one of $a$ and $c$ will still have a gradient.

\dualfigure{\imgT log_dMP_goguen.pdf}{\imgT log_dMT_goguen.pdf}{The derivatives of the Goguen implication. Note that we plot these in log scale.}{fig:deriv_goguen}

The R-implication of the product t-norm is the Goguen implication and given by $I_{GG}(a, c) = \begin{cases} 1, &\text{if } a\leq c \\ \frac{c}{a}, &\text{otherwise.} \end{cases}$.
We plot this implication in Figure \ref{fig:goguen-i}. 
The derivatives of $I_{GG}$ are
\begin{equation}
    \dmpa{I_{GG}}{a}{c} = \begin{cases}
    0, &\text{if } a\leq c \\
    \frac{1}{a}, &\text{otherwise}
    \end{cases}, \quad \dmta{I_{GG}}{a}{c} = \begin{cases}
    0, &\text{if } a\leq c \\
    \frac{c}{a^2}, &\text{otherwise}
    \end{cases}.
\end{equation}

We plot these in Figure \ref{fig:deriv_goguen}. This derivative is not very useful. First of all, both the modus ponens and modus tollens derivatives increase with $\neg a$. This is opposite of the modus ponens rule as when the antecedent is \textit{low}, it increases the consequent most. For example, if $\pred{raven}$ is 0.1 and $\pred{black}$ is 0, then the derivative with respect to $\pred{black}$ is 10, because of the singularity when $a$ approaches 0.

\subsubsection{Sigmoidal Implications}
\label{sec:sigm_implication}
For the second method for tackling the corner problem, we introduce a new class of fuzzy implications formed by transforming other fuzzy implications using the sigmoid function and translating it so that the boundary conditions still hold. The derivation, along with several proofs of properties, can be found in \ref{appendix:sigm}.

\begin{deff}
\label{deff:sigmoidal}
If $I$ is a fuzzy implication, then the $I$-sigmoidal implication $\sigma_I$ is given for $s>0$ and $b_0\in \mathbb{R}$ as
\begin{equation}
    \sigma_I(a, c) = \frac{1 + e^{-s(1+b_0)}}{e^{-b_0 s}-e^{-s(1+b_0)}}\cdot 
   \left(\left(1 + e^{-b_0s}\right) \cdot \sigma\left(s\cdot\left(I(a, c) + b_0 \right)\right) - 1\right) 
\end{equation}
    where $\sigma(x)=\frac{1}{1+e^x}$ denotes the sigmoid function.
\end{deff}
Here $b_0$ is a parameter that controls the position of the sigmoidal curve and $s$ controls the `spread' of the curve. $\sigma_I$ is the function $ \sigma\left(s\cdot\left(I(a, c) + b_0\right)\right)$ linearly transformed so that its codomain is the closed interval $[0, 1]$. For the common value of $b_0=-\frac{1}{2}$, a simpler form exists: 
\begin{equation}
    \sigma_I(a, c)=\frac{1}{e^{\frac{s}{2}}-1}\cdot 
  \left(\left(1 + e^{\frac{s}{2}}\right) \cdot \sigma\left(s\cdot\left(I(a, c) - \frac{1}{2}\right)\right) - 1\right). 
\end{equation}
Next, we give the derivative of $\sigma_I$. Substituting $d=\frac{1 + e^{-s\cdot(1 + b_0)}}{e^{-s\cdot b_0}-e^{-s\cdot(1 + b_0)}}$ and $h=\left(1 + e^{-s\cdot b_0}\right)$, we find
\begin{equation}
\label{eq:deriv_sigm}
    \frac{\partial \sigma_I(a, c)}{\partial I(a, c)}=d\cdot h\cdot s\cdot  \sigma\left(s\cdot\left(I(a, c) + b_0\right)\right)\cdot(1 -  \sigma\left(s\cdot\left(I(a, c) + b_0\right)\right)).
\end{equation}
This keeps the properties of the original function but smoothens the gradient for higher values of $s$. 
As the derivative of the sigmoid function is positive, this derivative vanishes only when the derivative of $I$ vanishes.

\dualfigure{\imgT logdMP_reichenbach.pdf}{\imgT logdMP_probsum9.pdf}{The consequent derivatives of the log-Reichenbach and log-Reichenbach-sigmoidal (with $s=9$) implications. The figure is plotted in log scale.}{fig:log_dmp}

\dualfigure{\imgT dmp_probsum9.pdf}{\imgT dmt_probsum9.pdf}{The derivatives of the Reichenbach-sigmoidal implication for $s=9$.}{fig:deriv_rcsigm}

\begin{figure}
    \centering
    \begin{subfigure}[b]{\graphwidth}
    \includegraphics[width=\linewidth]{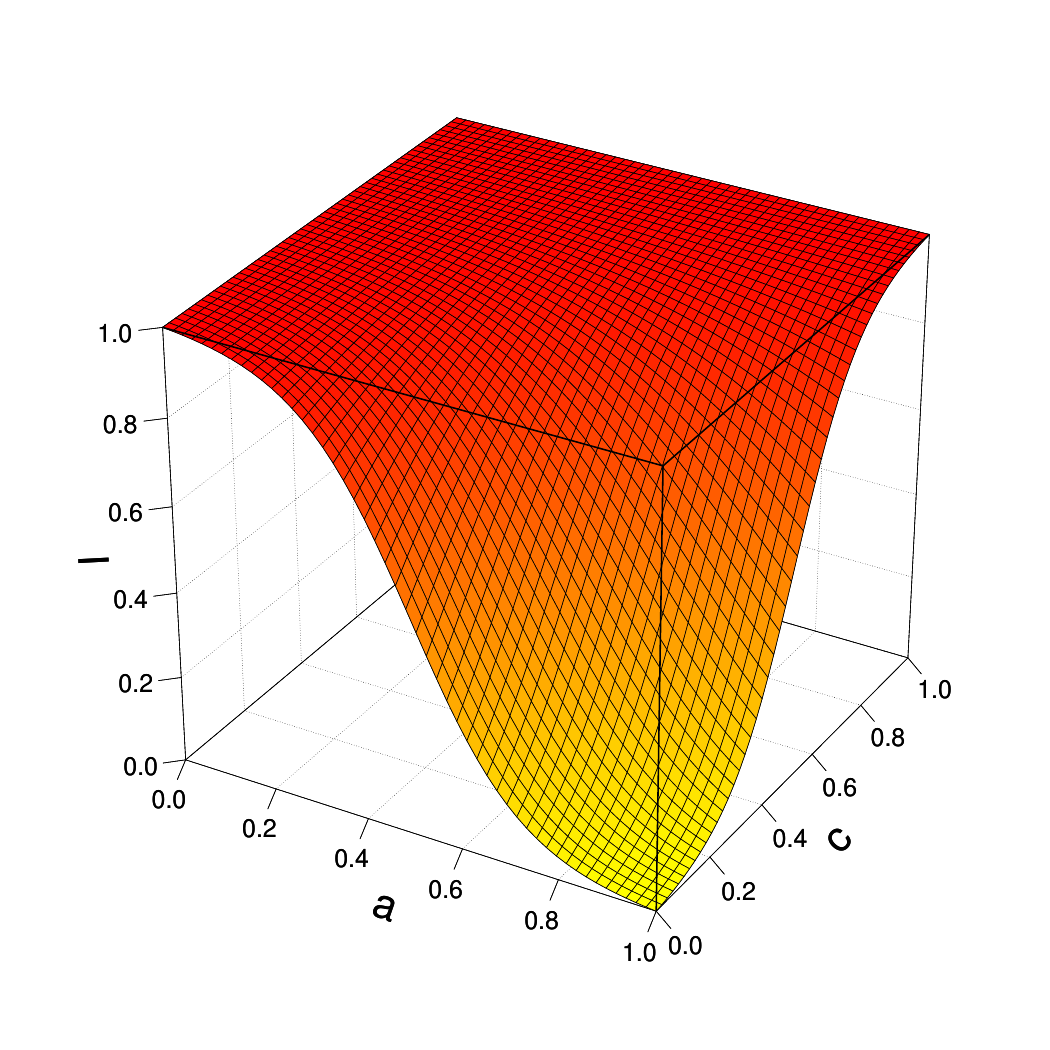}
    \caption{$b_0=-0.5, s=9$}
    \label{fig:probsum_sigm_s_9-0.5}
    \end{subfigure}
    \begin{subfigure}[b]{\graphwidth}
    \includegraphics[width=\linewidth]{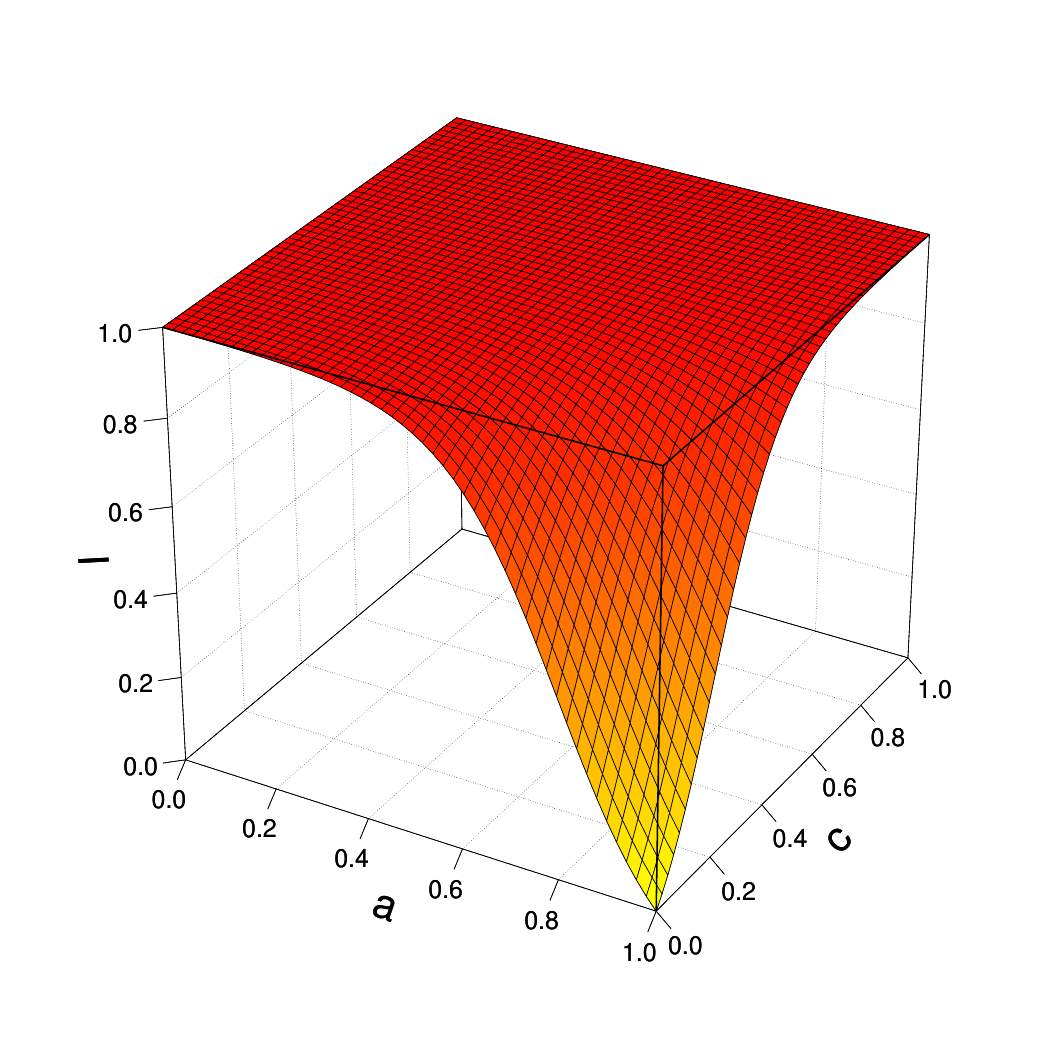}
    \caption{$b_0=-0.2, s=9$}
    \label{fig:probsum_sigm_s_9-0.2}
    \end{subfigure}
    \\
    \begin{subfigure}[b]{\graphwidth}
    \includegraphics[width=\linewidth]{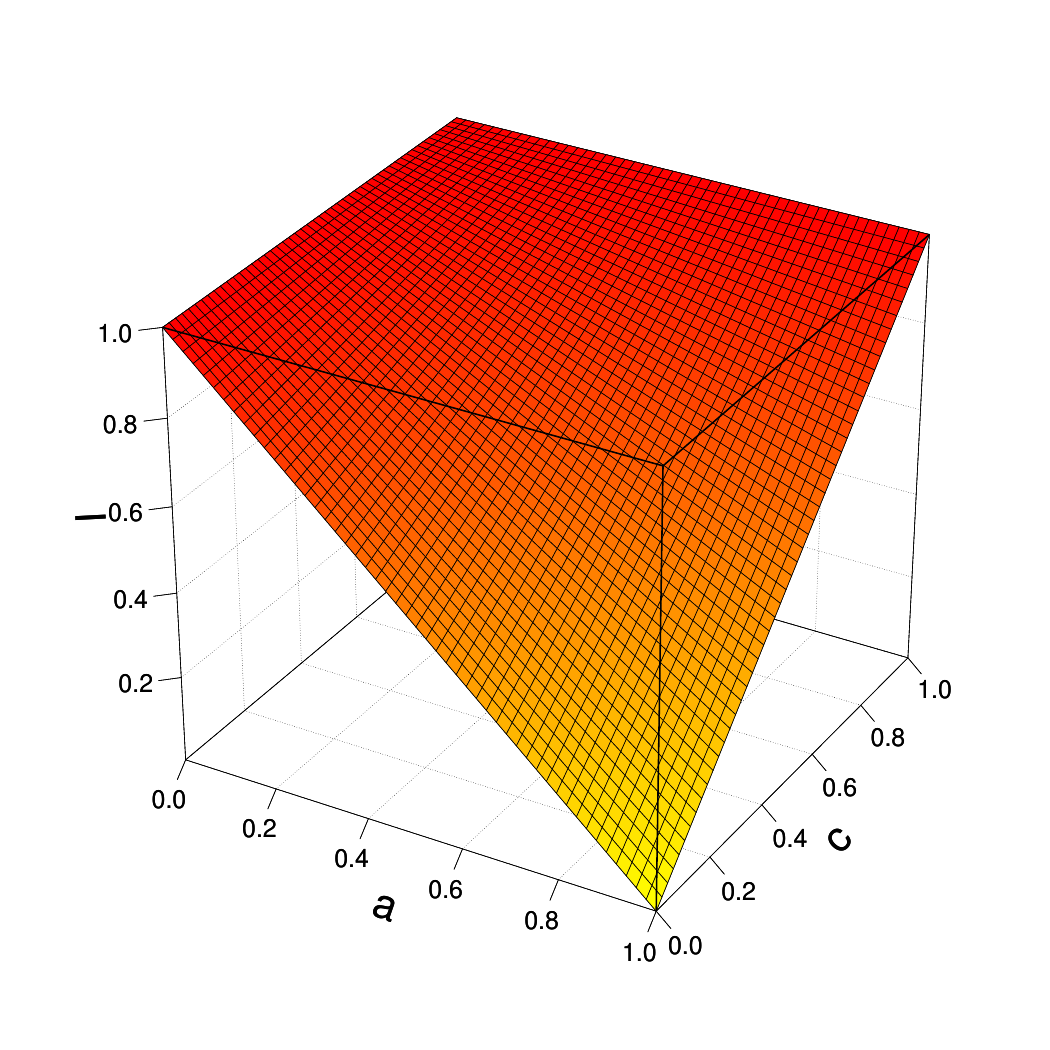}
    \caption{$b_0=-0.5, s=0.01$}
    \label{fig:probsum_sigm_s_0.01}
    \end{subfigure}
    \begin{subfigure}[b]{\graphwidth}
    \includegraphics[width=\linewidth]{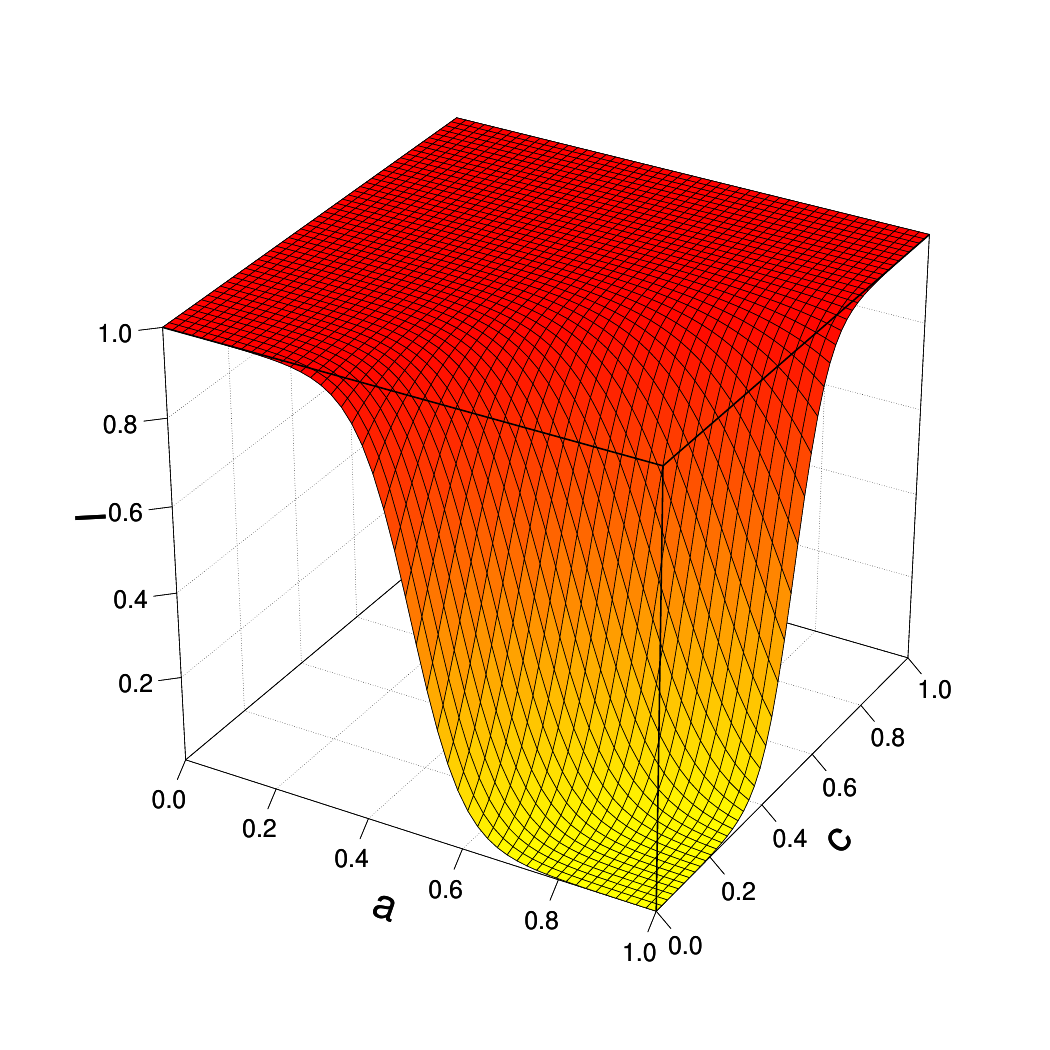}
    \caption{$b_0=-0.5, s=20$}
    \label{fig:probsum_sigm_s_20}
    \end{subfigure}%
    \caption{The Reichenbach-sigmoidal implication for different values of $b_0$ and $s$.}
    \label{fig:probsum_sigm_s}
    \end{figure}

We plot the derivatives for the Reichenbach-sigmoidal implication $\sigma_{I_{RC}}$ in Figure \ref{fig:deriv_rcsigm}. As expected by Proposition \ref{prop:sigm_contrapos}, it is  differentiable contrapositive symmetric. Compared to the derivatives of the Reichenbach implication 
it has a small gradient in all corners. 
When using the log-product aggregator, the derivative of the antecedent with respect to the total valuation is divided by the truth of the implication. In Figure \ref{fig:log_dmp} we compare the consequent derivative of the normal Reichenbach implication with the Reichenbach-sigmoidal implication when using the $\log$ function. Clearly, for both there is a singularity at $a=1,\ c=0$, as then the implication is 0 and so the derivative of the log function becomes infinite. A significant difference is that the sigmoidal variant is less `flat' than the normal Reichenbach implication. This can be useful, as this means there is a larger gradient for values of $c$ that make the implication less true. In particular, the gradient at the modus ponens case ($a=1,\ c=1$) and the modus tollens case ($a=0,\ c=0$) are far smaller, which could help balancing the effective total gradient by solving the `corner' problem of the Reichenbach implication we brought up in Section \ref{sec:prod-implication}. These derivatives are smaller for higher values of $s$. 

In Figure \ref{fig:probsum_sigm_s} we plot the Reichenbach-sigmoidal implication for different values of the hyperparameters $b_0$ and $s$. Comparing \ref{fig:probsum_sigm_s_9-0.5} and \ref{fig:probsum_sigm_s_9-0.2} we see that larger values of $b_0$ move the sigmoidal shape so that its center is at lower input values. Note that for $s=0.01$ in Figure \ref{fig:probsum_sigm_s_0.01}, the plotted function is indiscernible from the plot of the Reichenbach implication in Figure \ref{fig:goguen-i} as the interval on which the sigmoid acts is extremely small and the sigmoidal transformation is almost linear. For very high values of $s$ like in \ref{fig:probsum_sigm_s_20} we see that the `S' shape is much thinner, and a larger part of the domain has a low derivative.

\subsection{Summary}
We analyzed several fuzzy implications from a theoretical perspective, while keeping the challenges caused by the material implication in mind. As a result of this analysis, we find that popular R-implications, in particular the Gödel implication, the Yager R-implication and the Goguen implication, will not work well in a differentiable setting. The other analyzed implications seem to have more intuitive derivatives, but may have other practical issues like non-smoothness.

\section{Experimental setup}
\label{chapter:experiments}


\label{sec:mnist}
To get insights in the behavior of these operators in practice, we next perform a series of simple experiments to analyze them. 
We discuss experiments using the MNIST dataset of handwritten digits \pcite{lecun-mnisthandwrittendigit-2010} to investigate the behavior of different fuzzy operators introduced in this paper. 
The goal of the experiments is not to show that our method is state of the art for the problem of semi-supervised learning on MNIST, but rather to be able to get insights into how fuzzy operators behave in a differentiable setting.\footnote{Code is available at \url{https://github.com/HEmile/differentiable-fuzzy-logics}.}
\subsection{Measures}
\label{sec:mnist_ldr}
To investigate the performance of the different configurations of \dfl, we first introduce several useful metrics. These give us insight into how different operators behave. In this section, we assume we are dealing with formulas of the form $\varphi=\forall x_1, ..., x_m\ \phi\rightarrow \psi$.

\begin{deff}
The \textit{consequent magnitude} $\mpmag$ and the \textit{antecedent magnitude} $\mtmag$ for a knowledge base $\corpus$  is defined as the sum of the partial derivatives of the consequent and antecedent with respect to the \dfl loss:
\begin{equation}
    \mpmag= \sum_{\varphi\in\corpus}\sum_{\instantiation\in\instantiations_\varphi}\frac{\partial \val(\mu, \varphi)}{\partial \val(\mu, \psi)}, \quad \mtmag= \sum_{\varphi\in\corpus}-\sum_{\instantiation\in\instantiations_\varphi}\frac{\partial \val(\mu, \varphi)}{\partial \val(\mu, \phi)},
\end{equation}
where $\instantiations_\varphi$ is the set of instances of the universally quantified formula $\varphi$ and $\psi$ and $\phi$ are evaluated under instantiation $\instantiation$.

The \textit{consequent ratio} $\mpratio$ is the sum of consequent magnitudes divided by the sum of consequent and antecedent magnitudes: $\mpratio = \frac{\mpmag}{\mpmag + \mtmag}$.
\end{deff}


\begin{deff}
Given a \textit{labeling function} $l$ that returns the truth value of a formula according to the data for instance $\instantiation$, the \textit{consequent and antecedent correctly updated magnitudes} are the sum of partial derivatives for which the consequent or the negated antecedent is true:
\begin{equation}
    \mpcorupdate = \sum_{\varphi\in\corpus}\sum_{\instantiation\in\instantiations_\varphi}  l(\psi, \mu) \cdot \frac{\partial \val(\mu, \varphi)}{\partial \val(\mu, \psi)}, \quad \mtcorupdate = \sum_{\varphi\in\corpus}-\sum_{\instantiation\in\instantiations_\varphi}l(\neg \phi, \mu) \cdot \frac{\partial \val(\mu, \varphi)}{\partial \val(\mu, \phi)}.
\end{equation}
\end{deff}
That is, if the consequent is true in the data, we measure the magnitude of the derivative with respect to the consequent. 
To evaluate these quantities, we define ratios similar to a precision metric:
\begin{deff}
The \textit{correctly updated ratio} for consequent and antecedent are defined as
\begin{equation}
    \mpupdateratio = \frac{\sum_{\varphi\in\corpus}\mpcorupdate_\varphi}{\sum_{\varphi\in\corpus}\mpmag_\varphi}, \quad    \mtupdateratio = \frac{\sum_{\varphi\in\corpus}\mtcorupdate_\varphi}{\sum_{\varphi\in\corpus}\mtmag_\varphi}.
\end{equation}
\end{deff}
These quantify what fraction of the updates are going in the right direction. When these ratios approach 1, \dfl will always increase the truth value of the consequent or negated antecedent correctly.\footnote{It can still change the truth value of a ground atom wrongly if $\phi$ or $\psi$ are not atomic formulas.} Otherwise, we are increasing truth values of subformulas that are wrong. 
Ideally, we want these measures to be high.

\subsection{Formulas}
We use a knowledge base $\corpus$ of universally quantified logic formulas. There is a predicate for each digit, that is $\pred{zero},\ \pred{one}, ..., \pred{eight}$ and  $\pred{nine}$. For example, $\pred{zero}(x)$ is true whenever $x$ is a handwritten digit labeled with 0. 
We have two sets of formulas where we learn an additional binary predicate.

\subsubsection{The \pred{same} problem}
The \pred{same} problem is a simple problem to test different operators for implication and universal aggregation.
We use the binary predicate $\pred{same}$ that is true whenever both its arguments are the same digit. We next describe the formulas we use. 
\begin{enumerate}
\item $ \forall x, y\ \pred{zero}(x)\otimes \pred{zero}(y) \rightarrow \pred{same}(x, y), ..., \forall x, y\ \pred{nine}(x)\otimes \pred{nine}(y) \rightarrow \pred{same}(x, y) $. 
If both $x$ and $y$ are handwritten zeros, for example, then they represent the same digit. 
\item $ \forall x, y\ \pred{zero}(x) \otimes \pred{same}(x, y) \rightarrow \pred{zero}(y), ..., \forall x, y\ \pred{nine}(x) \otimes \pred{same}(x, y) \rightarrow \pred{nine}(y) $. 
If $x$ and $y$ represent the same digit and one of them represents zero, then the other one does as well. 
\item $ \forall x, y\ \pred{same}(x, y) \rightarrow \pred{same}(y, x) $. 
This formula encodes the symmetry of the $\pred{same}$ predicate. 
\end{enumerate}
We find in \ref{appendix:formulas} that a set of operators is better than random guessing for the consequent updates if $\mpupdateratio>0.1$,  and that we know with confidence a set of operators to be better than random if $\mtupdateratio > 0.99$.

\subsubsection{The \pred{sum9} problem}
In the second problem we use the binary predicate $\pred{sum9}$ that is true whenever its arguments sum to 9. We use this problem to test existential quantification, conjunction and disjunction. The formulas are
\begin{enumerate}
    \item $\forall x \exists y\  \pred{sum9}(x, y)$. For each digit, there is another such that their sum is 9.\footnote{We sample minibatches of 64 digits, which means there is a negligible probability that there exists a digit in the minibatch for which there is no match (0.0117, to be precise).} 
    \item $\forall x, y\ \pred{sum9}(x, y) \rightarrow (\pred{zero}(x) \otimes \pred{nine}(y)) \oplus (\pred{one}(x) \otimes \pred{eight}(y)) \oplus \dots \oplus (\pred{nine}(x) \otimes \pred{zero}(y))$: This formula defines the $\pred{sum9}$ predicate.
\end{enumerate}
\subsection{Experimental Methodology}

We split the MNIST dataset so that 1\% of it is labeled and 99\% is unlabeled. Given a handwritten digit $\bx$ labeled with digit $y$, $p_\btheta(y|\bx)$ computes the distribution over the 10 possible labels. We use 2 convolutional layers with max pooling, the first with 10 and the second with 20 filters. Then follows two fully connected hidden layers with 320 and 50 nodes and a softmax output layer. The probability that $\pred{same}(\bx_1, \bx_2)$ for two handwritten digits $\bx_1$ and $\bx_2$ holds is modeled by $p_\btheta(\pred{same}|\bx_1, \bx_2)$. This takes the 50-dimensional embeddings of $\bx_1$ and $\bx_2$ of the fully connected hidden layer $e_{\bx_1}$ and $e_{\bx_2}$. These are used in a network architecture called a Neural Tensor Network \pcite{Socher2013}:

\begin{equation}
    p_\btheta(\pred{same}|\bx_1, \bx_2)=\sigma\left(u^\intercal \tanh\left(e_{\bx_1}^\intercal W^{[1:k]}e_{\bx_2} + V \begin{bmatrix} e_{\bx_1} \\ e_{\bx_2} \end{bmatrix} + b\right)\right).
\end{equation}
$W^{[1:k]}\in\mathbb{R}^{d\times d\times k}$ is used for the bilinear tensor product, $V\in \mathbb{R}^{k\times 2d}$ is used for a the concatenated embeddings and $b\in\mathbb{R}^k$ is used as a bias vector. We use $k=50$ for the size of the hidden layer. $u\in\mathbb{R}^{k}$ is used to compute the output logit, which goes through the sigmoid function $\sigma$ to get the confidence value. 

The loss function we use is split up in three parts, the first over the unlabeled dataset $\dataset_u$, and the two others over the labeled dataset $\dataset_l$:
\begin{align}
    \loss(\btheta) = w_{\dfl} \cdot \loss_{\dfl} (\langle \dataset_u, \eta, \btheta \rangle , \corpus)  -\sum_{\bx, y\in \dataset_l} \log p_\btheta(y|\bx) 
    -  \sum_{\substack{\bx_1, y_1, \bx_2, y_2\\\in \dataset_l\times\dataset_l}}
    \log p_\btheta(\pred{same}=\boldsymbol{1}_{y_1=y_2}|\bx_1, \bx_2)
\end{align}
The first term is the \dfl loss which is weighted by the \textit{\dfl weight} $w_{\dfl}$. 
The second is the supervised cross entropy loss with a batch size of 64. 
The third is the supervised binary cross entropy loss used to learn recognize $\pred{same}(x, y)$.\footnote{It is possible to not use this loss term and learn the $\pred{same}$ predicate using just the formulas, although this is more challenging and only works with a good set of fuzzy operators.} 
This loss is $\log p_\btheta(\pred{sum9}=\boldsymbol{1}_{y_1 + y_2=9}|\bx_1, \bx_2)$ for the \pred{sum9} problem. As there are far more negative examples than positive examples, we undersample the negative examples.
Note that the two supervised losses can also be seen as a universal aggregation over logical facts using the log-product aggregator: $-A_{\log T_P}(p_\btheta(y_1|\bx_1), ..., p_\btheta(y_{|\dataset_l|}|\bx_{|\dataset_l|}))$. 
For optimization, we used ADAM \citep{kingmaAdamMethodStochastic2017} with a learning rate of 0.001. 

\section{Results}
We ran experiments for many combinations of operators with the aim of showing that the discussed insights are present in practice.
We report the accuracy of recognizing digits in the test set, the consequent ratio $\mpratio$, and the consequent and antecedent correctly updated ratios $\mpupdateratio$ and $\mtupdateratio$. 
We train for at most 70.000 iterations (or until convergence).  
The purely supervised baseline has a test accuracy of $95.18\%\pm 0.204$, and runs for about 35 minutes. 
Semi-supervised methods should improve upon this baseline to be useful. 
Our implementation including \dfl runs for 1 hour and 52 minutes.

\subsection{Symmetric Configurations}
First, we consider several \textit{symmetric} configurations, where the conjunction is a t-norm $T$, disjunction the dual t-conorm of $T$, universal aggregation the extended t-norm $A_T$, existential aggregation the extended t-conorm $E_S$ and the implication either is the S-implication based on the t-conorm or the R-implication based on the t-norm. 
For example, for $T_P$ we use $T_P$ for conjunction, $S_P$ for disjunction, $\logprod=\log\circ A_{T_P}$ for aggregation and $I_{RC}$ for implication. 
Symmetric configurations will retain many equivalence relations in fuzzy logic, unlike when one would choose arbitrary configuration of operators. 
\subsubsection{Symmetric Configurations on \pred{same} problem}
\begin{table}
    \centering
    \begin{tabular}{l...c|....}
     & \multicolumn{4}{c}{S-Implications} & \multicolumn{4}{c}{R-Implications} \\
    \hline
                       & \mc{Accuracy}   & \mc{$\mpratio$} & \mc{$\mpupdateratio$} & \mc{$\mtupdateratio$}  
                       & \mc{Accuracy}   & \mc{$\mpratio$} & \mc{$\mpupdateratio$} & \mc{$\mtupdateratio$}  \\
    \hline
    $T_G$              & 95.3            & 0.32            & 0.31                   & 0.83                           
                       & 95.0            & 1               & 0.11                  & - \\ 
    $T_P$              & \bft{96.5}      & 0.08            & 0.72                  & \textbf{0.99}  
                       & 94.8            & 0.62            & 0.04                  & 0.96     \\
    $T_{LK}$           & 94.9            & 0.5             & \bft{0.86}            & 0.12                           
                       & 94.9            & 0.5             & \bft{0.86}            & 0.12                \\  
    $T_Y,\ p=1.5$      & 95.2            &                 &                        &                            
                       & 95.2            &                 &                       &  \\
    $T_Y,\ p=2$        & 77.7            & 0.20            & 0.51                  & 0.75                           
                       & 95.0            & 0.62            & 0.61                  & 0.46 \\
    $T_Y,\ p=20$       & 95.6            & 0.02            & 0.54                  & 0.75
                       & 95.5            & 0.53            & 0.01                  & \bft{0.99}        \\
    $T_{Nm}$           & 95.2            &                &                      &  
                       & 95.2            &                &                      &    \\
    \hline
    \end{tabular}
    \caption{Results on the \pred{same} problem for several symmetric configurations with either S-implications or R-implications. For all, $w_{\dfl}=1$ except for $T_P$, for which $w_{\dfl}=10$.}
    \label{table:mnist_symmetric}
\end{table}

All configurations are run with $w_{\dfl}=1$ except for $T_P$ which is run using $w_{\dfl}=10$. The results on the $\pred{same}$ problem can be found in Table \ref{table:mnist_symmetric}.
One general observation that can be made is that S-implications seem to work much better than R-implications. The only configuration with R-implications that outperform the supervised baseline is $T_Y$, $p=20$, but here the S-implication performs similar to it.
We hypothesize this is because the derivatives of R-implications vanish whenever $a \leq c$. 

The Gödel t-norm performs on par with the supervised baseline. This is because the min aggregator is single-passing, like the configuration as a whole. 
The single instance which receives a derivative might just be an exception as argued in Section \ref{sec:aggr_min} and evident from the low values of $\mpupdateratio$ and $\mtupdateratio$. 

The \luk\ t-norm performs worse than the supervised baseline. 
Since $A_{LK}$ either has a derivative of 0 or 1 everywhere, the total gradient is very large when it does not vanish. 
By the definition of $I_{LK}$, $\mpratio=\frac{1}{2}$ as the consequent and negated antecedent derivatives are equal (see Equation \ref{eq:impl_deriv_luk}). $\mpupdateratio$ is very low with only 0.01, which is worse than random guessing. As half of the gradient is MP reasoning, that half is nearly always incorrect.
The performance of the Yager t-norm seems highly dependent on the choice of the parameter $p$. For $p=20$ the top performance is quite a bit higher than the baseline. 
The lower the value of $p$, the more likely it is that the derivative of the universal aggregator vanishes. 
However, for $p=2$, the results are even worse than the \luk\ t-norm, which corresponds to $p=1$, while the derivative with $p=1.5$ simply vanishes throughout the whole run.

The product t-norm performs best and also has the highest values for $\mpupdateratio$ and $\mtupdateratio$. 
To a large extend this is because the log-product aggregator is very effective, as other symmetric configurations also perform much better with it \ref{sec:mnist_symmetric_aggregator}. 
Finally, the Nilpotent t-norm performs exactly like the supervised baseline since the derivative of the universal aggregator vanished during the complete training run. 

\subsubsection{Symmetric configurations on the \pred{sum9} problem}
\begin{table}
    \centering
    \begin{tabular}{l....}
    \hline
                   & \mc{Accuracy}                  & \mc{$\mpratio$} & \mc{$\mpupdateratio$} & \mc{$\mtupdateratio$} \\ \hline
    $T_G$          & 95.2                           & 0.31            & 0.44                  & 0.78             \\
    $T_P$          & \bft{96.1}                     & 0.13            & 0.80                  & 0.95             \\
    $T_{LK}$       & 95.2                           &                 &                       &                 \\
    $T_Y,\ p=1.5$    & 95.2                           &                 &                       &              \\
    $T_Y,\ p=2$    & 95.2                           &                 &                       &              \\
    $T_Y,\ p=20$   & 95.5                          & 0.99            & 0.82                  & 0.71             \\
    $T_{Nm}$       & 95.2                           &                 &                       &   \\
    \hline
    \end{tabular}
    \caption{Results on the \pred{sum9} problem for several symmetric configurations using the S-implication. $w_{\dfl} = 1$ except for $T_P$ with $w_{\dfl}=10$.}
    \label{table:sum9_symmetric}
\end{table}
In addition to the \pred{same} problem, we also run the symmetric configurations on the \pred{sum9} problem to be able to also take into account how existential quantification and disjunction behave. 
The results are in Table \ref{table:sum9_symmetric}.
These closely reflect the results for the $\pred{same}$ problem. 
Again, the only configurations that clearly outperform the supervised baseline are the product t-norm and the Yager t-norm with $p=20$, with the product t-norm being the most promising candidate.
Furthermore, in addition to the Nilpotent t-norm, the derivatives of the \luk\ t-norm and Yager t-norm with $p=2$  vanish throughout the whole run.

\subsection{Individual operators}
\label{sec:experiment_individual}
We also perform several experiments where we investigate the contribution of specific fuzzy operators without regard to whether the resulting configurations are sensible in a logical sense.
We do this to better understand how each operator contributes to the learning process.
Throughout this section, we fix the universal aggregation operator to the log-product aggregator $A_{\log T_P}$, the existential aggregation operator to the generalized mean $E_{GM}$ with $p=1.5$, the conjunction and disjunction to $T_Y$ and $S_Y$, also with $p=1.5$, and the implication to the Reichenbach-sigmoidal implication with $s=9$ and $b=-0.5$. 
We select these because of their promise in initial experiments.  
\subsubsection{Aggregation}
\begin{table}
    \centering
    \begin{tabular}{l...c|l....}
     & \multicolumn{4}{c}{Universal aggregation} & & \multicolumn{4}{c}{Existential aggregation} \\
    \hline
                       & \mc{Accuracy}   & \mc{$\mpratio$} & \mc{$\mpupdateratio$} & \mc{$\mtupdateratio$} & 
                       & \mc{Accuracy}   & \mc{$\mpratio$} & \mc{$\mpupdateratio$} & \mc{$\mtupdateratio$}  \\
    \hline
    $A_{T_G}$          & 86.6            & 0.37            & 0.65                  & 0.45  &                          
    $E_{S_G}$          & 95.3            & 0.38            & 0.60                  & 0.66 \\ 
    $A_{\log T_P}$     & \bft{96.3}      & 0.14            & 0.53                  & 0.96   &
    $E_{S_P}$          & 96.1            & 0.15            & 0.68                  & 0.94     \\
    $A_{T_{LK}}$       & 78.2            & 0.00            & \bft{0.94}            & \textbf{0.99}                           &
    $E_{S_{LK}}$       & 95.9            & 0.17            & \bft{0.76}            & 0.87                \\  
    $A_{T_Y},\ p=1.5$  & 79.5            & 0.00            & 0.77                  & 0.98 &                           
    $E_{S_Y}$          & 95.9            & 0.21            & 0.64                  & 0.85 \\
    $A_{T_Y},\ p=2$    & 83.3            & 0.00            & 0.83                  & 0.98 &                           
    $E_{S_Y}$          & 96.3            & 0.27            & 0.59                  & 0.81 \\
    $A_{T_Y},\ p=20$   & 84.0            & 0.00            & 0.81                  & 0.98 &
    $E_{S_Y}$          & 96.3            & 0.37            & 0.62                  & 0.70        \\
    $A_{GME},\ p=1.5$  & 96.1            & 0.43            & 0.43                  & 0.76 &
    $E_{GM}$           & \bft{96.9}      & 0.29            & 0.13                  & \bft{0.95} \\
    $A_{RMSE}$         & 96.2            & 0.46            & 0.42                  & 0.72 &
    $E_{GM}$           & 96.7            & 0.29            & 0.14                  & 0.94 \\
    $A_{GME},\ p=20$   & 95.5            & 0.45            & 0.38                  & 0.70 &
    $E_{GM}$           & 96.4            & 0.39            & 0.64                  & 0.70 \\
    $A_{T_{Nm}}$       & 79.8            & 0.34            & 0.50                  & 0.58  & 
    $E_{S_{Nm}}$       & 95.5            & 0.31            &0.58                   & 0.78   \\
    \hline
    \end{tabular}
    \caption{Left: Results on the \pred{same} problem, varying the universal aggregator. For all, $w_{\dfl}=10$. Right: Results on the \pred{sum9} problem, varying the existential aggregator.}
    \label{table:aggregation}
\end{table}
Table \ref{table:aggregation} shows the results when varying the universal aggrator in the \pred{same} problem, and when varying the existential aggregator in the \pred{sum9} problem. The log-product operator with $w_{\dfl}=10$ is the best universal aggregator, with $A_{RMSE}$ trailing behind it slightly. 
Other generalized mean errors are also effective. 
We also see that the single-passing aggregators (minimum and Nilpotent minimum aggregators) and aggregators that vanish on a large part of their domain (Yager-based and Nilpotent minimum) all perform poorly. 
However, curiously the Yager-based aggregators have very high $\mpupdateratio$ and $\mtupdateratio$.

The generalized means have the best result for existential aggregation, with $p=1.5$ performing best. 
They manage to properly select the inputs that make the existential quantifier true by softly increasing the largest inputs.
Unlike with universal aggregation, the Yager existential aggregator is also a decent choice. 
The maximum aggregator and Nilpotent maximum aggregator only somewhat outperform the supervised baseline, however, which again is due to them being single-passing. 
\subsubsection{Conjunction and Disjunction}
\begin{table}
    \centering
    \begin{tabular}{l....}
    \hline
                   & \mc{Accuracy}                  & \mc{$\mpratio$} & \mc{$\mpupdateratio$} & \mc{$\mtupdateratio$} \\ \hline
    $T_G$          & 20.4                           & 0.47            & 0.15                  & 0.90             \\
    $T_P$          & 20.3                           & 0.46            & 0.16                  & 0.92             \\
    $T_{LK}$       & \bft{97.0}                     & 0.31            & 0.15                       & 0.93                \\
    $T_Y,\ p=1.5$  & \bft{97.0}                     & 0.29            & 0.11                  & \bft{0.96}             \\
    $T_Y,\ p=2$    & 96.8                           & 0.30                & 0.11              & \bft{0.96}             \\
    $T_Y,\ p=20$   & 94.9                           & 0.66            & 0.16                  & 0.86             \\
    $T_{Nm}$       & 96.9                           & 0.31            & \bft{0.17}                  & 0.93   \\
    \hline
    \end{tabular}
    \caption{Results on the \pred{sum9} problem, varying the t-norm and t-conorm together. }
    \label{table:tnorms}
\end{table}
In Table \ref{table:tnorms}, we compare different t-norms together with their corresponding t-conorms. 
Here, it is the operators that vanish on a large part of their domain that work best, namely Yager t-norms and the Nilpotent minimum. 
These seem to work much better than when used in aggregation since the amount of inputs is much smaller, reducing the probability that the derivative vanishes. 
The product t-norm and Gödel t-norm, which corresponds to weak conjunction and disjunction, seem to do perform very poorly. 
In this table, there is a clear relation where lower $\mpratio$ seem to perform better. 
It is likely lower in the Yager t-norm since the disjunction in the consequent will very often be 1, which happens when the sum in the Yager t-norm hits the boundary. 
\subsubsection{Implications}
\label{sec:experiment_implications}
In Table \ref{table:implications}, we compare different fuzzy implications on the \pred{same} problem. The Reichenbach implication and Yager S-implication work well, both having an accuracy around 97\%. 
The Kleene Dienes and Yager R-implications surpass the baseline as well. 
In these experiments, the sigmoidal-Reichenbach implication, which we run with $s=9$ and $b=-\frac{1}{2}$ performs as well as the normal Reichenbach implication. 
However, we find in \ref{sec:mnist_rcsigmoidal} that with the SGD algorithm, this implication outperforms the Reichenbach implication, reaching 97.3 accuracy. 

\begin{table}
    \centering
    \begin{tabular}{l...c|l....}
     & \multicolumn{4}{c}{S-Implications} & & \multicolumn{4}{c}{R-Implications} \\
    \hline
                       & \mc{Accuracy}   & \mc{$\mpratio$} & \mc{$\mpupdateratio$} & \mc{$\mtupdateratio$}  
                       & & \mc{Accuracy}   & \mc{$\mpratio$} & \mc{$\mpupdateratio$} & \mc{$\mtupdateratio$}  \\
    \hline
    $I_{KD}$           & 95.7            & 0.08            & \bft{0.82}            & 0.98                           &
    $I_G$              & 90.5            & 1               & 0.05                  &  \\ 
    $I_{RC}$           & \bft{96.3}      & 0.09            & 0.73                  & \textbf{0.98}  &
    $I_{GG}$           & 93.6            & 0.99            & 0.00                  & 0.87     \\
    $I_{LK}$           & 96.0            & 0.5             & 0.07                   & 0.93 &                           
    $I_{LK}$             & 96.0            & 0.5             & 0.07                   & 0.93                   \\  
    $I_Y,\ p=1.5$      & 96.1            & 0.14            & 0.77                  & 0.96 &                          
    $I_{RY}$           & 96.1            & 0.14            & 0.14                  & 0.66 \\
    $I_Y,\ p=2$        & 96.2            & 0.13            & 0.82                  & 0.97 &                           
    $I_{RY}$           & 95.7            & 0.13            & 0.64                  & 0.38 \\
    $I_Y,\ p=20$       & 95.1            & 0.57            & 0.48                  & 0.98 &
    $I_{RY}$           & 96.0            & 0.14            & 0.65                  & 0.38        \\
    $I_{FD}$           & 96.1            & 0.19            & 0.68                  & 0.91           &
    $I_{FD}$           & 96.1            & 0.19            & 0.68                  & 0.91   \\
    $\sigma_{I_{RC}}$  & \bft{96.3}      & 0.14           & 0.53                 & 0.96 &&&&& \\
    \hline
    \end{tabular}
    \caption{Results on the \pred{same} problem for several symmetric configurations with either S-implications or R-implications.}
    \label{table:implications}
\end{table}

As argued in Sections \ref{sec:godel_implications} and \ref{sec:prod-implication}
, the Gödel implication and Goguen implication have worse performance than the supervised baseline by making many incorrect modus ponens inferences. 
While the derivatives of $I_{LK}$ and $I_G$ only differ in that $I_G$ disables the derivatives with respect to negated antecedent, $I_{LK}$ performs among the better while $I_G$ is the worst test implication, suggesting that the derivatives with respect to the negated antecedent are required to successfully applying \dfl. 
Note that S-implications tend to perform better than R-implications, in particular for the Gödel t-norm and the product t-norm. 
This could be because they inherently balance derivatives with respect to the consequent and negated antecedent by being contrapositive differentiable symmetric.

%
\subsubsection{Additional experiments}
In \ref{sec:mnist_rcsigmoidal} we investigate the parameters $s$ and $b_0$ of the sigmoidal-Reichenbach implication. 
We find here that it is the best performing implication on the \pred{same} problem when using the vanilla SGD optimizer, reaching 97.3\% accuracy. 
Furthermore, in \ref{sec:mnist_formulas_experiments} we investigate for the \pred{same} problem what the influence of each rule is to the learning process. 

We also ran the \pred{sum9} problem on the vanilla SGD optimizer with the fixed configuration from Section \ref{sec:experiment_individual}, which reaches 97.7\% accuracy.
This is significantly higher than with ADAM, confirming our findings for \pred{same} problem in Section \ref{sec:experiment_implications}.
Finally, we ran with the same settings using both the formulas from the \pred{same} problem and the \pred{sum9} problem. 
This has the best accuracy we find in our experiments with \textbf{98.0\%}, and confirms that adding more background knowledge increases the final performance.
\subsection{Analysis}

We plot the accuracy of the different configurations with respect to $\mpcorupdate$ and $\mtcorupdate$ in Figures \ref{fig:cucons_accuracy} and \ref{fig:cuant_accuracy}. 
The blue dots represent runs on the \pred{same} problem, while the red dots represent runs on the \pred{sum9} problem. Figure \ref{fig:cuant_accuracy} shows a positive correlation, suggesting that it is vital to the learning process that updates going into the antecedent are correct. 
Although there seems to be a slight positive correlation in Figure \ref{fig:cucons_accuracy} for the \pred{same} problem, it is not as pronounced. Furthermore, it seems that for the \pred{sum9}, this correlation is negative instead, as the configurations with the highest accuracy  have low values of $\mpcorupdate$. 

\begin{figure}
\centering
\begin{subfigure}[b]{\graphwidth}
\includegraphics[width=\linewidth]{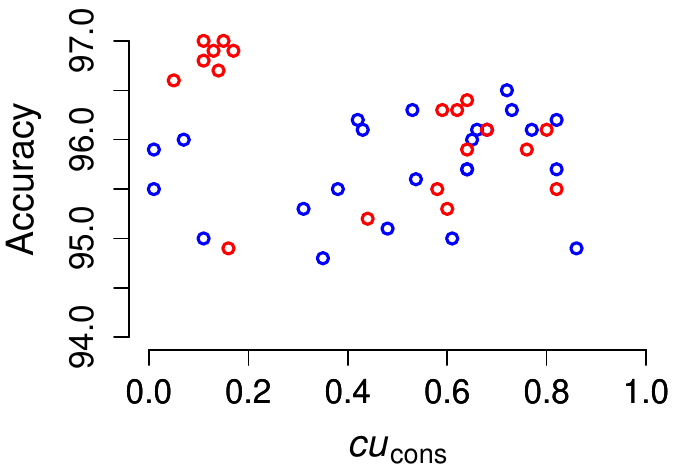}
\caption{Plot of $\mpupdateratio$ to accuracy.}
\label{fig:cucons_accuracy}
\end{subfigure}
\begin{subfigure}[b]{\graphwidth}
\includegraphics[width=\linewidth]{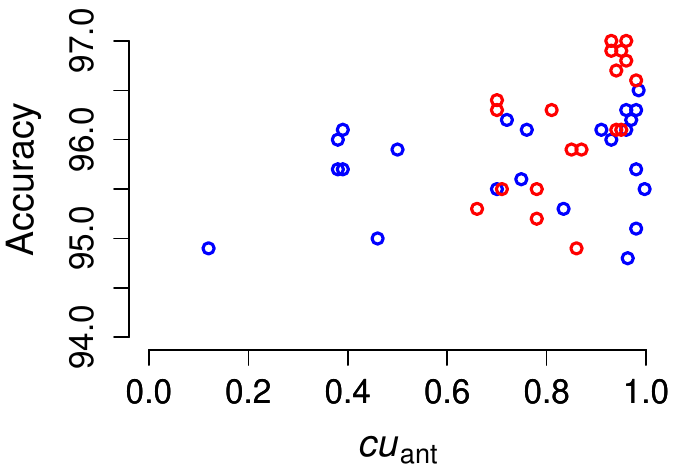}
\caption{Plot of $\mtupdateratio$ to accuracy.}
\label{fig:cuant_accuracy}
\end{subfigure}%
\\
\begin{subfigure}[b]{\graphwidth}
\includegraphics[width=\linewidth]{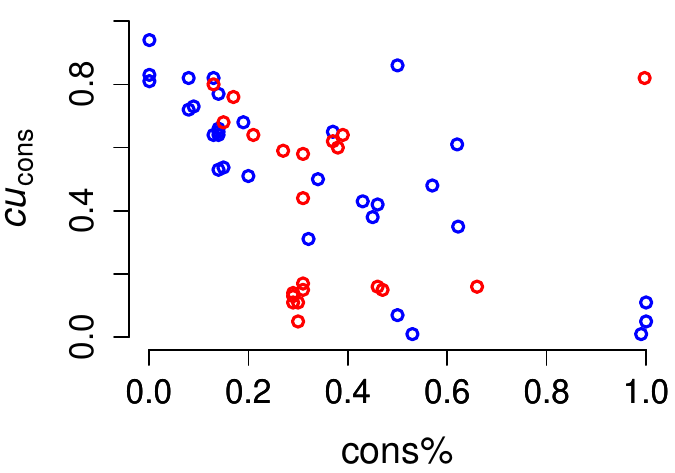}
\caption{Plot of $\mpratio$ to $\mpupdateratio$.}
\label{fig:mpmt_cucons}
\end{subfigure}
\begin{subfigure}[b]{\graphwidth}
\includegraphics[width=\linewidth]{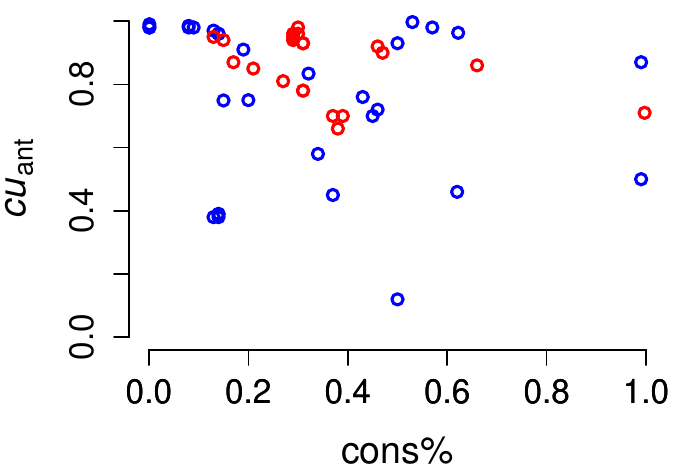}
\caption{Plot of $\mpratio$ to $\mtupdateratio$.}
\label{fig:mpmt_cuant}
\end{subfigure}
\caption{We plot several of the analytical measures to find their relations. Blue dots represent runs on the \pred{same} problem while red dots represent runs on the \pred{sum9} problem.}
\label{fig:analyze_measures}
\end{figure}

We plot all experimental values of $\mpratio$ to the values of $\mpupdateratio$ and $\mtupdateratio$ in Figures \ref{fig:mpmt_cucons} and \ref{fig:mpmt_cuant}. 
For both, there seems to be a negative correlation. 
Apparently, a larger consequent ratio decreases the correctness of the updates. 
In \ref{sec:mnist_rcsigmoidal} we find, when experimenting with the value of $s$, that this could be because for lower values of $\mpratio$, a smaller portion of the reasoning happens in the corners around $a=0,\ c=0$ and $a=1,\ c=1$, and more for instances that the agent is less certain about. 
Since all S-implications have strong derivatives at both these corners (Proposition \ref{prop:diff_left_neutral}), this phenomenon is likely present in other S-implications.

This all suggests we need to properly balance the contribution of updates to the antecedent and consequent. 
Since usually, as reasoned in Section \ref{sec:implication_challenges}, derivatives with respect to the antecedent are more common, this balance should be reflected in the experimental ratio between these updates. 

\subsection{Conclusions}
We have run experiments on many configurations of hyperparameters to explore what works and what does not. 
The only well performing fully symmetric option is the product t-norm with the Reichenbach implication. 
If we are willing to forego symmetry, we find that the choice of the aggregators is the most important factor for performance. 
For universal aggregation, we recommend the log-product aggregator, while for existential quantification we recommend the generalized mean with a value of $p$ of somewhere between 1 and 2. 
We found that it is especially important to choose aggregation operators that do not vanish on a large part of their domain, and that are not single-passing. 
In our experiments, a well tuned sigmoidal-Reichenbach implication coupled with vanilla SGD proved to be the most effective fuzzy implication.
In general, we recommend choosing S-implications above R-implications.
For conjunction and disjunction, we recommend tuning the Yager t-norm, although this value can be dependent on the complexity of the formulas to prevent the derivative from vanishing during the whole run. 

Although \dfuzz significantly improves on the supervised baseline and is thus suited for semi-supervised learning, it is not currently competitive with state-of-the-art methods like Ladder Networks \pcite{rasmus2015semi} which has an accuracy of 98.9\% for 100 labeled pictures and 99.2\% for 1000. 




\section{Related Work}
\label{chapter:related_work}
\dfuzz falls into the discipline of Statistical Relational Learning \pcite{getoor2007}, which concerns models that can reason under uncertainty and learn relational structures like graphs. 

\subsection{\dfuzz}
\label{sec:rel-dfuzz}
Special cases of \dfl have been researched in several papers under different names. Logic Tensor Networks (LTN) \pcite{badreddineLogicTensorNetworks2020,Serafini2016} implements function symbols and uses neural model to interpret predicates. LTN is applied to weakly supervised learning on Scene Graph Parsing \pcite{Donadello2017} and transfer learning in Reinforcement Learning \pcite{badreddine2019injecting}.

Semantic-based regularization (SBR) \pcite{diligenti2017b} applies \dfl to kernel machines. They use R-implications and the mean aggregator. \tcite{sen2008} applies SBR to collective classification by predicting using a trained deep learning model, and then optimizes the \dfl loss to find new truth values. This ensures predictions are consistent with the formulas during test-time. 

\tcite{marra2019learning} 
uses t-norm Fuzzy Logics, where the R-implication is used alongside weak disjunction. By using t-norms based on generator functions, the satisfiability computation can be simplified and generalizations of common loss functions can be found. \tcite{marra2018} applies \dfl to image generation. It uses the product t-norm, the log-product aggregator and the Goguen implication. By using function symbols that represent generator neural networks, they create constraints that are used to create a semantic description of an image generation problem. \tcite{Rocktaschel2015b} uses the product t-norm and Reichenbach implication for relation extraction by using an efficient matrix embedding of the rules. \tcite{guo2016} extends this to link prediction and triple classification by using a margin-based ranking loss for implications. 

\tcite{Demeester2016} uses a regularization technique equivalent to the \luk\ implication. Instead of using existing data, it finds a loss function which does not iterate over objects, yet can guarantee that the rules hold. This is very scalable, but can only model simple implications. A promising approach is using adversarial sets \pcite{minervini2017}, which is a set of objects from the domain that do not satisfy the knowledge base. These are probably the most informative objects. It uses gradient descent to find objects that \textit{minimize} the satisfiability. The parameters of the deep learning model are then updated so that it predicts consistent with the knowledge base on this adversarial set. A benefit of this approach is that it does not have to iterate over instances that already satisfy the constraints. Adversarial sets are applied to natural language interpretation in \pcite{minervini2018}. Both papers use the \luk\ implication and Gödel t-norm and t-conorm. They are not able to infer new labels on existing unlabeled data as they use artificial data, but these methods are not orthogonal and can be used jointly. 


\subsection{Neuro-symbolic methods using Fuzzy Logic Operators}
Posterior regularization \pcite{Ganchev2010,P16-1228} is a framework for weakly-supervised learning on structured data. 
It projects the output of a deep learning model to a `rule-regularized subspace' to make it consistent with the knowledge base. This output is used as a label for the deep learning model to imitate. Unlike this paper, it does not compute derivatives over the computation of the satisfaction of the knowledge base. \tcite{Marra2019} and \tcite{10.1007/978-3-030-29908-8_43} instead use gradient descent for the projection. Therefore, unlike earlier methods for posterior regularization, derivatives with respect to the operators are used. They learn relative formula weights jointly with the parameters of the deep learning model. 

\tcite{arakelyanComplexQueryAnswering2021} uses t-norms, t-conorms and existential quantification to answer queries by finding what entity embedding has the highest truth value of a given query. This search is done using gradient descent. 
By comparing what entity embedding best fits the optimized entity embedding, the authors can answer complex FOL queries. The authors either use the product or the Gödel t-norms. 

Another recent work which employs fuzzy logic operators in a neuro-symbolic setting is Logical Neural Networks \citep{riegel2020logical}. 
This work stands orthogonal to our work, as the foremost distinction is that they employ logics on the low-level (i.e., logical connectives as neurons and neural activation functions) while we employ it on the higher level (i.e., in defining the loss function). 
They limit their work to the propositional level for simplification purposes, although they argue that extending it to relational level is straightforward.

$\partial$ILP \pcite{Evans2018} is a differentiable inductive logic programming that uses the product t-norm and t-conorm to do differentiable inference. The Neural Theorem Prover \pcite{Rocktaschel2017} does differentiable proving of queries and combines different proof paths using the Gödel t-norm and t-conorm.  \tcite{Sourek2015LiftedNetworks} also introduces a method for differentiable query proving, with learnable weights for formulas. They use operators inspired by fuzzy logic and transformed by the sigmoid function. 

There is a vast literature on Fuzzy Neural Networks \pcite{Jang1993a,Jang1997a,Lin1991a} that replace standard neural network neurons with neurons based on fuzzy logic. Some neurons use fuzzy logic operators which are differentiated through if the networks are trained using backpropagation. 

\subsection{\dprob}
\label{sec:related_probabilistic}
Some approaches use probabilistic logics instead of fuzzy logics and interpret predicates probabilistically. 
As deep learning classifiers can model probability distributions, probabilistic logics could be a more natural choice than fuzzy logics. 
DeepProbLog \pcite{DBLP:conf/nips/2018} is a probabilistic logic programming language with neural predicates that compute the probabilities of ground atoms. 
It supports automatic differentiation which can be used to back-propagate from the loss at a query predicate to the deep learning models that implement the neural predicates, similar to \dfl. 
It also supports probabilistic rules which can handle exceptions to rules.
We compare another differentiable probabilistic logic called Semantic Loss \pcite{pmlr-v80-xu18h} in \ref{appendix:prl-sl} and show similarities between it and \dfl using operators based on the product t-norm. This similarity suggests that many practical problems that \dpfl has are also present in Semantic Loss. %
They apply Semantic Loss to MNIST semi-supervised learning with a different knowledge base than ours.  
As inference is exponential in the size of the grounding for probabilistic logics, both approaches use an advanced compilation technique \pcite{Darwiche2011} to make inference feasible for larger problems.

\label{chapter:conclusions}

\section{Discussion}

This paper presented theoretical results of \dfuzz operators and then evaluated their behavior on semi-supervised learning.
We now discuss some problems with deploying solutions using \dfl.

\dfl can be seen as a form of multi-objective optimization \pcite{hwang2012multiple}. In the \dfl loss (Equation \ref{eq:lossrl}) we sum up the valuations of different formulas, each of which is a separate objective. 
Each of these objectives can be weighted differently, resulting in wildly varying loss landscapes. 
Having so many objectives requires significant hyperparameter tuning. A method capable of learning relative formula weights jointly like \pcite{Marra2019, 10.1007/978-3-030-29908-8_43, Sourek2015LiftedNetworks}, could solve this problem.

A second challenge is related to the class imbalance problem \pcite{japkowicz2002class, buda2018systematic}. We argued in Section \ref{sec:implication_challenges} that for a significant portion of common-sense background knowledge, the modus tollens case is by far the most common. 
Our \pred{same} problem indeed showed that most well-performing implications have a far larger derivative with respect to the negated antecedent than to the consequent. 
This imbalance will only increase for more complex problems. 
However, simply removing derivatives with respect to the antecedent does not seem to be the solution. A reason for this could be that those are usually correct, unlike derivatives with respect to the consequent. In fact, we found in \ref{sec:mnist_formulas_experiments} that the formula in which the digits are in the antecedent performs better on its own than the formula in which the digits are in the consequent, even though the model could not learn from any new positive examples. 

Although we have focused on experimenting with the accuracy of the derivatives of the implication, it should be noted that the derivatives of the disjunction operator make a choice as well. For example, if the agent observes a walking object and the supervisor knows that only humans and animals can walk, how is the supervisor supposed to choose whether it is a human or an animal? Here, similar imbalances exist in the different possible classes: There might be more images of humans than of animals.

Further, we pose whether it is more important that we choose operators based on the performance on the task at hand, or based on its logical properties. 
The best configuration uses operators based on both the product and Yager t-norms. 
The product t-norm is the only viable symmetric choice in our experiments. 
The largest benefit of a `symmetric' choice of operators is that the truth value of formulas that are logically equivalent in classical logic will be equal. 
This makes it easier to analyze how the background knowledge will behave and does not require putting it in a particular form.  

As a final remark, noteworthy is the interpretation of truth values. As aforementioned, the logic we use is fuzzy logic which was originally aimed to address logical reasoning in the presence of vagueness rather than probabilistic uncertainty. The truth values derived using fuzzy operators, therefore, are not probabilistic (see, for instance, \cite[p.~4]{hajek1998metamathematics})
\footnote{Indeed, when reasoning about belief, using fuzzy logic semantics instead of probabilistic logic semantics straight out-of-the-box, can yield undesirable results: Consider an event $a$ where $p(a)$ (probability of $a$) is 0.5. Now consider a disjunction, where $p(a \vee a)$ has the value $0.5$. However, in \luk{} logic, $S(a, a)$ would yield 1. }.

However, since a considerably large amount of problems addressed by machine learning literature is probabilistic (as it has mathematical origins in statistics), the classification task used in our running example is also of probabilistic origin. 
With this choice we also aimed to respect the recent literature: Applications of fuzzy operators on a general set of problems which are not necessarily fuzzy is not uncommon in neuro-symbolic AI. 
Examples include \citep{Serafini2016}, \citep{riegel2020logical}, \cite{arakelyanComplexQueryAnswering2021}, among others cited in Section \ref{chapter:related_work}.

\section{Conclusion}
We analyzed \dfuzz in order to understand how reasoning using logical formulas behaves in a differentiable setting. We examined how the properties of a large amount of different operators affect \dfl. 
We have found substantial differences between the properties of a large number of such \dfuzz operators, and we showed that many of them, including some of the most popular operators, are highly unsuitable for use in a differentiable learning setting. 
By analyzing aggregation functions, we found that the log-product aggregator and the RMSE aggregator have convenient connections to both fuzzy logic and machine learning and can deal with outliers. 
Next, we analyzed conjunction and disjunction operators and found several strong candidates. 
In particular, the Gödel t-norm and t-conorm are a simple choice, and that the Yager t-norm and the product t-conorm have intuitive derivatives.

We noted an interesting imbalance between derivatives with respect to the negated antecedent and the consequent of the implication. Because the modus tollens case is much more common, we conclude that a large part of the useful inferences on the MNIST experiments are made by decreasing the antecedent, or by `modus tollens reasoning'. Furthermore, we found that derivatives with respect to the consequent often increase the truth value of something that is false as the consequent is false in the majority of times. Therefore, we argue that `modus tollens reasoning' should be embraced in future research. As a possible solution to problems caused by this imbalance, we introduced a smoothed fuzzy implication called the Reichenbach-sigmoidal implication. 

Experimentally, we found that the product t-norm is the only t-norm that can be used as a base for all choices of operators. The product t-conorm and the Reichenbach implication have intuitive derivatives  that correspond to inference rules from classical logic, and the log-product aggregator is the most effective universal aggregation operator. 

In order to gain the largest improvements over a supervised baseline however, we had to abandon the normal symmetric configurations of norms, where t-norms, t-conorms, implications and the aggregation operators satisfy the usual algebraic relations. Instead, we had to resort to non-symmetric configurations where operators based on different t-norms are combined. 
The Reichenbach-sigmoidal implication performs best in our experiments. 
Its hyperparameters can be tweaked to decrease the imbalance of the derivatives with respect to the negated antecedent and consequent. 
For existential quantification, we found that the general mean error performs best, and for conjunction and disjunction the family of Yager t-norms and the Nilpotent minimum has the highest final accuracy. 

We believe a proper empirical comparison of different methods that introduce background knowledge through logic could be useful to properly understand the details, performance, possible applications and challenges of each method. 
Secondly, we believe more work is required in using background knowledge to help deep models train on real-world problems. One research direction would be to develop methods that can properly deal with exceptions. An approach in which formula importance weights can be learned could be used to distinguish between relevant and irrelevant formulas in the background knowledge, and probabilistic instead of fuzzy logics could be a more natural fit. 
Lastly, additional research on the vast space of fuzzy logic operators might find more properties that are useful in \dfl.

\medskip
\section*{Declaration of competing interest}
The authors declare that they have no known competing financial interests or personal relationships that could have appeared to influence the work reported in this paper.
\section*{Acknowledgements}
We sincerely thank all the anonymous reviewers whose comments substantially improved the content and the quality of this manuscript.
This work is partly funded by the MaestroGraph research programme with project number 612.001.552, which is financed by the Dutch Research Council (NWO). Erman Acar is generously funded by the Hybrid Intelligence Project which is financed by the Dutch Ministry of Education, Culture and Science. This work is also supported by the DAS-5 distributed supercomputer \pcite{bal2016medium}.
\bibliographystyle{elsarticle-harv}


\bibliography{references}

\begin{thebibliography}{76}
\expandafter\ifx\csname natexlab\endcsname\relax\def\natexlab#1{#1}\fi
\providecommand{\url}[1]{\texttt{#1}}
\providecommand{\href}[2]{#2}
\providecommand{\path}[1]{#1}
\providecommand{\DOIprefix}{doi:}
\providecommand{\ArXivprefix}{arXiv:}
\providecommand{\URLprefix}{URL: }
\providecommand{\Pubmedprefix}{pmid:}
\providecommand{\doi}[1]{\href{http://dx.doi.org/#1}{\path{#1}}}
\providecommand{\Pubmed}[1]{\href{pmid:#1}{\path{#1}}}
\providecommand{\bibinfo}[2]{#2}
\ifx\xfnm\relax \def\xfnm[#1]{\unskip,\space#1}\fi
\bibitem[{Arakelyan et~al.(2021)Arakelyan, Daza, Minervini and
  Cochez}]{arakelyanComplexQueryAnswering2021}
\bibinfo{author}{Arakelyan, E.}, \bibinfo{author}{Daza, D.},
  \bibinfo{author}{Minervini, P.}, \bibinfo{author}{Cochez, M.},
  \bibinfo{year}{2021}.
\newblock \bibinfo{title}{Complex query answering with neural link predictors},
  in: \bibinfo{booktitle}{International Conference on Learning
  Representations}.
\bibitem[{Badreddine et~al.(2020)Badreddine, d'Avila Garcez, Serafini and
  Spranger}]{badreddineLogicTensorNetworks2020}
\bibinfo{author}{Badreddine, S.}, \bibinfo{author}{d'Avila Garcez, A.},
  \bibinfo{author}{Serafini, L.}, \bibinfo{author}{Spranger, M.},
  \bibinfo{year}{2020}.
\newblock \bibinfo{title}{Logic tensor networks}.
\newblock \bibinfo{journal}{arXiv:2012.13635 [cs]}
  \href{http://arxiv.org/abs/2012.13635}{{\tt arXiv:2012.13635}}.
\bibitem[{Badreddine and Spranger(2019)}]{badreddine2019injecting}
\bibinfo{author}{Badreddine, S.}, \bibinfo{author}{Spranger, M.},
  \bibinfo{year}{2019}.
\newblock \bibinfo{title}{{Injecting Prior Knowledge for Transfer Learning into
  Reinforcement Learning Algorithms using Logic Tensor Networks}}.
\newblock \bibinfo{journal}{arXiv preprint arXiv:1906.06576} .
\bibitem[{Bal et~al.(2016)Bal, Epema, de~Laat, van Nieuwpoort, Romein,
  Seinstra, Snoek and Wijshoff}]{bal2016medium}
\bibinfo{author}{Bal, H.}, \bibinfo{author}{Epema, D.},
  \bibinfo{author}{de~Laat, C.}, \bibinfo{author}{van Nieuwpoort, R.},
  \bibinfo{author}{Romein, J.}, \bibinfo{author}{Seinstra, F.},
  \bibinfo{author}{Snoek, C.}, \bibinfo{author}{Wijshoff, H.},
  \bibinfo{year}{2016}.
\newblock \bibinfo{title}{{A medium-scale distributed system for computer
  science research: Infrastructure for the long term}}.
\newblock \bibinfo{journal}{Computer} \bibinfo{volume}{49},
  \bibinfo{pages}{54--63}.
\bibitem[{Ball(1997)}]{ball1997elementary}
\bibinfo{author}{Ball, K.}, \bibinfo{year}{1997}.
\newblock \bibinfo{title}{{An elementary introduction to modern convex
  geometry}}.
\newblock \bibinfo{journal}{Flavors of geometry} \bibinfo{volume}{31},
  \bibinfo{pages}{1--58}.
\bibitem[{Besold et~al.(2017)Besold, Garcez, Bader, Bowman, Domingos, Hitzler,
  Kuehnberger, Lamb, Lowd, Lima, de~Penning, Pinkas, Poon and
  Zaverucha}]{Besold2017a}
\bibinfo{author}{Besold, T.R.}, \bibinfo{author}{Garcez, A.d.},
  \bibinfo{author}{Bader, S.}, \bibinfo{author}{Bowman, H.},
  \bibinfo{author}{Domingos, P.}, \bibinfo{author}{Hitzler, P.},
  \bibinfo{author}{Kuehnberger, K.U.}, \bibinfo{author}{Lamb, L.C.},
  \bibinfo{author}{Lowd, D.}, \bibinfo{author}{Lima, P.M.V.},
  \bibinfo{author}{de~Penning, L.}, \bibinfo{author}{Pinkas, G.},
  \bibinfo{author}{Poon, H.}, \bibinfo{author}{Zaverucha, G.},
  \bibinfo{year}{2017}.
\newblock \bibinfo{title}{{Neural-Symbolic Learning and Reasoning: A Survey and
  Interpretation}}.
\newblock \bibinfo{journal}{arXiv preprint arXiv:1711.03902} \URLprefix
  \url{http://arxiv.org/abs/1711.03902}.
\bibitem[{Bishop(2006)}]{Bishop2006PatternLearning}
\bibinfo{author}{Bishop, C.M.}, \bibinfo{year}{2006}.
\newblock \bibinfo{title}{{Pattern Recoginiton and Machine Learning}}.
\bibitem[{Brock et~al.(2018)Brock, Donahue and Simonyan}]{brock2018large}
\bibinfo{author}{Brock, A.}, \bibinfo{author}{Donahue, J.},
  \bibinfo{author}{Simonyan, K.}, \bibinfo{year}{2018}.
\newblock \bibinfo{title}{{Large scale gan training for high fidelity natural
  image synthesis}}.
\newblock \bibinfo{journal}{arXiv preprint arXiv:1809.11096} .
\bibitem[{Buda et~al.(2018)Buda, Maki and Mazurowski}]{buda2018systematic}
\bibinfo{author}{Buda, M.}, \bibinfo{author}{Maki, A.},
  \bibinfo{author}{Mazurowski, M.A.}, \bibinfo{year}{2018}.
\newblock \bibinfo{title}{{A systematic study of the class imbalance problem in
  convolutional neural networks}}.
\newblock \bibinfo{journal}{Neural Networks} \bibinfo{volume}{106},
  \bibinfo{pages}{249--259}.
\bibitem[{Calvo et~al.(2002)Calvo, Koles{\'a}rov{\'a}, Komorn{\'i}kov{\'a} and
  Mesiar}]{calvoAggregationOperatorsProperties2002}
\bibinfo{author}{Calvo, T.}, \bibinfo{author}{Koles{\'a}rov{\'a}, A.},
  \bibinfo{author}{Komorn{\'i}kov{\'a}, M.}, \bibinfo{author}{Mesiar, R.},
  \bibinfo{year}{2002}.
\newblock \bibinfo{title}{Aggregation operators: {{Properties}}, classes and
  construction methods}, in: \bibinfo{editor}{Calvo, T.},
  \bibinfo{editor}{Mayor, G.}, \bibinfo{editor}{Mesiar, R.} (Eds.),
  \bibinfo{booktitle}{Aggregation Operators: {{New}} Trends and Applications}.
  \bibinfo{publisher}{{Physica-Verlag HD}}, \bibinfo{address}{{Heidelberg}},
  pp. \bibinfo{pages}{3--104}.
\bibitem[{Cignoli et~al.(2002)Cignoli, Esteva, Godo and
  Montagna}]{cignoliClassLeftcontinuousTnorms2002}
\bibinfo{author}{Cignoli, R.}, \bibinfo{author}{Esteva, F.},
  \bibinfo{author}{Godo, L.}, \bibinfo{author}{Montagna, F.},
  \bibinfo{year}{2002}.
\newblock \bibinfo{title}{On a class of left-continuous t-norms}.
\newblock \bibinfo{journal}{Fuzzy Sets and Systems. An International Journal in
  Information Science and Engineering} \bibinfo{volume}{131},
  \bibinfo{pages}{283--296}.
\newblock \DOIprefix\doi{10.1016/S0165-0114(01)00215-9}.
\bibitem[{Cintula et~al.(2011)Cintula, H\'{a}jek and
  Nogura}]{cintula2011handbook}
\bibinfo{author}{Cintula, P.}, \bibinfo{author}{H\'{a}jek, P.},
  \bibinfo{author}{Nogura, C.}, \bibinfo{year}{2011}.
\newblock \bibinfo{title}{Handbook of Mathematical Fuzzy Logic, Volume 1}.
\newblock \bibinfo{publisher}{College Publications}.
\bibitem[{Daniele and Serafini(2019)}]{10.1007/978-3-030-29908-8_43}
\bibinfo{author}{Daniele, A.}, \bibinfo{author}{Serafini, L.},
  \bibinfo{year}{2019}.
\newblock \bibinfo{title}{{Knowledge Enhanced Neural Networks}}, in:
  \bibinfo{editor}{Nayak, A.C.}, \bibinfo{editor}{Sharma, A.} (Eds.),
  \bibinfo{booktitle}{PRICAI 2019: Trends in Artificial Intelligence},
  \bibinfo{publisher}{Springer International Publishing},
  \bibinfo{address}{Cham}. pp. \bibinfo{pages}{542--554}.
\bibitem[{Darwiche(2011)}]{Darwiche2011}
\bibinfo{author}{Darwiche, A.}, \bibinfo{year}{2011}.
\newblock \bibinfo{title}{{SDD: A new canonical representation of propositional
  knowledge bases}}.
\newblock \bibinfo{journal}{IJCAI International Joint Conference on Artificial
  Intelligence} ,
  \bibinfo{pages}{819--826}\DOIprefix\doi{10.5591/978-1-57735-516-8/IJCAI11-143}.
\bibitem[{David and Nagaraja(2003)}]{david2004order}
\bibinfo{author}{David, H.A.}, \bibinfo{author}{Nagaraja, H.N.},
  \bibinfo{year}{2003}.
\newblock \bibinfo{title}{{Order statistics, Third Edition}}.
\newblock \bibinfo{journal}{Encyclopedia of Statistical Sciences} .
\bibitem[{Demeester et~al.(2016)Demeester, Rockt{\"{a}}schel and
  Riedel}]{Demeester2016}
\bibinfo{author}{Demeester, T.}, \bibinfo{author}{Rockt{\"{a}}schel, T.},
  \bibinfo{author}{Riedel, S.}, \bibinfo{year}{2016}.
\newblock \bibinfo{title}{{Lifted Rule Injection for Relation Embeddings}}, in:
  \bibinfo{booktitle}{Proceedings of the 2016 Conference on Empirical Methods
  in Natural Language Processing}, \bibinfo{publisher}{Association for
  Computational Linguistics}. pp. \bibinfo{pages}{1389--1399}.
\newblock \URLprefix \url{http://aclweb.org/anthology/D16-1146},
  \DOIprefix\doi{10.18653/v1/D16-1146}.
\bibitem[{Diligenti et~al.(2017a)Diligenti, Gori and Sacca}]{diligenti2017b}
\bibinfo{author}{Diligenti, M.}, \bibinfo{author}{Gori, M.},
  \bibinfo{author}{Sacca, C.}, \bibinfo{year}{2017}a.
\newblock \bibinfo{title}{{Semantic-based regularization for learning and
  inference}}.
\newblock \bibinfo{journal}{Artificial Intelligence} \bibinfo{volume}{244},
  \bibinfo{pages}{143--165}.
\bibitem[{Diligenti et~al.(2017b)Diligenti, Roychowdhury and
  Gori}]{diligenti2017}
\bibinfo{author}{Diligenti, M.}, \bibinfo{author}{Roychowdhury, S.},
  \bibinfo{author}{Gori, M.}, \bibinfo{year}{2017}b.
\newblock \bibinfo{title}{{Integrating Prior Knowledge into Deep Learning}},
  in: \bibinfo{booktitle}{Machine Learning and Applications (ICMLA), 2017 16th
  IEEE International Conference on}, \bibinfo{organization}{IEEE}. pp.
  \bibinfo{pages}{920--923}.
\bibitem[{Donadello et~al.(2017)Donadello, Serafini and Garcez}]{Donadello2017}
\bibinfo{author}{Donadello, I.}, \bibinfo{author}{Serafini, L.},
  \bibinfo{author}{Garcez, A.d.}, \bibinfo{year}{2017}.
\newblock \bibinfo{title}{{Logic Tensor Networks for Semantic Image
  Interpretation}}, in: \bibinfo{booktitle}{IJCAI International Joint
  Conference on Artificial Intelligence}, pp. \bibinfo{pages}{1596--1602}.
\newblock \URLprefix \url{http://arxiv.org/abs/1705.08968}.
\bibitem[{Evans and Grefenstette(2018)}]{Evans2018}
\bibinfo{author}{Evans, R.}, \bibinfo{author}{Grefenstette, E.},
  \bibinfo{year}{2018}.
\newblock \bibinfo{title}{{Learning explanatory rules from noisy data}}.
\newblock \bibinfo{journal}{Journal of Artificial Intelligence Research}
  \bibinfo{volume}{61}, \bibinfo{pages}{65--170}.
\newblock \DOIprefix\doi{10.1613/jair.5477}.
\bibitem[{Ganchev and Gillenwater(2010)}]{Ganchev2010}
\bibinfo{author}{Ganchev, K.}, \bibinfo{author}{Gillenwater, J.},
  \bibinfo{year}{2010}.
\newblock \bibinfo{title}{{Posterior Regularization for Structured Latent
  Variable Models}}.
\newblock \bibinfo{journal}{Journal of Machine Learning Research}
  \bibinfo{volume}{11}, \bibinfo{pages}{2001--2049}.
\newblock \URLprefix \url{http://dl.acm.org/citation.cfm?id=1756006.1859918}.
\bibitem[{Garcez et~al.(2012)Garcez, Broda and Gabbay}]{garcez2012neural}
\bibinfo{author}{Garcez, A.S.d.}, \bibinfo{author}{Broda, K.B.},
  \bibinfo{author}{Gabbay, D.M.}, \bibinfo{year}{2012}.
\newblock \bibinfo{title}{{Neural-symbolic learning systems: foundations and
  applications}}.
\newblock \bibinfo{publisher}{Springer Science {\&} Business Media}.
\bibitem[{Garnelo et~al.(2016)Garnelo, Arulkumaran and
  Shanahan}]{garnelo2016towards}
\bibinfo{author}{Garnelo, M.}, \bibinfo{author}{Arulkumaran, K.},
  \bibinfo{author}{Shanahan, M.}, \bibinfo{year}{2016}.
\newblock \bibinfo{title}{{Towards deep symbolic reinforcement learning}}.
\newblock \bibinfo{journal}{arXiv preprint arXiv:1609.05518} .
\bibitem[{Getoor and Taskar(2007)}]{getoor2007}
\bibinfo{author}{Getoor, L.}, \bibinfo{author}{Taskar, B.},
  \bibinfo{year}{2007}.
\newblock \bibinfo{title}{{Introduction to statistical relational learning}}.
  volume~\bibinfo{volume}{1}.
\newblock \bibinfo{publisher}{MIT press Cambridge}.
\bibitem[{Goodfellow et~al.(2016)Goodfellow, Bengio, Courville and
  Bengio}]{goodfellow2016deep}
\bibinfo{author}{Goodfellow, I.}, \bibinfo{author}{Bengio, Y.},
  \bibinfo{author}{Courville, A.}, \bibinfo{author}{Bengio, Y.},
  \bibinfo{year}{2016}.
\newblock \bibinfo{title}{{Deep learning}}. volume~\bibinfo{volume}{1}.
\newblock \bibinfo{publisher}{MIT press Cambridge}.
\bibitem[{Guo et~al.(2016)Guo, Wang, Wang, Wang and Guo}]{guo2016}
\bibinfo{author}{Guo, S.}, \bibinfo{author}{Wang, Q.}, \bibinfo{author}{Wang,
  L.}, \bibinfo{author}{Wang, B.}, \bibinfo{author}{Guo, L.},
  \bibinfo{year}{2016}.
\newblock \bibinfo{title}{{Jointly embedding knowledge graphs and logical
  rules}}, in: \bibinfo{booktitle}{Proceedings of the 2016 Conference on
  Empirical Methods in Natural Language Processing}, pp.
  \bibinfo{pages}{192--202}.
\bibitem[{Gupta and Nadarajah(2004)}]{gupta2004handbook}
\bibinfo{author}{Gupta, A.K.}, \bibinfo{author}{Nadarajah, S.},
  \bibinfo{year}{2004}.
\newblock \bibinfo{title}{{Handbook of beta distribution and its
  applications}}.
\newblock \bibinfo{publisher}{CRC press}.
\bibitem[{H\'{a}jek(1998)}]{hajek1998metamathematics}
\bibinfo{author}{H\'{a}jek, P.}, \bibinfo{year}{1998}.
\newblock \bibinfo{title}{The Metamathematics of Fuzzy Logic}.
\newblock \bibinfo{publisher}{Kluwer}.
\bibitem[{Hall(1927)}]{10.2307/2331961}
\bibinfo{author}{Hall, P.}, \bibinfo{year}{1927}.
\newblock \bibinfo{title}{{The Distribution of Means for Samples of Size N
  Drawn from a Population in which the Variate Takes Values Between 0 and 1,
  All Such Values Being Equally Probable}}.
\newblock \bibinfo{journal}{Biometrika} \bibinfo{volume}{19},
  \bibinfo{pages}{240--245}.
\newblock \URLprefix \url{http://www.jstor.org/stable/2331961}.
\bibitem[{Harnad(1990)}]{harnad1990symbol}
\bibinfo{author}{Harnad, S.}, \bibinfo{year}{1990}.
\newblock \bibinfo{title}{{The symbol grounding problem}}.
\newblock \bibinfo{journal}{Physica D: Nonlinear Phenomena}
  \bibinfo{volume}{42}, \bibinfo{pages}{335--346}.
\bibitem[{Hempel(1945)}]{Hempel1945}
\bibinfo{author}{Hempel, C.G.}, \bibinfo{year}{1945}.
\newblock \bibinfo{title}{{Studies in the Logic of Confirmation (II.)}}.
\newblock \bibinfo{journal}{Mind} \bibinfo{volume}{54},
  \bibinfo{pages}{97--121}.
\newblock \URLprefix \url{http://www.jstor.org/stable/2250948},
  \DOIprefix\doi{10.2307/2250948}.
\bibitem[{Hu et~al.(2016)Hu, Ma, Liu, Hovy and Xing}]{P16-1228}
\bibinfo{author}{Hu, Z.}, \bibinfo{author}{Ma, X.}, \bibinfo{author}{Liu, Z.},
  \bibinfo{author}{Hovy, E.}, \bibinfo{author}{Xing, E.}, \bibinfo{year}{2016}.
\newblock \bibinfo{title}{{Harnessing Deep Neural Networks with Logic Rules}},
  in: \bibinfo{booktitle}{Proceedings of the 54th Annual Meeting of the
  Association for Computational Linguistics (Volume 1: Long Papers)},
  \bibinfo{publisher}{Association for Computational Linguistics}. pp.
  \bibinfo{pages}{2410--2420}.
\newblock \URLprefix \url{http://aclweb.org/anthology/P16-1228},
  \DOIprefix\doi{10.18653/v1/P16-1228}.
\bibitem[{Hwang and Masud(2012)}]{hwang2012multiple}
\bibinfo{author}{Hwang, C.L.}, \bibinfo{author}{Masud, A.S.M.},
  \bibinfo{year}{2012}.
\newblock \bibinfo{title}{{Multiple objective decision making—methods and
  applications: a state-of-the-art survey}}. volume \bibinfo{volume}{164}.
\newblock \bibinfo{publisher}{Springer Science {\&} Business Media}.
\bibitem[{Irwin(1927)}]{irwin1927frequency}
\bibinfo{author}{Irwin, J.O.}, \bibinfo{year}{1927}.
\newblock \bibinfo{title}{{On the Frequency Distribution of the Means of
  Samples from a Population Having any Law of Frequency with Finite Moments,
  with Special Reference to Pearson's Type II}}.
\newblock \bibinfo{journal}{Biometrika} , \bibinfo{pages}{225--239}.
\bibitem[{Jang(1993)}]{Jang1993a}
\bibinfo{author}{Jang, J.S.R.}, \bibinfo{year}{1993}.
\newblock \bibinfo{title}{{ANFIS: Adaptive-Network-Based Fuzzy Inference
  System}}.
\newblock \bibinfo{journal}{IEEE Transactions on Systems, Man and Cybernetics}
  \DOIprefix\doi{10.1109/21.256541}.
\bibitem[{Jang et~al.(1997)Jang, Sun and Mizutani}]{Jang1997a}
\bibinfo{author}{Jang, J.S.R.}, \bibinfo{author}{Sun, C.T.},
  \bibinfo{author}{Mizutani, E.}, \bibinfo{year}{1997}.
\newblock \bibinfo{title}{{Neuro-Fuzzy and Soft Computing: A Computational
  Approach to Learning and Machine Intelligence}}.
\bibitem[{Japkowicz and Stephen(2002)}]{japkowicz2002class}
\bibinfo{author}{Japkowicz, N.}, \bibinfo{author}{Stephen, S.},
  \bibinfo{year}{2002}.
\newblock \bibinfo{title}{{The class imbalance problem: A systematic study}}.
\newblock \bibinfo{journal}{Intelligent data analysis} \bibinfo{volume}{6},
  \bibinfo{pages}{429--449}.
\bibitem[{Jayaram and Baczynski(2008)}]{Jayaram2008}
\bibinfo{author}{Jayaram, B.}, \bibinfo{author}{Baczynski, M.},
  \bibinfo{year}{2008}.
\newblock \bibinfo{title}{{Fuzzy Implications}}. volume \bibinfo{volume}{231}.
\newblock \bibinfo{publisher}{Springer, Berlin, Heidelberg}.
\newblock \URLprefix \url{http://link.springer.com/10.1007/978-3-540-69082-5},
  \DOIprefix\doi{10.1007/978-3-540-69082-5}.
\bibitem[{Kingma and Ba(2017)}]{kingmaAdamMethodStochastic2017}
\bibinfo{author}{Kingma, D.P.}, \bibinfo{author}{Ba, J.}, \bibinfo{year}{2017}.
\newblock \bibinfo{title}{Adam: {{A Method}} for {{Stochastic Optimization}}}.
\newblock \bibinfo{journal}{arXiv:1412.6980 [cs]}
  \href{http://arxiv.org/abs/1412.6980}{{\tt arXiv:1412.6980}}.
\bibitem[{Klement et~al.(2013)Klement, Mesiar and Pap}]{klement2013triangular}
\bibinfo{author}{Klement, E.P.}, \bibinfo{author}{Mesiar, R.},
  \bibinfo{author}{Pap, E.}, \bibinfo{year}{2013}.
\newblock \bibinfo{title}{{Triangular norms}}. volume~\bibinfo{volume}{8}.
\newblock \bibinfo{publisher}{Springer Science {\&} Business Media}.
\bibitem[{Klir and Yuan(1995)}]{klir1995fuzzy}
\bibinfo{author}{Klir, G.}, \bibinfo{author}{Yuan, B.}, \bibinfo{year}{1995}.
\newblock \bibinfo{title}{{Fuzzy sets and fuzzy logic}}.
  volume~\bibinfo{volume}{4}.
\newblock \bibinfo{publisher}{Prentice hall New Jersey}.
\bibitem[{van Krieken et~al.(2019)van Krieken, Acar and van
  Harmelen}]{vankrieken2019ravens}
\bibinfo{author}{van Krieken, E.}, \bibinfo{author}{Acar, E.},
  \bibinfo{author}{van Harmelen, F.}, \bibinfo{year}{2019}.
\newblock \bibinfo{title}{{Semi-Supervised Learning using Differentiable
  Reasoning}}.
\newblock \bibinfo{journal}{IFCoLog Journal of Logic and its Applications}
  \bibinfo{volume}{6}, \bibinfo{pages}{633--653}.
\bibitem[{LeCun and Cortes(2010)}]{lecun-mnisthandwrittendigit-2010}
\bibinfo{author}{LeCun, Y.}, \bibinfo{author}{Cortes, C.},
  \bibinfo{year}{2010}.
\newblock \bibinfo{title}{{MNIST handwritten digit database}} \URLprefix
  \url{http://yann.lecun.com/exdb/mnist/}.
\bibitem[{Lin and Lee(1991)}]{Lin1991a}
\bibinfo{author}{Lin, C.T.}, \bibinfo{author}{Lee, G.C.}, \bibinfo{year}{1991}.
\newblock \bibinfo{title}{{Neural-Network-Based Fuzzy Logic Control and
  Decision System}}.
\newblock \bibinfo{journal}{IEEE Transactions on Computers}
  \DOIprefix\doi{10.1109/12.106218}.
\bibitem[{Liu and Kerre(1998)}]{Liu1998}
\bibinfo{author}{Liu, Y.}, \bibinfo{author}{Kerre, E.}, \bibinfo{year}{1998}.
\newblock \bibinfo{title}{{An overview of fuzzy quantifiers.(I).
  Interpretations}}.
\newblock \bibinfo{journal}{Fuzzy Sets and Systems} \bibinfo{volume}{95},
  \bibinfo{pages}{1--21}.
\bibitem[{Manhaeve et~al.(2018)Manhaeve, Duman{\v{c}}i{\'{c}}, Kimmig,
  Demeester and De~Raedt}]{DBLP:conf/nips/2018}
\bibinfo{author}{Manhaeve, R.}, \bibinfo{author}{Duman{\v{c}}i{\'{c}}, S.},
  \bibinfo{author}{Kimmig, A.}, \bibinfo{author}{Demeester, T.},
  \bibinfo{author}{De~Raedt, L.}, \bibinfo{year}{2018}.
\newblock \bibinfo{title}{{DeepProbLog: Neural Probabilistic Logic
  Programming}}, in: \bibinfo{editor}{Bengio, S.}, \bibinfo{editor}{Wallach,
  H.M.}, \bibinfo{editor}{Larochelle, H.}, \bibinfo{editor}{Grauman, K.},
  \bibinfo{editor}{Cesa-Bianchi, N.}, \bibinfo{editor}{Garnett, R.} (Eds.),
  \bibinfo{booktitle}{Advances in Neural Information Processing Systems 31:
  Annual Conference on Neural Information Processing Systems 2018, NeurIPS
  2018, 3-8 December 2018, Montr{\'{e}}al, Canada}.
\newblock \URLprefix \url{http://arxiv.org/abs/1805.10872
  http://papers.nips.cc/book/advances-in-neural-information-processing-systems-31-2018}.
\bibitem[{Marcus(2018)}]{marcus2018deep}
\bibinfo{author}{Marcus, G.}, \bibinfo{year}{2018}.
\newblock \bibinfo{title}{{Deep learning: A critical appraisal}}.
\newblock \bibinfo{journal}{arXiv preprint arXiv:1801.00631} .
\bibitem[{Marra et~al.(2018)Marra, Giannini, Diligenti and Gori}]{marra2018}
\bibinfo{author}{Marra, G.}, \bibinfo{author}{Giannini, F.},
  \bibinfo{author}{Diligenti, M.}, \bibinfo{author}{Gori, M.},
  \bibinfo{year}{2018}.
\newblock \bibinfo{title}{{Constraint-Based Visual Generation}}.
\newblock \bibinfo{journal}{arXiv preprint arXiv:1807.09202} .
\bibitem[{Marra et~al.(2019a)Marra, Giannini, Diligenti and Gori}]{Marra2019}
\bibinfo{author}{Marra, G.}, \bibinfo{author}{Giannini, F.},
  \bibinfo{author}{Diligenti, M.}, \bibinfo{author}{Gori, M.},
  \bibinfo{year}{2019}a.
\newblock \bibinfo{title}{{Integrating Learning and Reasoning with Deep Logic
  Models}}.
\newblock \bibinfo{journal}{arXiv preprint arXiv:1901.04195} ,
  \bibinfo{pages}{1--17}\URLprefix
  \url{https://arxiv.org/pdf/1901.04195v1.pdf}.
\bibitem[{Marra et~al.(2019b)Marra, Giannini, Diligenti, Maggini and
  Gori}]{marra2019learning}
\bibinfo{author}{Marra, G.}, \bibinfo{author}{Giannini, F.},
  \bibinfo{author}{Diligenti, M.}, \bibinfo{author}{Maggini, M.},
  \bibinfo{author}{Gori, M.}, \bibinfo{year}{2019}b.
\newblock \bibinfo{title}{{Learning and T-Norms Theory}}.
\newblock \bibinfo{journal}{arXiv preprint arXiv:1907.11468} .
\bibitem[{Mayo(2003)}]{Mayo2003}
\bibinfo{author}{Mayo, M.}, \bibinfo{year}{2003}.
\newblock \bibinfo{title}{{Symbol Grounding and its Implications for Artificial
  Intelligence BT}}.
\newblock \bibinfo{journal}{Twenty-Sixth Australasian Computer Science
  Conference (ACSC2003)} \bibinfo{volume}{16}, \bibinfo{pages}{55--60}.
\newblock \URLprefix \url{http://crpit.com/confpapers/CRPITV16Mayo.pdf}.
\bibitem[{Minervini et~al.(2017)Minervini, Demeester, Rockt{\"{a}}schel and
  Riedel}]{minervini2017}
\bibinfo{author}{Minervini, P.}, \bibinfo{author}{Demeester, T.},
  \bibinfo{author}{Rockt{\"{a}}schel, T.}, \bibinfo{author}{Riedel, S.},
  \bibinfo{year}{2017}.
\newblock \bibinfo{title}{{Adversarial sets for regularising neural link
  predictors}}, in: \bibinfo{booktitle}{Uncertainty in Artificial
  Intelligence-Proceedings of the 33rd Conference, UAI 2017}.
\bibitem[{Minervini and Riedel(2018)}]{minervini2018}
\bibinfo{author}{Minervini, P.}, \bibinfo{author}{Riedel, S.},
  \bibinfo{year}{2018}.
\newblock \bibinfo{title}{{Adversarially Regularising Neural NLI Models to
  Integrate Logical Background Knowledge}}, in: \bibinfo{booktitle}{Proceedings
  of the 22nd Conference on Computational Natural Language Learning}, pp.
  \bibinfo{pages}{65--74}.
\bibitem[{Muggleton and de~Raedt(1994)}]{MUGGLETON1994629}
\bibinfo{author}{Muggleton, S.}, \bibinfo{author}{de~Raedt, L.},
  \bibinfo{year}{1994}.
\newblock \bibinfo{title}{{Inductive Logic Programming: Theory and methods}}.
\newblock \bibinfo{journal}{The Journal of Logic Programming}
  \bibinfo{volume}{19-20}, \bibinfo{pages}{629--679}.
\newblock \URLprefix
  \url{http://www.sciencedirect.com/science/article/pii/0743106694900353},
  \DOIprefix\doi{https://doi.org/10.1016/0743-1066(94)90035-3}.
\bibitem[{Murphy et~al.(2013)Murphy, Weiss and Jordan}]{murphy1999}
\bibinfo{author}{Murphy, K.P.}, \bibinfo{author}{Weiss, Y.},
  \bibinfo{author}{Jordan, M.I.}, \bibinfo{year}{2013}.
\newblock \bibinfo{title}{{Loopy Belief Propagation for Approximate Inference:
  An Empirical Study}}, in: \bibinfo{booktitle}{Proceedings of the Fifteenth
  conference on Uncertainty in artificial intelligence},
  \bibinfo{organization}{Morgan Kaufmann Publishers Inc.}. pp.
  \bibinfo{pages}{467--475}.
\newblock \URLprefix \url{http://arxiv.org/abs/1301.6725},
  \DOIprefix\doi{10.1.1.32.5538}.
\bibitem[{Nov{\'{a}}k et~al.(1999)Nov{\'{a}}k, Perfilieva and
  Mo{\v{c}}ko{\v{r}}}]{Nov_k_1999}
\bibinfo{author}{Nov{\'{a}}k, V.}, \bibinfo{author}{Perfilieva, I.},
  \bibinfo{author}{Mo{\v{c}}ko{\v{r}}, J.}, \bibinfo{year}{1999}.
\newblock \bibinfo{title}{Mathematical Principles of Fuzzy Logic}.
\newblock \bibinfo{publisher}{Springer {US}}.
\newblock \URLprefix \url{https://doi.org/10.1007%2F978-1-4615-5217-8},
  \DOIprefix\doi{10.1007/978-1-4615-5217-8}.
\bibitem[{P{\'{a}}ll~J{\'{o}}nsson(2018)}]{PallJonsson2018}
\bibinfo{author}{P{\'{a}}ll~J{\'{o}}nsson, H.}, \bibinfo{year}{2018}.
\newblock \bibinfo{title}{{Real Logic and Logical Tensor Networks}}.
\newblock \bibinfo{journal}{University of Amsterdam, MSc thesis} ,
  \bibinfo{pages}{45}.
\bibitem[{Pearl(1988)}]{pearl1988}
\bibinfo{author}{Pearl, J.}, \bibinfo{year}{1988}.
\newblock \bibinfo{title}{{Probabilistic reasoning in intelligent systems:
  networks of plausible inference}}.
\newblock \bibinfo{publisher}{Elsevier}.
\bibitem[{Pearl(2018)}]{pearl2018theoretical}
\bibinfo{author}{Pearl, J.}, \bibinfo{year}{2018}.
\newblock \bibinfo{title}{{Theoretical Impediments to Machine Learning With
  Seven Sparks from the Causal Revolution}}, in:
  \bibinfo{booktitle}{Proceedings of the Eleventh ACM International Conference
  on Web Search and Data Mining}, \bibinfo{organization}{ACM}.
  p.~\bibinfo{pages}{3}.
\bibitem[{Pham and Turkkan(1994)}]{Pham1994a}
\bibinfo{author}{Pham, T.G.}, \bibinfo{author}{Turkkan, N.},
  \bibinfo{year}{1994}.
\newblock \bibinfo{title}{{Reliability of a Standby System with
  Beta-Distributed Component Lives}}.
\newblock \bibinfo{journal}{IEEE Transactions on Reliability}
  \DOIprefix\doi{10.1109/24.285114}.
\bibitem[{Radford et~al.(2019)Radford, Wu, Child, Luan, Amodei and
  Sutskever}]{radford2019language}
\bibinfo{author}{Radford, A.}, \bibinfo{author}{Wu, J.},
  \bibinfo{author}{Child, R.}, \bibinfo{author}{Luan, D.},
  \bibinfo{author}{Amodei, D.}, \bibinfo{author}{Sutskever, I.},
  \bibinfo{year}{2019}.
\newblock \bibinfo{title}{{Language Models are Unsupervised Multitask
  Learners}} .
\bibitem[{Rasmus et~al.(2015)Rasmus, Berglund, Honkala, Valpola and
  Raiko}]{rasmus2015semi}
\bibinfo{author}{Rasmus, A.}, \bibinfo{author}{Berglund, M.},
  \bibinfo{author}{Honkala, M.}, \bibinfo{author}{Valpola, H.},
  \bibinfo{author}{Raiko, T.}, \bibinfo{year}{2015}.
\newblock \bibinfo{title}{{Semi-supervised learning with ladder networks}}, in:
  \bibinfo{booktitle}{Advances in Neural Information Processing Systems}, pp.
  \bibinfo{pages}{3546--3554}.
\bibitem[{Riegel et~al.(2020)Riegel, Gray, Luus, Khan, Makondo, Akhalwaya,
  Qian, Fagin, Barahona, Sharma, Ikbal, Karanam, Neelam, Likhyani and
  Srivastava}]{riegel2020logical}
\bibinfo{author}{Riegel, R.}, \bibinfo{author}{Gray, A.},
  \bibinfo{author}{Luus, F.}, \bibinfo{author}{Khan, N.},
  \bibinfo{author}{Makondo, N.}, \bibinfo{author}{Akhalwaya, I.Y.},
  \bibinfo{author}{Qian, H.}, \bibinfo{author}{Fagin, R.},
  \bibinfo{author}{Barahona, F.}, \bibinfo{author}{Sharma, U.},
  \bibinfo{author}{Ikbal, S.}, \bibinfo{author}{Karanam, H.},
  \bibinfo{author}{Neelam, S.}, \bibinfo{author}{Likhyani, A.},
  \bibinfo{author}{Srivastava, S.}, \bibinfo{year}{2020}.
\newblock \bibinfo{title}{Logical neural networks}.
\newblock \href{http://arxiv.org/abs/2006.13155}{{\tt arXiv:2006.13155}}.
\bibitem[{Rockt{\"{a}}schel and Riedel(2017)}]{Rocktaschel2017}
\bibinfo{author}{Rockt{\"{a}}schel, T.}, \bibinfo{author}{Riedel, S.},
  \bibinfo{year}{2017}.
\newblock \bibinfo{title}{{End-to-End Differentiable Proving}} \URLprefix
  \url{https://arxiv.org/pdf/1705.11040.pdf http://arxiv.org/abs/1705.11040}.
\bibitem[{Rockt{\"{a}}schel et~al.(2015)Rockt{\"{a}}schel, Singh and
  Riedel}]{Rocktaschel2015b}
\bibinfo{author}{Rockt{\"{a}}schel, T.}, \bibinfo{author}{Singh, S.},
  \bibinfo{author}{Riedel, S.}, \bibinfo{year}{2015}.
\newblock \bibinfo{title}{{Injecting Logical Background Knowledge into
  Embeddings for Relation Extraction}}, in: \bibinfo{booktitle}{Proceedings of
  the 2015 Conference of the North American Chapter of the Association for
  Computational Linguistics: Human Language Technologies}, pp.
  \bibinfo{pages}{1119--1129}.
\newblock \URLprefix
  \url{https://rockt.github.io/pdf/rocktaschel2015injecting.pdf
  http://www.aclweb.org/anthology/N15-1118
  http://aclweb.org/anthology/N15-1118}, \DOIprefix\doi{10.3115/v1/N15-1118}.
\bibitem[{Sen et~al.(2008)Sen, Namata, Bilgic, Getoor, Galligher and
  Eliassi-Rad}]{sen2008}
\bibinfo{author}{Sen, P.}, \bibinfo{author}{Namata, G.},
  \bibinfo{author}{Bilgic, M.}, \bibinfo{author}{Getoor, L.},
  \bibinfo{author}{Galligher, B.}, \bibinfo{author}{Eliassi-Rad, T.},
  \bibinfo{year}{2008}.
\newblock \bibinfo{title}{{Collective classification in network data}}.
\newblock \bibinfo{journal}{AI magazine} \bibinfo{volume}{29},
  \bibinfo{pages}{93}.
\bibitem[{Serafini and Garcez(2016)}]{Serafini2016}
\bibinfo{author}{Serafini, L.}, \bibinfo{author}{Garcez, A.D.},
  \bibinfo{year}{2016}.
\newblock \bibinfo{title}{{Logic tensor networks: Deep learning and logical
  reasoning from data and knowledge}}.
\newblock \bibinfo{journal}{CEUR Workshop Proceedings} \bibinfo{volume}{1768}.
\bibitem[{Silver et~al.(2017)Silver, Schrittwieser, Simonyan, Antonoglou,
  Huang, Guez, Hubert, Baker, Lai, Bolton and {others}}]{silver2017mastering}
\bibinfo{author}{Silver, D.}, \bibinfo{author}{Schrittwieser, J.},
  \bibinfo{author}{Simonyan, K.}, \bibinfo{author}{Antonoglou, I.},
  \bibinfo{author}{Huang, A.}, \bibinfo{author}{Guez, A.},
  \bibinfo{author}{Hubert, T.}, \bibinfo{author}{Baker, L.},
  \bibinfo{author}{Lai, M.}, \bibinfo{author}{Bolton, A.},
  \bibinfo{author}{{others}}, \bibinfo{year}{2017}.
\newblock \bibinfo{title}{{Mastering the game of Go without human knowledge}}.
\newblock \bibinfo{journal}{Nature} \bibinfo{volume}{550},
  \bibinfo{pages}{354}.
\bibitem[{Socher et~al.(2013)Socher, Chen, Manning and Ng}]{Socher2013}
\bibinfo{author}{Socher, R.}, \bibinfo{author}{Chen, D.},
  \bibinfo{author}{Manning, C.D.}, \bibinfo{author}{Ng, A.Y.},
  \bibinfo{year}{2013}.
\newblock \bibinfo{title}{{Reasoning With Neural Tensor Networks for Knowledge
  Base Completion}}.
\newblock \bibinfo{journal}{Proc.{\textbackslash} of NIPS'13} ,
  \bibinfo{pages}{1--10}\DOIprefix\doi{10.1023/A:1004554217069}.
\bibitem[{{\v{S}}ourek et~al.(2015){\v{S}}ourek, Aschenbrenner,
  {\v{Z}}elezn{\'{y}} and Ku{\v{z}}elka}]{Sourek2015LiftedNetworks}
\bibinfo{author}{{\v{S}}ourek, G.}, \bibinfo{author}{Aschenbrenner, V.},
  \bibinfo{author}{{\v{Z}}elezn{\'{y}}, F.}, \bibinfo{author}{Ku{\v{z}}elka,
  O.}, \bibinfo{year}{2015}.
\newblock \bibinfo{title}{{Lifted relational neural networks}}, in:
  \bibinfo{booktitle}{CEUR Workshop Proceedings}.
\newblock \URLprefix \url{https://arxiv.org/pdf/1508.05128.pdf}.
\bibitem[{{\v{s}}ourek et~al.(2018){\v{s}}ourek, Aschenbrenner,
  {\v{Z}}elezn{\'{y}}, Schockaert and Ku{\v{z}}elka}]{Sourek2018}
\bibinfo{author}{{\v{s}}ourek, G.}, \bibinfo{author}{Aschenbrenner, V.},
  \bibinfo{author}{{\v{Z}}elezn{\'{y}}, F.}, \bibinfo{author}{Schockaert, S.},
  \bibinfo{author}{Ku{\v{z}}elka, O.}, \bibinfo{year}{2018}.
\newblock \bibinfo{title}{{Lifted relational neural networks: Efficient
  learning of latent relational structures}}.
\newblock \bibinfo{journal}{Journal of Artificial Intelligence Research}
  \bibinfo{volume}{62}, \bibinfo{pages}{69--100}.
\newblock \DOIprefix\doi{10.1613/jair.1.11203}.
\bibitem[{Vranas(2004)}]{Vranas2004}
\bibinfo{author}{Vranas, P.B.}, \bibinfo{year}{2004}.
\newblock \bibinfo{title}{{Hempel's raven paradox: A lacuna in the standard
  Bayesian solution}}.
\newblock \bibinfo{journal}{British Journal for the Philosophy of Science}
  \bibinfo{volume}{55}, \bibinfo{pages}{545--560}.
\newblock \DOIprefix\doi{10.1093/bjps/55.3.545}.
\bibitem[{Weisberg(1971)}]{weisberg1971}
\bibinfo{author}{Weisberg, H.}, \bibinfo{year}{1971}.
\newblock \bibinfo{title}{{The Distribution of Linear Combinations of Order
  Statistics from the Uniform Distribution}}.
\newblock \bibinfo{journal}{Ann. Math. Statist.} \bibinfo{volume}{42},
  \bibinfo{pages}{704--709}.
\newblock \URLprefix \url{https://doi.org/10.1214/aoms/1177693419},
  \DOIprefix\doi{10.1214/aoms/1177693419}.
\bibitem[{Xu et~al.(2018)Xu, Zhang, Friedman, Liang and den
  Broeck}]{pmlr-v80-xu18h}
\bibinfo{author}{Xu, J.}, \bibinfo{author}{Zhang, Z.},
  \bibinfo{author}{Friedman, T.}, \bibinfo{author}{Liang, Y.},
  \bibinfo{author}{den Broeck, G.}, \bibinfo{year}{2018}.
\newblock \bibinfo{title}{{A Semantic Loss Function for Deep Learning with
  Symbolic Knowledge}}, in: \bibinfo{editor}{Dy, J.}, \bibinfo{editor}{Krause,
  A.} (Eds.), \bibinfo{booktitle}{Proceedings of the 35th International
  Conference on Machine Learning}, \bibinfo{publisher}{PMLR},
  \bibinfo{address}{Stockholmsm{\"{a}}ssan, Stockholm Sweden}. pp.
  \bibinfo{pages}{5502--5511}.
\newblock \URLprefix \url{http://proceedings.mlr.press/v80/xu18h.html}.
\bibitem[{Yager(1980)}]{Yager1980}
\bibinfo{author}{Yager, R.R.}, \bibinfo{year}{1980}.
\newblock \bibinfo{title}{{On a general class of fuzzy connectives}}.
\newblock \bibinfo{journal}{Fuzzy Sets and Systems} \bibinfo{volume}{4},
  \bibinfo{pages}{235--242}.
\bibitem[{Zhou(2017)}]{zhou2017}
\bibinfo{author}{Zhou, Z.H.}, \bibinfo{year}{2017}.
\newblock \bibinfo{title}{{A brief introduction to weakly supervised
  learning}}.
\newblock \bibinfo{journal}{National Science Review} \bibinfo{volume}{5},
  \bibinfo{pages}{44--53}.

\end{thebibliography}

\appendix
\section{Background on Fuzzy Logic operators}
\label{sec:background-operators}

In this section, we will introduce the semantics of the fuzzy operators $\otimes$ (t-norm), $\oplus$ (t-conorm) and $\neg$ (negation) that are used to connect truth values of fuzzy predicates, and the semantics of the $\forall$ quantifier. We follow \pcite{Jayaram2008} in this section and refer to it for proofs and additional results.



\subsection{Fuzzy Negation}
\label{sec:negation}
The functions that are used to compute the negation of a truth value of a formula are called \textit{fuzzy negations}.
 \begin{deff}
 A \textit{fuzzy negation} is a decreasing function $N: [0, 1]\rightarrow [0, 1]$ so that $N(1) = 0$ and for all $x$, $N(N(x)) \geq x$ \pcite{cignoliClassLeftcontinuousTnorms2002}. $N$ is called \textit{strict} if it is strictly decreasing and continuous, and \textit{strong} if for all $a\in [0,1]$, $N(N(a)) = a$.
\end{deff}
A consequence of these conditions is that $N(0)=1$. Throughout the paper we also use $N$ to refer to the classical negation $N(a) = 1-a$.

\subsection{Triangular Norms}
 The functions that are used to compute the conjunction of two truth values are called \textit{t-norms}. For a rigorous overview, see \tcite{klement2013triangular}.

\begin{deff}
\label{deff:tnorm}
A \textit{t-norm} (triangular norm) is a function $T: [0,1]^2\rightarrow [0, 1]$ that is commutative and associative, and
\begin{enumerate}
    \item \textit{Monotonicity}: For all $a\in [0, 1]$, $T(a, \cdot)$ is increasing and
    \item \textit{Neutrality}: For all $a\in [0,1]$, $T(1, a) = a$.
\end{enumerate}
\end{deff}
The phrase `$T(a, \cdot)$ is increasing' means that whenever $0\leq b_1\leq b_2\leq 1$, then $T(a, b_1) \leq T(a, b_2)$.

\begin{deff}
\label{deff:tnormprops}
A t-norm $T$ can have the following properties:
\begin{enumerate}[a)]
    \item \textit{Continuity}: A continuous t-norm is continuous in both arguments.
    \item \textit{Left-continuity}: A left-continuous t-norm is left-continuous in both arguments. That is, for all $a, b\in [0,1]$, $\lim_{x\rightarrow a^-} T(x, b) = T(a, b)$ (the limit of $T(x, b)$ as $x$ increases and approaches $a$ is $a$).
    \item \textit{Idempotency}: An idempotent t-norm has the property that for all  $a\in [0,1]$, $T(a, a) = a$.
    \item \textit{Strict-monotony}: A strictly monotone t-norm  has the property that for all $a\in (0, 1]$, $T(a, \cdot)$ is strictly increasing.
    \item \textit{Strict}: A strict t-norm is continuous and strictly monotone.
\end{enumerate}
\end{deff}


Table $\ref{tab:tnorms}$ shows the four basic t-norms and two other t-norms of interest alongside their properties. 

\subsection{Triangular Conorms}
\label{appendix:t-norms}
The functions that are used to compute the disjunction of two truth values are called \textit{t-conorms} or \textit{s-norms}.
\begin{deff}
\label{deff:snorm}
A \textit{t-conorm} (triangular conorm, also known as s-norm) is a function $S: [0,1]^2\rightarrow [0, 1]$ that is commutative and associative, and
\begin{enumerate} 
    \item \textit{Monotonicity:} For all $a\in [0, 1]$, $S(a, \cdot)$ is increasing and
    \item \textit{Neutrality}: For all $a\in [0,1]$, $S(0, a) = a$.
\end{enumerate}
\end{deff}
T-conorms are obtained from t-norms using De Morgan's laws from classical logic, i.e. $p\vee q = \neg(\neg p \wedge \neg q)$. Therefore, if $T$ is a t-norm and $N_C$ the strong negation, $T$'s \textit{$N_C$-dual} $S$ is calculated using
\begin{equation}
\label{eq:tconorm}
    S(a, b) =  1 - T(1 - a, 1 - b)
\end{equation}

Table \ref{tab:snorms} shows several common t-conorms derived using Equation \ref{eq:tconorm} and the t-norms from Table \ref{tab:tnorms}, alongside the same optional properties as those for t-norms in Definition \ref{deff:tnormprops}. 

\subsection{Aggregation operators}
\label{appendix:aggregation}
The functions that are used to compute quantifiers like $\forall$ and $\exists$ are aggregation operators \pcite{Liu1998}. 

\begin{deff}
\label{deff:aggr}
An \textit{aggregation operator} \pcite{calvoAggregationOperatorsProperties2002} is a function $A: \bigcup_{n\in \mathbb{N}} [0, 1]^n\rightarrow [0, 1]$ that is non-decreasing with respect to each argument, and for which $A(0, ..., 0)=0$ and $A(1, ..., 1) = 1$. 
\end{deff}
Aggregation operators are \textit{variadic functions} which are functions that are defined for any sequence of arguments. 
For this reason we will often use the notation $\aggregate_{i=1}^n x_i:= A(x_1, ..., x_n)$. 
Table \ref{tab:aggregation} shows some common aggregation operators that we will talk about.
Furthermore, we will only consider \emph{symmetric} aggregation operators, that are invariant to permutation of the sequence. 

The $\forall$ quantifier is interpreted as the conjunction over all arguments $x$. Therefore, we can extend a t-norm $T$ from 2-dimensional inputs to $n$-dimensional inputs as they are commutative and associative \pcite{klement2013triangular}:
\begin{equation}
\label{eq:aggtnorm}
\begin{aligned}
    A_T() &= 0\\
    A_T(x_1, x_2, ..., x_n) &= T(x_1, A_T(x_2, ..., x_n))
\end{aligned}
\end{equation}
These operators are a straightforward choice for modelling the $\forall$ quantifier, as they can be seen as a series of conjunctions. All operators constructed in this way are \textit{symmetric} aggregation operators, for which the output value is the same for every ordering of its arguments. This generalizes commutativity. 

We can do the same for a t-conorm $S$ to model the $\exists$ quantifier:
\begin{equation}
\begin{aligned}
    E_S()&= 0\\
    E_S(x_1, x_2, ..., x_n) &= S(x_1, A_S(x_2, ..., x_n))
\end{aligned}
\end{equation}

\subsection{Fuzzy Implications}
\label{sec:fuzz_imp}
The functions that are used to compute the truth value of $p\rightarrow q$ are called fuzzy implications. $p$ is called the \textit{antecedent} and $q$ the \textit{consequent} of the implication. We follow \tcite{Jayaram2008} and refer to it for details and proofs.
\begin{deff}
\label{def:implication}
A \textit{fuzzy implication} is a function $I: [0, 1]^2\rightarrow [0, 1]$ so that for all $a, c\in [0, 1]$, $I(\cdot, c)$ is decreasing, $I(a, \cdot)$ is increasing and for which $I(0, 0) = 1$,  $I(1, 1) = 1$ and $I(1, 0) = 0$.
\end{deff}
From this definition follows that $I(0, 1) = 1$. 
\begin{deff}
\label{deff:implications_optional}
Let $N$ be a fuzzy negation. A fuzzy implication $I$ satisfies
\begin{enumerate}[a)]
    \item \textit{left-neutrality (LN)} if for all $c\in [0,1]$, $I(1, c) = c$;
    \item the \textit{exchange principle (EP)} if for all $ a,b,c\in[0,1]$,  $I(a, I(b, c)) = I(b, I(a, c))$;
    \item the \textit{identity principle (IP)} if for all $a\in[0,1]$, $I(a, a) = 1$;
    \item \textit{$N$-contrapositive symmetry (CP)} if for all $a, c\in [0,1]$, $I(a, c)=I(N(c), N(a))$;
    \item \textit{$N$-left-contrapositive symmetry (L-CP)} if for all $a, c\in [0,1]$, $I(N(a), c) = I(N(c), a)$;
    \item \textit{$N$-right-contrapositive symmetry (R-CP)} if for all $a,c\in[0,1]$, $I(a, N(c)) = I(c, N(a))$.
\end{enumerate}
\end{deff}
All these statements generalize a law from classical logic. \textit{Left neutrality} generalizes  $(1\rightarrow p) \equiv p$, the \textit{exchange principle} generalizes $p\rightarrow(q\rightarrow r) \equiv q\rightarrow(p\rightarrow r)$, and the \textit{identity principle} generalizes that $p\rightarrow p$ is a tautology. Furthermore, $N$\textit{-contrapositive symmetry} generalizes $p\rightarrow q \equiv \neg q \rightarrow \neg p$, $N$\textit{-left-contrapositive symmetry} generalizes $\neg p \rightarrow q \equiv \neg q \rightarrow p$ and $N$\textit{-right-contrapositive symmetry} generalizes $p\rightarrow \neg q \equiv q \rightarrow \neg p$. 

\subsubsection{S-Implications}
\label{appendix:s-implications}
In classical logic, the (material) implication is defined as follows:
\begin{equation*}
    p\rightarrow q = \neg p \vee q
\end{equation*}
Using this definition, we can use a t-conorm $S$ and a fuzzy negation $N$ to construct a fuzzy implication.
\begin{deff}
Let $S$ be a t-conorm and $N$ a fuzzy negation. The function $I_{S, N}: [0, 1]^2\rightarrow[0,1]$ is called an \textit{(S, N)-implication} and is defined for all $a, c\in [0, 1]$ as
\begin{equation}
\label{eq:s-impl}
    I_{S, N}(a, c) =  S(N(a), c).
\end{equation}
If N is a strong fuzzy negation, then $I_{S, N}$ is called an \textit{S-implication} (or strong implication).
\end{deff}
As we only consider the strong negation $N_C$, we omit the $N$ and use $I_S$ to refer to $I_{S, N_C}$

All S-implications $I_{S}$ are fuzzy implications and satisfy LN, EP and R-CP. Additionally, if the negation  $N$ is strong, it satisfies CP and if, in addition, it is strict, it also satisfies L-CP. 
In Table \ref{tab:simplications} we show several S-implications that use the strong fuzzy negation $N_C$ and the common t-conorms (Table \ref{tab:snorms}). Note that S-implications are rotations of the t-conorms.

\subsubsection{R-Implications}
\label{appendix:r-implications}
R-implications are another way of constructing implication operators. They are the standard choice in t-norm fuzzy logics. 



\begin{deff}
\label{deff:r-implication}
Let $T$ be a t-norm. The function $I_T: [0,1]^2\rightarrow [0, 1]$ is called an \textit{R-implication} and defined as
\begin{equation}
\label{eq:r-implication}
    I_T(a, c) = \sup\{b\in [0, 1]|T(a, b) \leq c\}
\end{equation}
\end{deff}
The \textit{supremum} of a set $A$, denoted $\sup\{A\}$, is the lowest upper bound of $A$. All R-implications are fuzzy implications, and all satisfy LN, IP and EP. $T$ is a left-continuous t-norm if and only if the supremum can be replaced with the maximum function.
Note that if $a\leq c$ then $I_T(a, c) = 1$. We can see this by looking at Equation \ref{eq:r-implication}. The largest value for $b$ possible is 1, since then, using the \textit{neutrality} property of t-norms, $T(a, 1) = a\leq c$.

Table \ref{tab:rimplications} shows the R-implications created from the common T-norms. Note that $I_{LK}$ and $I_{FD}$ appear in both tables: They are both S-implications and R-implications.

\section{Implementation of \dfuzz}
\label{sec:implementation}
\begin{algorithm}
    \caption{Computation of the \dfuzz loss. First it computes the fuzzy Herbrand interpretation $g$ given the current embedded interpretation $\interpretation$. This performs a forward pass through the neural networks that are used to interpret the predicates. Then it computes the valuation of each formula $\varphi$ in the knowledge base $\corpus$, implementing Equations \ref{eq:rlpred}-\ref{eq:rlaggr}.  }
    \label{alg:real_logic}
    \begin{algorithmic}[1] 
        \Function{$e$}{$\varphi, g, \constants, \mu$} \Comment{The valuation function computes the Fuzzy truth value of $\varphi$.}
            \If{$\varphi=\pred{P}(x_1, ..., x_m)$} 
                \State \textbf{return} $g[\pred{P}, (\mu(x_1), ..., \mu(x_m)]$ \Comment{Find the truth value of a ground atom using the dictionary $g$.}
            \ElsIf{$\varphi=\neg\phi$}
                \State \textbf{return} $N(e(\phi, g, \constants, \mu))$
            \ElsIf{$\varphi=\phi\otimes\psi$}
                \State \textbf{return} $T(e(\phi, g, \constants, \mu), e(\psi, g, \constants, \mu))$
            \ElsIf{$\varphi=\phi\oplus\psi$}
                \State \textbf{return} $S(e(\phi, g, \constants, \mu), e(\psi, g, \constants, \mu))$
            \ElsIf{$\varphi=\phi\rightarrow\psi$}
                \State \textbf{return} $I(e(\phi, g, \constants, \mu), e(\psi, g, \constants, \mu))$
            \ElsIf{$\varphi=\forall x\ \phi$} \Comment{Apply the universal aggregation operator.}
                \label{alg:quantifier}
                \State \textbf{return} $\aggregate_{o\in\constants}e(\phi, g, \constants, \mu\cup\{(x,o)\})$ \Comment{Each assignment can be seen as an instance of $\varphi$.}
            \ElsIf{$\varphi=\exists x\ \phi$} 
                \label{alg:quantifier}
                \State \textbf{return} $\Eaggregate_{o\in\constants}e(\phi, g, \constants, \mu\cup\{(x, o)\})$ 
            \EndIf
        \EndFunction
        \State
        \Procedure{\dfl}{$\interpretation, \predicates, \corpus, \objects, N, T, S, I, A, E$} \Comment{Computes the \dfuzz loss.}
            \State $\constants\gets o_1, ..., o_b \text{ sampled from } \objects$ \Comment{Sample $\batch$ constants to use this pass.}
            \label{alg:sample}
            \State $g\gets dict()$ \Comment{Collects truth values for ground atoms.}
            \For{$\pred{P}\in \predicates$}
                \For{$o_1, ..., o_{\alpha(\pred{P})} \in \constants$}
                    \State $g[\pred{P}, (o_1, ..., o_{\alpha(\pred{P})})]\gets \interpretation(\pred{P})(o_1, ..., o_{\alpha(\pred{P})})$ \Comment{Calculate the truth values of the ground atoms.}
                \EndFor
            \EndFor
            \label{alg:satisfaction}
            \State \textbf{return} $\aggregate_{\varphi\in\corpus} w_{\varphi}\cdot\val(\varphi, g, \constants, \emptyset)$ \Comment{Calculate valuation of the formulas $\varphi$. Start with an empty variable assignment. This implements Equation \ref{eq:lossrl}.}
        \EndProcedure
    \end{algorithmic}
\end{algorithm}


The computation of the satisfaction is shown in pseudocode form in Algorithm \ref{alg:real_logic}. By first computing the dictionary $g$ that contains truth values for all ground atoms,\footnote{The dictionary $g$ could be seen as a `fuzzy Herbrand interpretation', in that it assigns a truth value to all ground atoms.} we can reduce the amount of forward passes through the computations of the truth values of the ground atoms that are required to compute the satisfaction.

This algorithm can fairly easily be parallelized for efficient computation on a GPU by noting that the individual terms that are aggregated over in line \ref{alg:quantifier} (the different \textit{instances} of the universal quantifier) are not dependent on each other. By noting that formulas are in prenex normal form, we can set up the dictionary $g$ using tensor operations so that the recursion has to be done only once for each formula. This can be done by applying the fuzzy operators elementwise over vectors of truth values instead of a single truth value, where each element of the vector represents a variable assignment. 

The complexity of this computation then is $O(|\corpus| \cdot P\cdot \batch^{d})$, where $\corpus$ is the set of formulas, $P$ is the amount of predicates used in each formula and $d$ is the maximum depth of nesting of universal quantifiers in the formulas in $\corpus$ (known as the \textit{quantifier rank}). This is exponential in the amount of quantifiers, as every object from the constants $\constants$ has to be iterated over in line \ref{alg:quantifier}, although as mentioned earlier this can be mitigated somewhat using efficient parallelization. Still, computing the valuation for transitive rules (such as. $\forall x\, y, z \ \pred{Q}(x, z) \otimes \pred{R}(z, y) \rightarrow \pred{P}(x, y)$) will for example be far more demanding than for antisymmetry formulas (such as $\forall x, y \ \pred{P}(x, y) \rightarrow \neg \pred{P}(y, x)$).

\section{Proofs}
\subsection{Single-Passing}
\label{appendix:singlepassing}
\begin{prop}
Any composition of single-passing functions is also single-passing.
\end{prop}
\begin{proof}
We will prove this by structural induction. Let $f: \mathbb{R}^n\rightarrow \mathbb{R}$ be a single-passing function and let $x_1, ..., x_n\in \mathbb{R}$. Then clearly $f(x_1, ..., x_n)$ is single-passing.

Next, assume by induction that $g: \mathbb{R}^n\rightarrow \mathbb{R}$ is a composition of single-passing functions that we assume is single-passing. Let $y: \mathbb{R}^m \rightarrow \mathbb{R}$ be a single-passing function. Let $Z$ be the set of inputs to $y$ and define $x_i=y(Z)$. We show that the composition $g\left(x_1, ..., y(Z), ..., x_n\right)$ is also single-passing. 
For any $z\in Z$ holds that
\begin{equation}
\label{eq:proof_chain}
    \frac{\partial g\left(x_1, ..., x_n\right)}{\partial z}=\frac{\partial g\left(x_1, ..., x_n\right)}{\partial x_i}\frac{\partial y_i(Z_i)}{\partial z}.
\end{equation} As $g$ is single-passing, there is at most 1 number $j\in 1, ..., n$ so that $\frac{\partial g(x_1, ..., x_n)}{\partial x_j}\neq 0$. If there are 0, then there can also be no $k\in 1, ..., m$ such that $\frac{\partial g(x_1, ..., x_n)}{\partial z_k}\neq 0$ as $\frac{\partial g\left(x_1, ..., x_n\right)}{\partial x_i}=0$. 
If there is 1, then either $j\neq i$, which is the direct input $x_j$. If $j=i$, then by the assumption of $y(Z)$ being single-passing, there is at most 1 $k\in 1,...,m$ so that $\frac{\partial y(Z)}{\partial z_k}\neq 0$ and by Equation \ref{eq:proof_chain} there is at most 1 input such that $\frac{\partial g(x_1, ..., x_n)}{\partial x}\neq 0$. We conclude that the composition $g\left(x_1, ..., y(Z), ..., x_n\right)$ is single-passing.
\end{proof}

\subsection{Nonvanishing Fractions}
\subsubsection{\luk\ Aggregator}
\label{appendix:lukfrac}
\begin{prop}
The fraction of inputs $x_1, ..., x_n\in [0,1]$ for which the derivative of $A_{T_{LU}}$ is nonvanishing is $\frac{1}{n!}$.
\end{prop}
\begin{proof}
Consider standard uniformly distributed random variables $x_1, ..., x_n\sim U(0, 1)$. The sum $Y=\sum_{i=1}^n x_i$ is Irwin-Hall distributed \pcite{irwin1927frequency,10.2307/2331961}. The cumulative density function of this distribution is 
\begin{equation}
    F_Y(y) = \frac{1}{n!}\sum_{k=0}^{\lfloor y \rfloor}(-1)^k {\binom{n}{k} } (y-k)^{n-1},
\end{equation}
where $\lfloor  \rfloor$ is the floor function. The derivative of $A_{T_{LU}}$ is nonvanishing when $Y > n-1$, or equivalently as the Irwin-Hall distribution is symmetric, when $Y < 1$. Using $F_Y$ gives $F_Y(1) = \frac{1}{n!}\left((-1)^0 {\binom{n}{0}} (1-0)^n + (-1)^1{\binom{n}{1} 1}(1-1)^n\right)=\frac{1}{n!}$.
\end{proof}

This result is the same for the bounded sum aggregator, as for that the condition is that $Y<1$. 

\subsubsection{Yager Aggregator}
\label{appendix:yagerfrac}
 \begin{prop}
 The fraction of inputs $x_1, ..., x_n\in[0, 1]$ for which the derivative of $A_{T_Y}$ with $p=2$ is nonvanishing is $\frac{\pi^{\frac{n}{2}}}{2^{n}\cdot \Gamma(\frac{1}{2}n + \frac{1}{2})}$, where $\Gamma$ is the Gamma function.
 \end{prop}
 \begin{proof}
The points $x_1, .., x_n\in \mathbb{R}$ for which $\sum_{i=1}^n(1-x_i)^2 < 1$ holds describes the volume of an $n$-ball\footnote{An $n$-ball is the generalization of the concept of a ball to any dimension and is the region enclosed by a $n-1$ hypersphere. For example, the 3-ball (or ball) is surrounded by a sphere (or 2-sphere). Similarly, the 2-ball (or disk) is surrounded by a circle (or 1-sphere). A hypersphere with radius 1 is the set of points which are at a distance of 1 from its center.} with radius 1. This volume is found by \pcite{ball1997elementary}(p.5):
 \begin{equation}
     V(n) = \frac{\pi^{\frac{n}{2}}}{\Gamma(\frac{1}{2} n + 1)}.
 \end{equation}
We are interested in the part of this volume where $x_1, ..., x_n\in [0, 1]$, that is, those in a single orthant\footnote{An orthant in $n$ dimensions is a generalization of the quadrant in two dimensions and the octant in three dimensions. } of the $n$-ball. The amount of orthants in which an $n$-ball lies is $2^n$.\footnote{To help understand this, consider $n=2$. The $1$-ball is the circle with center $(0, 0)$. The area of this circle is evenly distributed over the four quadrants.} Thus, the volume of the part of the $n$-ball where $x_1, ..., x_n\in [0,1]$ is $\frac{V(n)}{2^{n}}=\frac{\pi^{\frac{n}{2}}}{2^{n}\cdot \Gamma(\frac{1}{2}n + \frac{1}{2})}$. As the total volume of points lying in $[0, 1]^n$ is 1, this is also the fraction of points for which the derivative of $A_{T_Y}$ with $p=2$ is nonvanishing.
 \end{proof}

\subsubsection{Nilpotent Aggregator}
\label{appendix:nilpfrac}
\begin{prop}
The fraction of inputs $x_1, ..., x_n\in [0,1]$ for which the derivative of $A_{T_{nM}}$ is nonvanishing is $\frac{1}{2^{n-1}}$.
\end{prop}
\begin{proof}
Consider $n$ standard uniformly distributed random variables $x_1, ..., x_n\sim U(0, 1)$. We are interested in the probability that $x_{(1)}+x_{(2)} > 1$, where $x_{(k)}$ is the $k$-th smallest sample (known as the $k$-th order statistic \pcite{david2004order}). \tcite{weisberg1971} derives the cumulative density function for linear combinations of standard uniform order statistics. Let $k_0=0<k_1 < ... < k_S \leq k_n$ be $S$ integers indicating the coefficients $d_i > 0$. We aim to calculate the probability that $\sum_{i=1}^S d_i x_{(i)} > v$. Let $r_s=k_s-k_{s-1}$ for all $1, ..., S$ and let $r_{S+1}=n-k_{S}$. Finally, let $c_{S+1}=0$ and $c_{(s)}=c_{(s+1)} + d_i$. $m$ is the largest integer so that $v\leq c_{(m)}$. Then, \tcite{weisberg1971} finds that 
\begin{equation}
    P\left\{\sum_{s=1}^{S} d_{s} x_{\left(k_{s}\right)} > v\right\}=\sum_{s=1}^{m} \frac{g_{s}^{\left(r_{s}-1\right)}\left(c_{(s)}\right)}{\left(r_{s}-1\right) !}
\end{equation}
where $g^{(i)}_s$ is the $i$-th order derivative of 
\begin{equation}
g_s(c) = \frac{(c-v)^n}{c\prod_{i=1, i\neq s}^{S+1}(c - c_{(i)})^{r_i}}
\end{equation}

Filling this in for our case, we find that $S=2$, where $k_1=1$, $k_2=2$ as $d_1 = d_2 = 1$. Therefore, $r_1=r_2=1$ and $r_3=n-2$ and $c_{(1)}=2$, $c_{(2)}=1$ and $c_{(3)}=0$. The largest integer $m$ so that $1 \leq c_{(m)}$ is 2. Filling this in, we find that
\begin{align}
    P\left\{x_{(1)} + x){(2)} > 1\right\} &=\frac{g_1^{(1-1)}(2)}{(1-1)!}+\frac{g_2^{(1-1)}(1)}{(1-1)!} \\
    &= \frac{(2-1)^n}{2(2-1)^1(2-0)^{n-2}}+\frac{(1-1)^n}{1(1-2)(1-0)^{n-2}} = \frac{1}{2^{n-1}}
\end{align}
\end{proof}

\subsection{Implications}
\label{appendix:implication-proofs}
\begin{prop}
    If a fuzzy implication $I$ is $N_C$-contrapositive symmetric, where $N_C$ is the strong negation, it is also contrapositive differentiable symmetric. 
\end{prop}
\begin{proof}
    Say we have an implication $I$ that is $N_C$-contrapositive symmetric. 
    Because $I$ is $N_C$-contrapositive symmetric, $I(1-c, 1-a) = I(a, c)$. 
    Thus, $\dmtp{1-c}{1-a} = -\frac{\partial I(a, c)}{\partial 1-c}=\frac{\partial I(a, c)}{\partial c}$. 
\end{proof}
\begin{prop}
    \label{prop:diff_left_neutral}
    If an implication $I$ is left-neutral, then $\dmp{1}{c} = 1$. If, in addition, $I$ is contrapositive differentiable symmetric, then $\dmt{a}{0} = 1$.
    \end{prop}
\begin{proof}
    First, assume $I$ is left-neutral. Then for all $c\in[0,1]$, $I(1, c) = c$. Taking the derivative with respect to $c$, it is clear that $\dmp{1}{c} = 1$. Next, assume $I$ is contrapositive differentiable symmetric. 
    Then, $\dmp{1}{c} = \dmtp{1 - c}{1-1} = \dmt{1-c}{0} = 1$. As $1-c\in[0, 1]$, $\dmt{a}{0}=1$.
\end{proof}
\section{Derivations of Used Functions}
\subsection{$p$-Error Aggregators}
\label{appendix:yager}
The unbounded Yager aggregator is
\begin{equation}
    A_{UY}(x_1, ..., x_n) = 1 - \left(\sum_{i=1}^n(1 - x_i)^p\right)^{\frac{1}{p}}, \quad p\geq 0.
\end{equation}
We can do an affine transformation $w\cdot A_{UY}(x_1, ..., x_n) - h$ on this function to ensure the boundary conditions in Definition \ref{deff:aggr} hold.
\begin{align}
\label{eq:yageragg1}
    w\cdot A_{UY}(0, ..., 0) - h = 0 \\ 
\label{eq:yageragg2}
    w\cdot A_{UY}(1, ..., 1) - h = 1 
\end{align}
Solving Equation \ref{eq:yageragg1} and \ref{eq:yageragg2} for $h$, we find
\begin{align}
    w\cdot \left(1 - \left(\sum_{i=1}^n (1 - 0)^p\right)^{\frac{1}{p}}\right) - h &= 0 \notag \\
    w\cdot\left(1 -  n^{\frac{1}{p}}\right) - h &= 0 \notag \\
    h &= w - w\cdot \sqrt[p]{n},
     \label{eq:yageragg3}\\
     w\cdot \left(1 - \left(\sum_{i=1}^n (1 - 1)^p\right)^{\frac{1}{p}}\right) - h &= 1 \notag \\
     w\cdot \left(1 - 0\right) - h &= 1 \notag \\
     h &= w - 1.
     \label{eq:yageragg4}
\end{align}
Equating \ref{eq:yageragg3} and \ref{eq:yageragg4} and solving for $h$, we find
\begin{align}
    w - w\cdot \sqrt[p]{n} &=  w - 1 \notag \\
    w &= \frac{1}{\sqrt[p]{n}}. 
\end{align}
And so $h = \frac{1}{\sqrt[p]{n}} - 1$. Filling in and simplifying we find
\begin{align}
    A_{pE}(x_1, ..., x_n) &= \frac{1}{\sqrt[p]{n}}\cdot \left(1 -  \left(\sum_{i=1}^n(1 - x_i)^p\right)^{\frac{1}{p}} \right) - \left( \frac{1}{\sqrt[p]{n}} - 1 \right) \notag \\
    &= 1 -  \frac{1}{\sqrt[p]{n}}\cdot\left(\sum_{i=1}^n(1 - x_i)^p\right)^{\frac{1}{p}} \notag \\
    &= 1 -  \left(\frac{1}{n}\sum_{i=1}^n(1 - x_i)^p\right)^{\frac{1}{p}}
\end{align}
Similarly for the t-conorm $A_{UYS}(x_1, ..., x_n) = \left(\sum_{i=1}^n x_i^p\right)^{\frac{1}{p}},  p\geq 0.$

\begin{align}
    w\cdot \left(\sum_{i=1}^n 0^p\right)^{\frac{1}{p}} - h &= 0 \notag \\
    w\cdot 0 - h &= 0 \notag \\
    h &= 0 \\
    w\cdot \left(\sum_{i=1}^n 1^p\right)^{\frac{1}{p}} - h &= 1 \notag \\
    w\cdot n^{\frac{1}{p}} - h &= 1 \notag \\
    w &= \frac{1}{\sqrt[p]{n}}\\
    A_{p-MEAN}(x_1, ..., x_n) &= \left(\frac{1}{n}\sum_{i=1}^n x_i^p\right)^{\frac{1}{p}},  p\geq 0.
\end{align}

\subsection{Sigmoidal Functions}
\label{appendix:sigm}
In Machine Learning, the logistic function or sigmoid function $\sigma(x) = \frac{1}{1 + e^{-x}}$ is a common activation function \pcite{goodfellow2016deep}(p.65-66). This inspired \pcite{Sourek2018} to introduce parameterized families of aggregation functions they call Max-Sigmoid activation functions:

\begin{equation}
    A'_{\sigma +\wedge}(x_1, ..., x_n) = \sigma\left(s\cdot\left(\sum_{i=1}^n x_i - n + 1 + b_0\right)\right), \quad
    A'_{\sigma +\vee}(x_1, ..., x_n) = \sigma\left(s\cdot\left(\sum_{i=1}^n x_i + b_0\right)\right)
\end{equation}
We generalize this transformation for any function $f: [0, 1]^n\rightarrow \mathbb{R}$ that is symmetric and increasing: 
\begin{equation}
    A'_{\sigma f}(x_1, ..., x_n) = \sigma(s\cdot(f(x_1, ..., x_n) + b_0))
\end{equation}
This cannot be an aggregation function according to Definition \ref{deff:aggr} as $\sigma\in(0, 1)$ and so the boundary conditions $A'_\sigma(1, ..., 1)=1$ and $A'_\sigma(0, ..., 0)$ do not hold. We can solve this by adding two linear parameters $w$ and $h$, redefining $\sigma_f$ as
\begin{equation}
    \sigma_f(x_1, ..., x_n) = w\cdot\sigma(s\cdot(f(x_1, ..., x_n) + b_0)) - h
\end{equation}

For this, we need to make sure the lowest value of $f$ on the domain $[0, 1]^n$ maps to 0 and the highest to 1. For this, we define $inf_f=\inf\{f(x_1, ..., x_n)|x_1, ..., x_n\in[0, 1]\}$ and $sup_f=\sup\{f(x_1, ..., x_n)|x_1, ..., x_n\in[0, 1]\}$. This gives the following system of equations
\begin{align}
   \label{eq:sigmaggis0}
   A_\sigma(0, ..., 0) &= w\cdot\sigma(s\cdot(inf_f + b_0)) - h = 0 \\
   \label{eq:sigmaggis1}
   A_\sigma(1, ..., 1) &= w\cdot\sigma(s\cdot(sup_f + b_0)) - h = 1
\end{align}
First solve both equations for $w$, starting with Equation \ref{eq:sigmaggis0}:

\begin{align}
    w\cdot\sigma(s\cdot(inf_f + b_0)) - h &= 0 \notag \\
    \frac{1}{1+e^{-s\cdot(inf_f + b_0)}} &= \frac{h}{w} \notag \\
    w &=  h\cdot (1 + e^{-s\cdot(inf_f + b_0)})
    \label{eq:sigmeq2}
\end{align}
Likewise for Equation \ref{eq:sigmaggis1}:
\begin{align}
    w\cdot\sigma(s\cdot(sup_f + b_0)) - h &= 1 \notag \\
    \frac{1}{1+e^{-s\cdot(sup_f + b_0)}} &= \frac{1 + h}{w} \notag \\
    w &= (1 + h)\cdot (1 + e^{-s\cdot(sup_f + b_0)})
    \label{eq:sigmeq1}
\end{align}
Now we can solve for $h$ by equating Equations \ref{eq:sigmeq1} and \ref{eq:sigmeq2}. Then
\begin{align*}
    (1 + h)\cdot 1 + e^{-s\cdot(sup_f + b_0)} &=  h\cdot 1 + e^{-s\cdot(inf_f + b_0)}\\
    h &= \frac{1 + e^{-s\cdot(sup_f + b_0)}}{1 + e^{-s\cdot(inf_f + b_0)}-1 + e^{-s\cdot(sup_f + b_0)}}
\end{align*}
We thus get the following formula:
\begin{align}
    A_\sigma(x_1, ..., x_n)
    &= \frac{1 + e^{-s\cdot(sup_f + b_0)}}{e^{-s\cdot(inf_f + b_0)}-e^{-s\cdot(sup_f + b_0)}}\cdot \left(\left(1 + e^{-s\cdot(inf_f + b_0)}\right) \cdot \sigma\left(s\cdot\left(f(x_1, ..., x_n) + b_0\right)\right) - 1\right)
\end{align}

If $f$ is a fuzzy logic operator of which the outputs are all in $[0,1]$, the most straightforward choice of $b_0$ is $-\frac{1}{2}$. This translates the outputs of $f$ to $[-\frac{1}{2}, \frac{1}{2}]$ and so it uses a symmetric part of the sigmoid function. For some function $f\in [0, 1]^n\rightarrow [0,1]$, we find the following simplification, noting that the supremum of $f$ is $1$ and the infimum is $0$:
\begin{align}
        \sigma_{f}(x_1, ..., x_n) &= \frac{1 + e^{-s (1-\frac{1}{2})}}{e^{s(0-\frac{1}{2})}-e^{-s(1-\frac{1}{2})}}\cdot \notag \\
  &\left(\left(1 + e^{-s(0-\frac{1}{2})}\right) \cdot \sigma\left(s\cdot\left(f(x_1, ..., x_n) - \frac{1}{2}\right)\right) - 1\right) \\
  &= \frac{1 + e^{-\frac{s}{2}}}{e^{\frac{s}{2}}-e^{-\frac{s}{2}}}\cdot \frac{e^{\frac{s}{2}} - 1}{e^{\frac{s}{2}} - 1}\cdot 
  \left(\left(1 + e^{\frac{s}{2}}\right) \cdot \sigma\left(s\cdot\left(f(x_1, ..., x_n) - \frac{1}{2}\right)\right) - 1\right) \\
  &= \frac{e^{\frac{s}{2}} - e^{-\frac{s}{2}}}{(e^{\frac{s}{2}}-e^{-\frac{s}{2}})(e^{\frac{s}{2}} - 1)}\cdot 
  \left(\left(1 + e^{\frac{s}{2}}\right) \cdot \sigma\left(s\cdot\left(f(x_1, ..., x_n) - \frac{1}{2}\right)\right) - 1\right) \\
  &= \frac{1}{e^{\frac{s}{2}}-1}\cdot 
  \left(\left(1 + e^{\frac{s}{2}}\right) \cdot \sigma\left(s\cdot\left(f(x_1, ..., x_n) - \frac{1}{2}\right)\right) - 1\right) \\
  \label{eq:deriv_sigmoid_tnorms}
\end{align}

Next, we proof several properties of the sigmoidal implication. 

\begin{prop}
\label{prop:sigm_mono_increase_f}
For all $a_1, c_1, a_2, c_2\in[0,1]$, 
\begin{enumerate}
    \item if $I(a_1, c_1) < I(a_2, c_2)$, then also $\sigma_I(a_1, c_1)< \sigma_I(a_2, c_2)$;
    \item if $I(a_1, c_1) = I(a_2, c_2)$, then also $\sigma_I(a_1, c_1)=\sigma_I(a_2, c_2)$.
\end{enumerate}
\end{prop}
\begin{proof}
\begin{enumerate}
\item We note that $\sigma_I$ can be written as $\sigma_I(a, c) = w\cdot \sigma\left(s\cdot\left(I(a, c) +b_0\right)\right) - h$ for constants $w=\frac{\left(1 + e^{-s\cdot(1 + b_0)}\right)^2}{e^{-s\cdot b_0}-e^{-s\cdot(1 + b_0)}}$ and $h=\frac{1 + e^{-s\cdot(1 + b_0)}}{e^{-s\cdot b_0}-e^{-s\cdot(1 + b_0)}}$. As $s>0$, $-s\cdot b_0>-s\cdot(1 + b_0)$. Therefore, $e^{-s\cdot b_0}-e^{-s\cdot(1 + b_0)}>0$. Furthermore, as $e^{-\frac{s}{2}} > 0$ then certainly $\left(1 + e^{-\frac{s}{2}}\right)^2 > 0$. As both $e^{-s\cdot b_0}-e^{-s\cdot(1 + b_0)}>0$ and $\left(1 + e^{-\frac{s}{2}}\right)^2>0$, then also $w>0$. As $s>0$, $s\cdot\left(I(a_1, c_1) + b_0\right) < s\cdot\left(I(a_2, c_2) + b_0\right)$ as $I(a_1, c_1) < I(a_2, c_2)$. Next, note that the sigmoid function $\sigma$ is a monotonically increasing function. Using $w>0$ we find that $\sigma_I(a_1, c_1) = w\cdot\sigma(s\cdot\left(I(a_1, c_1) + b_0\right)) < w\cdot\sigma(s\cdot\left(I(a_2, c_2) + b_0\right))=\sigma_I(a_2, c_2)$.
\item $$
    \sigma_I(a_1, c_1) = w\cdot \sigma\left(s\cdot\left(I(a_1, c_1) +b_0\right)\right) - h =  w\cdot \sigma\left(s\cdot\left(I(a_2, c_2) +b_0\right)\right) - h = \sigma_I(a_2, c_2) 
$$
\end{enumerate}
\end{proof}

\begin{prop}
\label{prop:sigm_supthenone}
$\sigma_I(a, c)$ is 1 if and only if $I(a, c) = 1$. Similarly, $\sigma_I(a, c)$ is 0 if and only if $I(a, c) = 0$. 
\end{prop}
\begin{proof}
Assume there is some $a, c\in[0, 1]$ so that $I(a, c) = 1$. By construction, $\sigma_I(a, c)$ is 1 (see \ref{appendix:sigm}).

Now assume there is some $a_1, c_1\in[0, 1]$ so that $\sigma_I(a_1,c_1) = 1$. Now consider some $a_2, c_2$ so that $I(a_2, c_2) = 1$. By the construction of $\sigma_I$, $\sigma_I(a_2, c_2) = 1$. For the sake of contradiction assume $I(a_1, c_1) < 1 $. However, by Proposition \ref{prop:sigm_mono_increase_f} as $I(a_1,c_1) < I(a_2,c_2)$ then $\sigma_I(a_1,c_1) < \sigma_I(a_2,c_2)$ has to hold. This is in contradiction with $\sigma_I(a_1, c_1) = \sigma_I(a_2, c_2) = 1$ so the assumption that $I(a_1, c_1)< 1$ has to be wrong and $I(a_1, c_1)=1$.

The proof for $I(a,c)=0$ is analogous.
\end{proof}

\begin{prop}
For all fuzzy implications $I$, $\sigma_I$ is also a fuzzy implication.
\end{prop}
\begin{proof}
By Definition \ref{def:implication} $I(\cdot, c)$ is decreasing and $I(a, \cdot)$ is increasing. Therefore, by Proposition \ref{prop:sigm_mono_increase_f}.1, $\sigma_I(\cdot, c)$ is also decreasing and $\sigma_I(a, \cdot)$ is also increasing. Furthermore, $I(0, 0) = 1$, $I(1, 1) = 1$ and $I(1, 0)=0$. We find by Proposition \ref{prop:sigm_supthenone} that then also $\sigma_I(0, 0) = 1$, $\sigma_I(1, 1) = 1$ and $\sigma_I(1, 0) = 0$. 
\end{proof}

$I$-sigmoidal implications only satisfy left-neutrality if $I$ is left-neutral and $s$ approaches 0.

\begin{prop}
\label{prop:sigm_contrapos}
If a fuzzy implication $I$ is contrapositive symmetric with respect to $N$, then $\sigma_I$ also is.
\end{prop}
\begin{proof}
Assume we have an implication $I$ that is contrapositive symmetric and so for all $a, c\in[0, 1]$, $I(a, c) = I(N(c), N(a))$. By Proposition \ref{prop:sigm_mono_increase_f}.2, $\sigma_I(a, c) = \sigma_I(N(c), N(a))$. Thus, $\sigma_I$ is also contrapositive symmetric with respect to $N$.
\end{proof}
By this proposition, if $I$ is an S-implication, $\sigma_I$ is contrapositive symmetric and thus also contrapositive differentiable symmetric. 

\begin{prop}
If $I$ satisfies the identity principle, then $\sigma_I$ also satisfies the identity principle.
\end{prop}
\begin{proof}
Assume we have a fuzzy implication $I$ that satisfies the identity principle. Then $I(a, a) = 1$ for all $a$. By Proposition \ref{prop:sigm_supthenone} it holds that $\sigma_I(a, a)$ is also 1.
\end{proof}

\subsection{Nilpotent Aggregator}
\label{appendix:nilpotent}
\begin{prop}
Equation \ref{eq:aggtnorm} is equal for the Nilpotent t-norm to
\begin{equation}
\label{eq:aggrnilp}
    A_{T_{nM}}(x_1, ..., x_n) = \begin{cases}
        \min(x_1, ..., x_n), &\text{if } x_i + x_j > 1;\ x_i \text{ and } x_j \text{ are the two lowest values in } x_1, ..., x_n   \\
        0, &\text{otherwise.}
    \end{cases}
\end{equation}
\end{prop}
\begin{proof}

We will proof this by induction. Base case: Assume $n=2$. Then $A_{T_{nM}}(x_1, x_2) = T_{nM}(x_1, x_2)$. $x_1$ and $x_2$ are the two lowest values of $x_1, x_2$, so the condition in Equation \ref{eq:aggrnilp} would change to $x_1 + x_2 > 1$.

Inductive step: We assume Equation \ref{eq:aggrnilp} holds for some $n \geq 2$. Then by Equation \ref{eq:aggtnorm} \\$A_{T_{nM}}(x_1, ..., x_{n+1}) = T_{nM}(A_{T_{nM}}(x_1, ..., x_n), x_{n+1})$. Note that if $A_{T_{nM}}(x_1, ..., x_n)=0$ then $A_{T_{nM}}(x_1, ..., x_{n+1})$ is also 0 as $x_{n+1}\in[0, 1]$ and so $0 + x_{n+1} > 1$ can never hold. We identify three cases: 
\begin{enumerate}
\item If $x_{n+1}$ is the lowest value in $x_1, ..., x_{n+1}$, then $A_{T_{nM}}(x_1, ..., x_n)$ is either the second lowest value in $x_1, ..., x_{n+1}$ or 0. If it is 0, the sum of the second and third lowest values is not greater than 1, and so the sum of the two lowest values can neither be. If it is not, then $T_{nM}(A_{T_{nM}}(x_1, ..., x_n), x_{n+1})$ first compares if $A_{T_{nM}}(x_1, ..., x_n) + x_{n+1} > 1$, that is, if the sum of the two lowest values in $x_1, ..., x_{n+1}$ is higher than 1, and returns $x_{n+1}$ if this holds and 0 otherwise.

\item If $x_{n+1}$ is the second lowest value in $x_1, ..., x_{n+1}$, then $A_{T_{nM}}(x_1, ..., x_n)$ is either the lowest value in $x_1, ..., x_{n+1}$ or 0. If it is 0, the sum of the first and third lowest values is not greater than 1, and so the sum of the two lowest values can neither be. If it is not, then $T_{nM}(A_{T_{nM}}(x_1, ..., x_n), x_{n+1})$ first compares if $A_{T_{nM}}(x_1, ..., x_n) + x_{n+1} > 1$, that is, if sum of the two lowest values in $x_1, ..., x_{n+1}$ is higher than 1, and returns $A_{T_{nM}}(x_1, ..., x_n)$ if this holds and 0 otherwise.

\item If $x_{n+1}$ is neither the lowest nor second lowest value in $x_1, ..., x_{n+1}$, then the sum $s$ of the two lowest values in $x_1, ..., x_n$ is also the sum of the two lowest values in $x_1, ..., x_{n+1}$. If $A_{T_{nM}}(x_1, ..., x_n)$ is 0, then $s$ can not have been greater than 1 and so $A_{T_{nM}}(x_1, ..., x_{n+1})$ is also 0. If it is not, then $s>1$ and  $A_{T_{nM}}(x_1, ..., x_n)$ is the lowest value and surely $A_{T_{nM}}(x_1, ..., x_n) + x_{n+1} > 1$ as $x_{n+1}$ is at least as large as the second lowest value. 
\end{enumerate}
\end{proof}

By considering $E_{S_{nM}}(x_1, ..., x_n) = 1-A_{T_{nM}}(1-x_1, ..., 1-x_n)$, it is easy to see that 
\begin{equation}
E_{S_{nM}}(x_1, ..., x_n) = \begin{cases}
    \max(x_1, ..., x_n), &\text{if } x_i + x_j < 1;\ x_i \text{ and } x_j \text{ are the two largest values in } x_1, ..., x_n   \\
    1, &\text{otherwise.}
\end{cases}
\end{equation}
\subsection{Yager R-Implication}
\label{appendix:yagerrimpl}
The Yager t-norm is defined as $T_Y(a, b) = 1 - \left((1 - a)^p + (1 - b)^p\right)^{\frac{1}{p}}$. The Yager R-implication then is defined (see Definition \ref{deff:r-implication}) as 
\begin{equation}
    I_{T_Y}(a, c) = \sup\{b\in [0, 1]|T_Y(a, b)\leq c\}
\end{equation}

When $a\leq c$, $I_{T_Y}=1$ as $T_Y(a, 1) = a\leq c$. Assuming $a>c$, we find by filling in 
\begin{equation}
    I_{T_Y}(a, c) = \sup\{b\in [0, 1]|1 - \left((1 - a)^p + (1 - b)^p\right)^{\frac{1}{p}}\leq c\}, \quad a > c
\end{equation}
To get a closed-form solution of $I_{T_Y}$ we have to find the largest $b$ for which $1 - \left((1 - a)^p + (1 - b)^p\right)^{\frac{1}{p}}\leq c$. Solving this inequality for $b$, we find 
\begin{align}
    c &\geq 1 - \left((1 - a)^p + (1 - b)^p\right)^{\frac{1}{p}} \notag\\ 
    (1 - c)^p &\leq (1 - a)^p + (1 - b)^p \notag\\
    1 - b &\geq \left((1 - c)^p - (1-a)^p\right)^{\frac{1}{p}} \notag\\
    b &\leq 1 - \left((1 - c)^p - (1 - a)^p\right)^{\frac{1}{p}}.
\end{align}
If $a>c$, then $(1-c)^p > (1-a)^p$ and thus $(1 - c)^p - (1 - a)^p>0$. Furthermore, as $a, c\in [0, 1]$, $(1 - c)^p - (1 - a)^p\leq 1$. Therefore, it has to be true that $1 - \left((1 - c)^p - (1 - a)^p\right)^{\frac{1}{p}}\in[0, 1]$. The largest value $b\in[0, 1]$ for which the condition holds is then equal to $1 - \left((1 - c)^p - (1 - a)^p\right)^{\frac{1}{p}}$ as it is in $[0, 1]$ and satisfies the inequality. 

Combining this with the earlier observation that when $a\leq c$, $I_{T_Y} = 1$, we find the following R-implication:
\begin{equation}
    I_{T_Y}(a, c) = \begin{cases}
        1, & \text{if } a \leq c \\
        1 - \left((1 - c)^p - (1 - a)^p\right)^{\frac{1}{p}}, & \text{otherwise.}
      \end{cases}
\end{equation}

We plot $I_{T_Y}$ for $p=2$ in Figure \ref{fig:yager-r-s-impl}. 
As expected, $p=1$ reduces to the \luk\ implication. The derivatives of this implication are 
\begin{align}
    \dmpa{I_{T_Y}}(a, c) &= \begin{cases}
    ((1 - c)^p - (1 - a)^p)^{\frac{1}{p}-1}\cdot (1 - c) , & \text{if } a > c \\
    0, &\text{otherwise,}
    \end{cases}\\
    \dmta{I_{T_Y}}(a, c)&= \begin{cases}
    ((1 - c)^p - (1 - a)^p)^{\frac{1}{p}-1}\cdot (1 - a) , & \text{if } a > c \\
    0, &\text{otherwise.}
    \end{cases}
\end{align}

\section{\productlogic}
\label{appendix:prl-sl}
We compare \productlogic (\dpfl), which uses the product t- and t-conorm $T_P,\ S_P$, the Reichenbach implication $I_{RC}$ and the log-product aggregator $\logprod$, and a probabilistic logic method called Semantic Loss \pcite{pmlr-v80-xu18h}. 
\begin{deff}
Let $\predicates$ be a set of predicates, $\objects$ the domain of discourse, $\interpretation$ an embedded interpretation of $\fol$ and $\corpus$ a knowledge base of background knowledge. The \textit{Semantic Loss} is defined as
\begin{equation}
\label{eq:semantic_loss}
    \loss_S(\btheta;\corpus) = -\log \sum_{\world\models \corpus} \prod_{\world \models \predP(o_1, ..., o_m)} \interpretation(\predP)(o_1, ..., o_m) \prod_{\world \models \neg \predP(o_1, ..., o_m)} \left( 1- \interpretation(\predP)(o_1, ..., o_m)\right),
\end{equation}
where $\world$ is a \textit{world} (also known as a \textit{Herbrand interpretation}) that assigns a binary truth value to every ground atom and
where $\interpretation(\predP)(o_1, ..., o_m)$ is the probability of a ground atom.
\end{deff}
This computes the logarithm of the sum of probabilities of all worlds for which $\corpus$ holds, or the probability of sampling a world that is consistent with $\corpus$. The probability that a ground atom $\predP(o_1, ..., o_m)$ is true is $\interpretation(\predP)(o_1, ..., o_m)$. The different ground atoms are assumed to be independent. By marginalizing out the world $\world$, it can be used for injecting background knowledge in unsupervised or semi-supervised learning. Compared to \dpfl, Semantic Loss is exponential in the size of the amount of ground atoms as the sum over valid worlds has to be computed. Equivalent formulas have equal Semantic Loss, and a knowledge base consisting of a conjunction of facts is equal to the cross-entropy loss function.  

\dpfl is connected in an interesting way to Semantic Loss. It corresponds to a single iteration of the \textit{loopy belief propagation} algorithm \pcite{murphy1999}. 

\begin{prop}
\label{prop:dpfl}
Let $\varphi$ be a closed formula so that if $\predP_1(o_{11}, ..., o_{1m})$ and $\predP_2(o_{21}, ..., o_{2m})$ are both ground atoms appearing in $\varphi$, then $\predP_1(o_{11}, ..., o_{1m})=\predP_2(o_{21}, ..., o_{2m})$. Then it holds that $\loss_S(\btheta; \varphi)=-\val(\{\}, \varphi)$ when using $T=T_P$, $S=S_P$, $I=I_{RC}$ and $A=\logprod$.
\end{prop}

As there are no loops when each ground atom appears uniquely in $\varphi$, the factor graph over which loopy belief propagation is done is a tree. As $\val(\{\}, \varphi)$ corresponds to a single iteration of loopy belief propagation, this is equal to regular belief propagation which is an exact method for computing queries on probabilistic models \pcite{pearl1988}. Clearly, this condition on $\varphi$ is very strong. Although loopy belief propagation is known to often be a good approximation empirically \pcite{murphy1999}, the degree to which \dpfl approximates Semantic Loss requires further research as this is not a guarantee. However, if \dpfl approximates Semantic Loss well, it can be a strong alternative as it is not an exponential computation. However, it also means that most problems of \dpfl will also be present in Semantic Loss. For example, if we just have the formula $\forall\pred{raven}(x)\rightarrow\pred{black}(x)$, the grounding of the knowledge base will not contain repeated ground atoms, and thus Semantic Loss and \dpfl are equivalent and share difficulties related to the imbalance of modus ponens and modus tollens.

\subsection{Proof}
In this section, we proof Proposition \ref{prop:dpfl}. Without loss of generality we assume that $\corpus=\varphi=\forall\ x_1, ..., x_n\ \phi$. Slightly rewriting Equation \ref{eq:semantic_loss}, we find that the probability distribution of Semantic Loss is 
\begin{equation}
p(\varphi|\interpretation) = \sum_{\world} p(\varphi|\world) p(\world|\interpretation)
\end{equation}
where we define the valuation probability $p(\varphi|\world) = I[\world\models\varphi]$ and the world probability $$p(\world|\interpretation) = \prod_{\world \models \predP(o_1, ..., o_k)} \interpretation(\predP)(o_1, ..., o_k) \prod_{\world \models \neg \predP(o_1, ..., o_k)} \left( 1- \interpretation(\predP)(o_1, ..., o_k)\right).$$
A Bayesian network is a joint distribution factorized as $p(x) = \prod_{i=1}^n p(x_i|x_{\{pa(i)\}})$ where $x_{\{pa(i)\}}$ is the set of random variables that are parents of $x_i$. In particular, we are interested in the joint distribution $p(\varphi, \world|\interpretation)$. We use the compositional structure of $\varphi$ to expand $p(\varphi|\world)$. 

Let $p(\phi|ch(\phi))$ be the probability that $\phi$, a subformula of $\varphi$, is true according to the binary truth table conditioned on the truth value of the direct subformulas $ch(\phi)$ of $\phi$. For example, if $\phi=\alpha\otimes \beta$, then $p(\phi|\alpha, \beta)=1$ if $\alpha$ and $\beta$ are 1, and otherwise $p(\phi|\alpha, \beta)=0$. For an atomic formula $\predP(o_1, ..., o_k)$, we do a lookup in the world $\world$, $p(\predP(o_1, ..., o_k|\world_{\predP(o_1, ..., o_k)})=\world_{\predP(o_1, ..., o_k)}$. Let $\Phi$ be the set of all subformulas of $\varphi$. We express the joint distribution as

\begin{equation}
    p(\Phi, \world|\interpretation) = p(\world|\interpretation)\prod_{\phi\in \Phi}p(\phi|ch(\phi)). 
\end{equation}

A specific world $\world$ uniquely determines a single $\Phi$ so that $\prod_{\phi\in \Phi}p(\phi|ch(\phi))=1$. Note that the distribution $p(\varphi|\world)=\prod_{\phi\in \Phi}p(\phi|ch(\phi))$ forms a polytree (or directed tree), as a logical expression is formed as a tree. From this Bayesian network, we define the factor graph over which we do the belief propagation. For brevity, we denote a specific ground atom $\predP(o_1, ..., o_k)$ as $\predP_O$.

We note that we use $m$ to denote an instantiation instead of $\mu$ in this section. $\mu$ is used to denote the messages as customary in belief propagation. 

\begin{itemize}
    \item There is a variable node $w_{\predP_O}$ for every ground atom $\predP_O$ appearing in the grounding of $\varphi$. Additionally, there is a factor node \\ $f_{w_{\predP(o_1, ..., o_k)}}(w_{\predP_O})=\interpretation(\predP)(o_1, ..., o_k)^{w_{\predP_O}} \cdot \left(1 - \interpretation(\predP)(o_1, ..., o_k)\right)^{1-w_{\predP_O}}$.
    \item There is a variable node $\phi$ and a factor node $f_\phi$ for every subformula $\phi\in \Phi$.
    \item For every $\phi=\predP_O,\ \phi\in\Phi$, $f_\phi(\phi, w_{\predP_O})=I[\phi = w_{\predP_O}]$. 
    \item For every $\phi=\neg \alpha,\ \phi\in\Phi$, $f_\phi(\phi, \alpha) = I[\phi=1 - \alpha]$.
    \item For every $\phi=\alpha\otimes \beta,\ \phi\in\Phi$, $f_\phi(\phi, \alpha, \beta) = I[\phi = \alpha\cdot \beta]$.
    \item Let  $\varphi=\forall\ x_1, ..., x_n\ \phi$ be the top node. Denote the set of all instances of $\varphi$ is $M$. Then $f_\varphi(\varphi, m_1, ..., m_{|M|}) = I[\varphi=\prod_{m\in M}\alpha_m]$ where $e_\mu$ is the random variable corresponding to the instantiation of $m$ in $\phi$.
\end{itemize}

We ignore the other connectives as they can be formed from $\neg$ and $\otimes$, both in classical logic as in \dpfl. Next, we compute the messages in belief propagation. We start from the world variable nodes $w_{\predP_O}$ and move up through the computation tree to $\varphi$.  The messages for factors to variables are given as \pcite{Bishop2006PatternLearning} $\mu_{f_s\rightarrow x}(x)=\sum_{X} f_s(x, X) \prod_{y \in ne(f_s) \setminus x}\mu_{y\rightarrow f_s}(y)$ where $ne(x)$ is the set of neighbours of node $x$. The messages for variables to factors are given as $\mu_{x\rightarrow f_s}(x) = \prod_{l\in ne(x)\setminus f_s}\mu_{f_l\rightarrow x}(x)$.

\begin{enumerate}[a.]
    \item $\mu_{f_{\world_{\predP_O}}\rightarrow \world_{\predP_O}}(\world_{\predP_O})=\interpretation(\predP)(o_1, ..., o_k)^{\world_{\predP_O}} \cdot \left(1 - \interpretation(\predP)(o_1, ..., o_k)\right)^{1-w_{\predP_O}}$, factor to variable for ground atom.
    \item $\mu_{\world_{\predP_O}\rightarrow f_\phi}(\world_{\predP_O})=\mu_{f_{\world_{\predP_O}} \rightarrow \world_{\predP_O}}(\world_{\predP_O})\prod_{\psi:\psi=\predP_O, \phi\neq\psi}\mu_{f_\psi \rightarrow \world_{\predP_O}}=\mu_{f_{\world_{\predP_O}}\rightarrow \world_{\predP_O}}(\world_{\predP_O})$, ground atom variable to atomic formula factor $\phi=\predP(o_1, ..., o_k)$. We here assume that the incoming messages $\mu_{f_\psi \rightarrow \world_{\predP_O}}$ from other atomic formulas using ground atom $\predP_O$ are initialized with 1. We have not yet, and are not able to, compute these as the graph might contain loops. 
    \item $\mu_{f_\phi\rightarrow \phi}(\phi) = I[\phi=1]\mu_{\world_{\predP_O}\rightarrow f_\phi}(1) + I[\phi=0]\mu_{\world_{\predP_O}\rightarrow f_\phi}(0)=\interpretation(\predP)(o_1, ..., o_k)^{\phi} \cdot \left(1 - \interpretation(\predP)(o_1, ..., o_k)\right)^{1-\phi}$, factor to variable for atomic formulas.
    \item $\mu_{\phi \rightarrow f_\alpha}(\alpha) = \mu_{f_\alpha \rightarrow \alpha}(\alpha)$ for subformula variables to factors of other subformulas $\alpha$. As $\phi$ is only used in two factors, this simply passes the downstream message through.
    \item $\mu_{f_\phi\rightarrow \phi}(\phi)=\mu_{f_\alpha\rightarrow \alpha}(1)^{1-\phi}\cdot \mu_{f_\alpha\rightarrow \alpha}(0)^{\phi}$ factor to variable for negated subformulas $\phi=\neg \alpha$.
    \item $\mu_{f_\phi \rightarrow \phi}(\phi)=\left(\mu_{f_\alpha\rightarrow \alpha}(1) \cdot \mu_{f_\beta\rightarrow \beta}(1)\right)^\phi \cdot \left(\mu_{f_\alpha\rightarrow \alpha}(0) \cdot \mu_{f_\beta\rightarrow \beta}(0) + \mu_{f_\alpha\rightarrow \alpha}(1) \cdot \mu_{f_\beta\rightarrow \beta}(0) + \mu_{f_\alpha\rightarrow \alpha}(0) \cdot \mu_{f_\beta\rightarrow \beta}(1)\right)^{1-\phi}$ factor to variable for conjunctions $\phi=\alpha\otimes \beta$.
    \item $\mu_{f_\varphi \rightarrow \varphi}(\varphi)=\left(\prod_{m\in M} \mu_{f_{\alpha_m}\rightarrow \alpha_m}(1)\right)^\varphi\cdot \left(\sum_{\alpha_1, ..., \alpha_{|M|}:\exists i: \alpha_i=0}\prod_{m\in M} \mu_{f_{\alpha_m}\rightarrow \alpha_m}(\alpha_m)\right)^{1-\varphi}$ for the factor of the universally quantified formula $\varphi=\forall x_1, ..., x_n\ \phi$ where $\alpha_m$ is the subformula corresponding to the instantiation $m$. The second term is the sum over all cases so that there is a subformula with truth value 0. 
\end{enumerate}

We wish to know what the marginal probability $p(\varphi=1|\interpretation)$ is. A marginal of a variable $\phi$ in a factor graph is found as $p(\phi) = \prod_{s\in ne(\phi)}\mu_{f_s\rightarrow x}(\phi)$. The variable node $\varphi$ only has the factor node $f_\varphi$ as a neighbor, so using (g.) we find\footnote{We use approximation equality for simplification. Note that this is not necessarily true in loopy belief propagation.}
\begin{equation}
\label{eq:marginalvarphi}
    p(\varphi=1|\interpretation) \approx \prod_{m\in M} \mu_{\alpha_m\rightarrow f_\varphi}(1).
\end{equation}
Next, we use induction to proof that the computation of $\mu_{f_{\alpha_m}\rightarrow \alpha_m}(1)$ is equal to \productlogic. 

\begin{itemize}
    \item Let $\phi$ be any subformula of $\alpha_m$ for some instantiation $m$ of $\varphi$. Let $\val$ be the valuation function (Definition \ref{deff:val}) with $T=T_P$ and $N=N_C$. We proof that $\mu_{f_\phi\rightarrow \phi}(1) = \val(m, \phi)$ and $\mu_{f_\phi\rightarrow \phi}(0)=1-\val(m, \phi)$. 
    \item Base case $\phi=\predP_O$. By (c.), $\mu_{f_\phi\rightarrow\phi}(1)=\interpretation(\predP)(o_1, ..., o_k)$ and $\mu_{f_\phi\rightarrow\phi}(0)=1-\interpretation(\predP)(o_1, ..., o_k)$. By Equation \ref{eq:rlpred}\footnote{Equation \ref{eq:rlpred} refers to the ungrounded case, but here the full grounding is already done so the lookup function $l$ is not required.}, $\val(m, \phi) = \interpretation(\predP)(o_1, ..., o_k)$.
    \item Inductive step $\phi=\neg \alpha$. By (e.) and using the inductive hypothesis, $\mu_{f_\phi\rightarrow \phi}(1)=\mu_{f_\alpha\rightarrow \alpha}(0)=1 - \val(\alpha, m)$ and $\mu_{f_\phi\rightarrow \phi}(0)=\mu_{f_\alpha\rightarrow \alpha}(1)=\val(m, \alpha)$. By Equation \ref{eq:rlneg}, $\val(m, \phi) = 1 - \val(m, \alpha)$.
    \item Inductive step $\phi=\alpha \otimes \beta$. By (f.) and using the inductive hypothesis, $\mu_{f_\phi\rightarrow \phi}(1) = \mu_{f_\alpha\rightarrow \alpha}(1)\cdot \mu_{f_\beta\rightarrow \beta}(1) = \val(m, \alpha)\cdot \val(m, \beta)$ and $\mu_{f_\phi\rightarrow \phi}(0) = \mu_{f_\alpha\rightarrow \alpha}(0)\cdot \mu_{f_\beta\rightarrow \beta}(0) + \mu_{f_\alpha\rightarrow \alpha}(1)\cdot \mu_{f_\beta\rightarrow \beta}(0) + \mu_{f_\alpha\rightarrow \alpha}(0)\cdot \mu_{f_\beta\rightarrow \beta}(1) = (1-\val(m, \alpha)) (1-\val(m, \beta)) + \val(m, \alpha)(1-\val(m, \beta)) + (1-\val(m, \alpha))\val(m, \beta) = 1-\val(m, \alpha)\cdot \val(m, \beta)$. By Equation \ref{eq:rlconj} and $T_P(a, c)=a\cdot c$, $\val(m, \phi) = \val(m, \alpha)\cdot \val(m, \beta)$.
\end{itemize}

Using Equation \ref{eq:marginalvarphi} we then find that $p(\varphi=1|\interpretation)\approx \prod_{m\in M}\val(m, \phi_m)$, which is equal to the \productlogic computation of the universal quantifier in Equation \ref{eq:rlaggr}.

Importantly, as the computation of the logic is itself a tree, the only loops are caused through ground atoms appearing in multiple subformulas. Therefore, when each ground atom only appears in a single formula, \productlogic computes the same probability as Semantic Loss. 

\subsection{Example}
For example, a formula $\predP$ corresponds to a variable node $\predP$ with two possible values 1 and 0, along with the factor node with factor $\interpretation(\predP)()$ if $\predP$ is $true$ and $1-\interpretation(\predP)()$ otherwise. If we consider the formula $\gamma=\predP \otimes \neg(\predP\otimes \predQ)$ we find the following factor graph:

\begin{tikzpicture}
\begin{scope}[]
    \node[Factor, label={$I[\predP1=\world_\predP]$}] (f1) at (8, 3) {$f_{\predP1}$};
    \node[Factor, label={$I[\predP2=\world_\predP]$}] (f4) at (12, 3) {$f_{\predP2}$};
    \node[Factor, label={$I[\predQ=\world_\predQ]$}] (f2) at (0, 3) {$f_\predQ$};
    \node[Variable, label={$\predP$}] (v1) at (6, 3) {$v_{\predP1}$};
    \node[Variable, label={right:{$\predP$}}] (v4) at (12, 1.5) {$v_{\predP2}$};
    \node[Variable, label={$\predQ$}] (v2) at (2, 3) {$v_\predQ$};
    \node[Factor, label={$I[\alpha=\predP\cdot\predQ]$}] (f3) at (4, 3) {$f_\alpha$};
    \node[Variable, label={below:{$\predP\otimes\predQ$}}] (v3) at (4, 0) {$v_\alpha$};
    
    \node[Variable, label={below:{$\neg(\predP\otimes\predQ)$}}] (v5) at (8, 0) {$v_\beta$};
    \node[Factor, label={below:{$I[\beta=1-\alpha]$}}] (f5) at (6, 0) {$v_\beta$};
    \node[Variable, label={below:{$\predP\otimes\neg(\predP\otimes\predQ)$}}] (v6) at (12, 0) {$v_\gamma$};
    \node[Factor, label={below:{$I[\gamma=\predP \cdot \beta]$}}] (f6) at (10, 0) {$f_\gamma$};
    
    \node[Variable, label={right:{$\world_Q$}}] (vq) at (0, 1.5) {$v_{\world_\predQ}$};
    \node[Factor, label={right:{$\interpretation(\predQ)()$}}] (fq) at (0, 0) {$f_{\world_\predQ}$};
    
    \node[Variable, label={$\world_P$}] (vp) at (10, 3) {$v_{\world_\predP}$};
    \node[Factor, label={left:{$\interpretation(\predP)()$}}] (fp) at (10, 1.5) {$f_{\world_\predP}$};
    
\end{scope}

\begin{scope}[]
    \path (f1) edge (v1);
    \path (f2) edge (v2);
    \path (v1) edge (f3);
    \path (v2) edge (f3);
    \path (v3) edge (f3);
    \path (fp) edge (vp);
    \path (vp) edge (f1);
    \path (vp) edge (f4);
    
    \path (v5) edge (f6);
    \path (f4) edge (v4);
    \path (v6) edge (f6);
    \path (v4) edge (f6);
    \path (v5) edge (f5);
    \path (v3) edge (f5);
    \path (fq) edge (vq);
    \path (vq) edge (f2);
\end{scope}
\end{tikzpicture}

Here, box nodes correspond to factor nodes and circle nodes correspond to variable nodes. As $\gamma$ is the top formula, this is where the messages get passed to. Note that there is a single loop, which is present because the atom $\predP$ is used twice in the formula. This causes two incorrect messages: $\mu_{v_{\world_\predP}\rightarrow f_{\predP 1}}$ and $\mu_{v_{\world_\predP}\rightarrow f_{\predP 2}}$. The first is incorrect as it does not have access to the incoming message $\mu_{f_{\predP 2}\rightarrow v_{\world_\predP}}$ and puts it to 1. 

With \productlogic, we find the expression $\interpretation(\predP)()\cdot\left(1 - \interpretation(\predP)()\cdot \interpretation(\predQ)()\right)$, while the correct probability is where $\predP$ is 1 and $\predQ$ is 0, that is $\interpretation(\predP)()\cdot\left(1 - \interpretation(\predQ)()\right)$.

\section{Additional experiments}
\label{appendix:experiments}
In this Appendix, we report additional experiments on the \pred{same} problem. 
Throughout this Appendix we use vanilla stochastic gradient descent with a learning rate of 0.01 and 0.5 momentum instead of ADAM. 
The reason for this is that the best configurations seem to perform better with this optimizer, while the other configurations seem to perform better with ADAM. 
\subsection{Comparing different formulas of the \pred{same} problem}
\label{appendix:formulas}
We note what the values of $\mpupdateratio$ and $\mtupdateratio$ roughly are for each of the used formulas if we were to pick at random.
\begin{enumerate}
\item $ \forall x, y\ \pred{zero}(x)\otimes \pred{zero}(y) \rightarrow \pred{same}(x, y), ..., \forall x, y\ \pred{nine}(x)\otimes \pred{nine}(y) \rightarrow \pred{same}(x, y) $. 
For this formula, $\mpupdateratio\geq \frac{1}{10}$ as it is the distribution of $\pred{same}(x, y)$\footnote{It is slightly more than $\frac{1}{10}$ because we are using a minibatch of examples. Therefore, the reflexive pairs (i.e., $\pred{same}(x, x)$) are common.} and $\mtupdateratio\leq \frac{99}{100}$ as it is 1 minus the probability that both $x$ and $y$ are $\pred{zero}$. The modus ponens case is true in more than $\frac{1}{100}$ cases, the modus tollens casein less than $\frac{9}{10}$ cases and the `distrust' option in more than $\frac{9}{100}$ cases.

\item $ \forall x, y\ \pred{zero}(x) \otimes \pred{same}(x, y) \rightarrow \pred{zero}(y), ..., \forall x, y\ \pred{nine}(x) \otimes \pred{same}(x, y) \rightarrow \pred{nine}(y) $. 
For this formula, $\mpupdateratio = \frac{1}{10}$ as it is the probability that a digit represents zero and $\mtupdateratio\leq \frac{99}{100}$. The modus ponens cases is true in more than $\frac{1}{100}$ cases, the modus tollens in $\frac{9}{10}$ cases and the `distrust' option in $\frac{9}{100}$ cases.

\item $ \forall x, y\ \pred{same}(x, y) \rightarrow \pred{same}(y, x) $. 
As this is a bi-implication, $\mpupdateratio\geq\frac{1}{10}$ and $\mtupdateratio\leq \frac{9}{10}$. The `distrust' option is not possible in this formula.
\end{enumerate}

From this, we can see that a set of operators is better than random guessing for the consequent updates if $\mpupdateratio>0.1$. It is more difficult to say what the value of $\mtupdateratio$ should be to be as good as random guessing, as the probabilities are upper bounded with the lowest bound at $0.9$. We can only say that we know a set of operators to be better than random if $\mtupdateratio > 0.99$.
\subsection{Varying the Aggregators}
\label{sec:mnist_symmetric_aggregator}
In this section, we analyze symmetric configurations in the \pred{same} problem, except that we use aggregators other than the one formed by extending the t-norm. In particular, we will consider the RMSE aggregator ($A_{GME}$ with $p=2$) and the log-product aggregator $\logprod$.

\begin{table}[h]
\centering
\begin{tabular}{l....|....}
    &  \multicolumn{4}{c}{$\logprod, w_{\dfl}=10$} & \multicolumn{4}{c}{$A_{RMSE}, w_{\dfl}=1$} \\
\hline 
\multicolumn{1}{l}{}     & \mc{Accuracy} & \mc{$\mpratio$} & \mc{$\mpupdateratio$} & \mc{$\mtupdateratio$} & \mc{Accuracy} & \mc{$\mpratio$} & \mc{$\mpupdateratio$} & \mc{$\mtupdateratio$}               \\
\hline
 $T_G$        & 96.3          & 0.10            & \bft{0.89}            & 0.97 \
                                    & 96.2          & 0.10            & \bft{0.89}            & 0.97                           \\
 $T_{LK}$     & 96.6          & 0.5             & 0.33                  & 0.67                           
                                    & 96.9          & 0.5             & 0.06                  & 0.95                           \\
                     $T_P$        & \bft{96.7}    & 0.44            & 0.48                  & 0.69                           
                     & \bft{97.0}    & 0.08            & 0.81                  & 0.98 \\
                     $T_Y,\ p=2$  & 96.4          & 0.40            & 0.54                  & 0.74                           
                     & 96.6          & 0.12            & 0.87                  & 0.97                           \\
                     $T_Y,\ p=20$ & 96.0          & 0.29            & 0.51                  & 0.78                           
                     & 95.9          & 0.18            & 0.82                  & 0.98                           \\
                     $T_{Nm}$     & 95.4          & 0.29            & 0.44                  & 0.84                            
                     & 95.4          & 0.03            & 0.62                  & \bft{0.99}                           \\

\hline
\end{tabular}
\caption{Configurations using the RMSE aggregator with $w_{\dfl}=1$ and the log product aggregator with $w_{dfl}=10$.}
\label{table:mnist_symmetric_rmse}
\end{table}

Table \ref{table:mnist_symmetric_rmse} shows the results when using the RMSE aggregator and a \dfl weight of 1 and the log product aggregator and a \dfl weight of 10. 
Nearly all configurations perform significantly better using these aggregators than when using their `symmetric' aggregator. In particular, the Gödel, \luk\ and Yager t-norms all outperform the baseline with both aggregators as they are differentiable everywhere and can handle outliers.

The product t-norm seems to do slightly worse with the RMSE aggregator than with the log-product aggregator. Like we discussed in Section \ref{sec:prod-implication}, $\mpratio$ is higher using this aggregator because the corners $a_i=0,\ c_i=0$ and $a_i=1,\ c_i=1$ will have no gradient when using the RMSE aggregator.
However, the values of $\mpupdateratio$ and $\mtupdateratio$ are much lower than when using the log-product aggregator. This could have to do with the previously made point: As it no longer has a gradient of 1 at the corners $a=0,\ c=0$ and $a=1,\ c=1$, the large gradients are only when the agent is not yet confident about some prediction. This case is inherently `riskier', but also contributes more information. It is not as informative to increase the confidence of $a=0$ if $a$ is already very low.


The \luk\ t-norm has a particularly high accuracy of 96.9\% with the log product and is on the level of performance of the product t-norm. However, it has a very low value for $\mpupdateratio$ of 0.06 and a relatively low value for $\mtupdateratio$. Interestingly, it is also the only configuration for which $\mpupdateratio$ is higher when using the RMSE aggregator than the log-product aggregator. 

\subsection{Reichenbach-Sigmoidal Implication}
\label{sec:mnist_rcsigmoidal}
The newly introduced Reichenbach-sigmoidal implication $\sigma_{I_{RC}}$ is a promising candidate for the choice of implication as we have argued in Section \ref{sec:sigm_implication}. 
To get a better understanding of this implication, we investigate the effect of its parameters in the \pred{same} problem. 
We fix the aggregator to the log-product, the conjunction operator to the Yager t-norm with $p=2$, and use a \dfl weight of $w_{\dfl}=10$.

\begin{figure}
\centering
\begin{subfigure}[b]{0.49\linewidth}
\includegraphics[width=\linewidth]{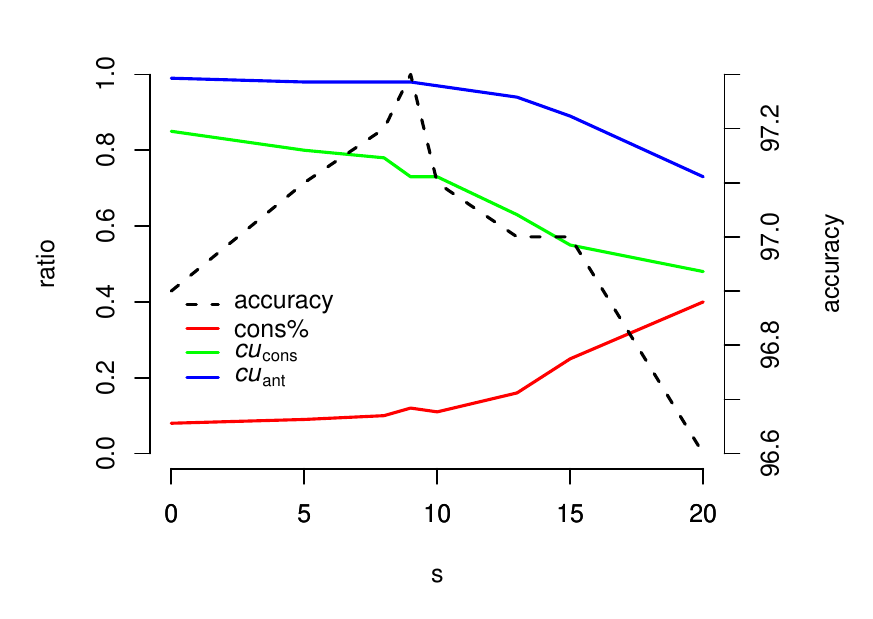}
\end{subfigure}
\begin{subfigure}[b]{0.49\linewidth}
\includegraphics[width=\linewidth]{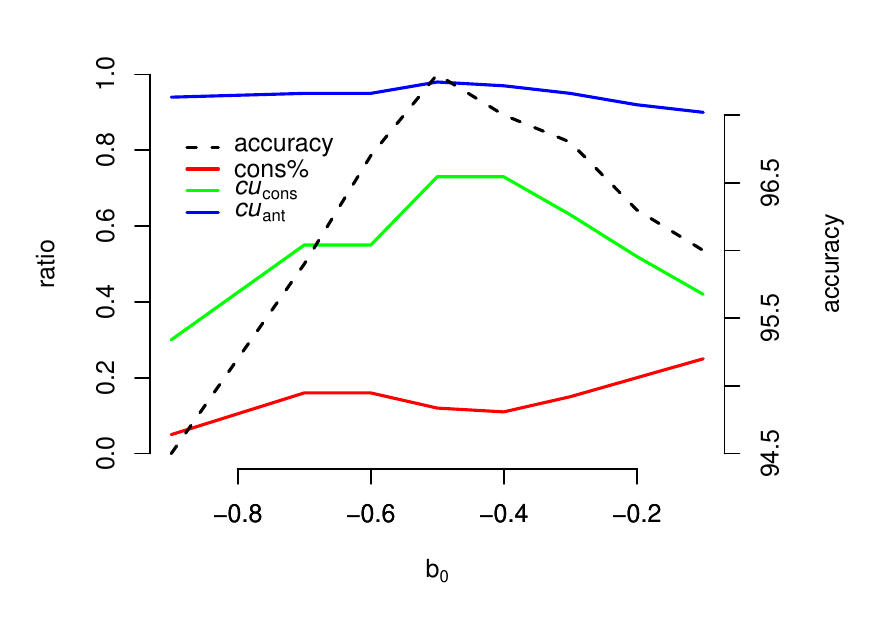}
\end{subfigure}%
\caption{The results using the Reichenbach-sigmoidal implication $\sigma_{I_{RC}}$, the $\logprod$ aggregator, $T_Y$ with $p=2$ and $w_{\dfl}=10$. Left shows the results for various values of $s$, keeping $b_0$ fixed to -0.5, and right shows the results for various values of $b_0$, keeping $s$ fixed to 9.}
\label{fig:mnist_s_experiments}
\end{figure}


On the left plot of Figure \ref{fig:mnist_s_experiments} we find the results when we experiment with the parameter $s$, keeping $b_0$ fixed to $-\frac{1}{2}$. Note that when $s$ approaches 0 the Reichenbach-sigmoidal implication is $I_{RC}$. The value of 9 gives the best results, with 97.3\% accuracy. 
This outperforms the other implications, also when using the vanilla SGD optimizer, as displayed in Table \ref{table:mnist_S-implication_experiments}.
Interestingly enough, there seem to be clear trends in the values of $\mpratio$,\ $\mpupdateratio$ and $\mtupdateratio$. Increasing $s$ seems to increase  $\mpratio$. This is because the antecedent derivative around the corner $a=0,\ c=0$ will be low, as argued in Section \ref{sec:sigm_implication}. When $s$ increases, the corners will be more smoothed out. 
Furthermore, both $\mpupdateratio$ and $\mtupdateratio$ decrease when $s$ increases. This could again be because around the corners the derivatives become small. Updates in the corner will likely be correct as the model is already confident about those. For a higher value of $s$, most of the gradient magnitude is at instances on which the model is less confident. We note that the same happened when using the RMSE aggregator and the product t-norm. Regardless, the best parameter value clearly is not the one for which the values of $\mpupdateratio$ and $\mtupdateratio$ are highest, namely the Reichenbach implication itself.

On the right plot of Figure \ref{fig:mnist_s_experiments} we experiment with the value of $b_0$. Clearly, $-\frac{1}{2}$ works best, having the highest accuracy and $\mpupdateratio$. 

\begin{table}
    \centering
    \begin{tabular}{l...c|l....}
    \hline
                                    & \mc{Accuracy} & \mc{$\mpratio$} & \mc{$\mpupdateratio$} & \mc{$\mtupdateratio$}   &  & \mc{Accuracy} & \mc{$\mpratio$} & \mc{$\mpupdateratio$} & \mc{$\mtupdateratio$}             \\
    \hline                    
    $I_{KD}$                        & 96.1          & 0.10            & \bft{0.88}            & 0.97             &              
    $I_G$                           & 90.6          & 1               & 0.07                     &                              \\
    $I_{LK}$                        & \bft{97.0}    & 0.5             & 0.03                  & 0.97                 &          
    $I_{LK}$                        & \bft{97.0}    & 0.5             & 0.03                  & 0.97                           \\
    $I_{RC}$                        & 96.9          & 0.08            & 0.85                  & \textbf{0.99} &
    $I_{GG}$                        & 94.0          & 0.86            & 0.01                  & 0.97                           \\
    $I_Y,\ p=2$                     & 96.6          & 0.12            & 0.87                  & 0.97                   &        
    $I_{T_Y}$   & 95.4          & 0.58            & 0.03                  & \bft{0.99}                              \\
    $I_{FD}$ & 96.3    & 0.07            & 0.65            & 0.96  &
    $I_{FD}$ & 96.3    & 0.07            & 0.65            & 0.96 \\ 

    & & & & & $I_{T_Y},\ p=0.5$               & 96.3          & 0.18            & 0.65                  & 0.37                              \\
    \hline
    \end{tabular}
    \caption{The results using $T_Y,\ p=2$ for the conjunction, $\logprod$ for the aggregator with $w_1=10$ and several S-implications and R-implications. These are run using vanilla SGD instead of ADAM.}
    \label{table:mnist_S-implication_experiments}
\end{table}

\subsection{Influence of Individual Formulas}
\label{sec:mnist_formulas_experiments}
\begin{table}

\centering
\begin{tabular}{l....}
\hline
Formulas                         & \mc{Accuracy} & \mc{$\mpratio$} & \mc{$\mpupdateratio$} & \mc{$\mtupdateratio$}               \\
\hline
(1) \& (2) & \bft{97.1}    & 0.05            & 0.54                  & \bft{0.99} \\
(2) \& (3)  & 95.9          & 0.12            & \bft{0.75}            & 0.95                           \\
(1) \& (3)  & 96.3          & 0.15            & 0.52                  & 0.98                           \\
(1)                     & 95.6          & 0.05            & 0.59                  & \bft{1.00} \\
(2)                     & 95.2          & 0.03            & \bft{0.78}            & 0.99                           \\
(3)                      & \bft{95.8}    & 0.19            & 0.64                  & 0.95               \\
\hline
\end{tabular}
\caption{The results using $\sigma_{I_{RC}}$ for the implication with $s=9$ and $b_0=-\frac{1}{2}$, $T_Y, p=2$ for the conjunction and $\logprod$ for the aggregator with $w_{\dfl}=10$, leaving some formulas of the \pred{same} problem out. The numbers indicated the formulas that are present during training.  }
\label{table:mnist_formula_experiments}
\end{table}
Finally, we compare what the influence of the different formulas of the \pred{same} problem are in Table \ref{table:mnist_formula_experiments}. Removing the reflexivity formula (3) does not largely impact the performance. The biggest drop in performance is by removing formula (1) that defines the $\pred{same}$ predicate. Using only formula (1) gets slightly better performance than only using formula (2), despite the fact that no positive labeled examples can be found using formula (1) as the predicates $\pred{zero}$ to $\pred{nine}$ are not in its consequent. Since 95\% of the derivatives are with respect to the negated antecedent, this formula contributes by finding additional counterexamples. Furthermore, improving the accuracy of the $\pred{same}$ predicate improves the accuracy on digit recognition: Just using the reflexivity formula (3) has the highest accuracy when used individually, even though it does not use the digit predicates.

\end{document}